\documentclass{article}

\usepackage{iclr2026_conference,times}

\usepackage[utf8]{inputenc}
\usepackage[T1]{fontenc}
\usepackage{hyperref}
\usepackage{url}
\usepackage{booktabs}
\usepackage{longtable}
\usepackage{amsfonts}
\usepackage{amsmath}
\usepackage{amsthm}
\usepackage{nicefrac}
\usepackage{microtype}
\usepackage{xcolor}
\usepackage{graphicx}
\usepackage{wrapfig}
\usepackage{float}
\usepackage{placeins}
\usepackage{orcidlink}

\usepackage{amsmath,amsfonts,bm}

\def\eqref#1{equation~\ref{#1}}

\def\plaineqref#1{\ref{#1}}

\def\1{\bm{1}}

\DeclareMathAlphabet{\mathsfit}{\encodingdefault}{\sfdefault}{m}{sl}
\SetMathAlphabet{\mathsfit}{bold}{\encodingdefault}{\sfdefault}{bx}{n}

\hypersetup{hidelinks}

\theoremstyle{plain}
\newtheorem{theorem}{Theorem}[section]
\newtheorem{proposition}[theorem]{Proposition}

\theoremstyle{definition}
\newtheorem{definition}[theorem]{Definition}

\theoremstyle{remark}

\title{Is Escalation Worth It? On the Depth of LLM Cascades}

\author{Dylan Bouchard\,\orcidlink{0009-0004-9233-2324}\thanks{Correspondence to \texttt{dbouchard92@gmail.com}}}

\iclrfinalcopy

\begin{document}

\maketitle

\begin{abstract}
LLM cascades, in which a cheap model defers to an expensive one on
low-confidence queries, are widely used to reduce inference cost. Given a
pool of models, a practitioner must decide how many models to include and
where to set each deferral threshold. We derive first-order optimality
conditions showing that, at an optimum, the ratio of expected accuracy
gain to expected downstream cost is equal across deferral boundaries. A
local search based on these conditions closely matches exhaustive search.
We also derive an identity that decomposes the accuracy gain of
score-based escalation over random escalation into two AUROC terms. Across
five benchmarks and nine deferral scores, with model sequences and
thresholds optimized from a pool of eight models, two-model cascades
improve mean test-set accuracy over single-model selection by 2.1 to 8.2
percentage points. However, allowing more than two models does not improve mean test-set
accuracy in 118 of 135 comparisons across scorers, datasets, and depth
caps, and adds at most 0.43 percentage points. 
To understand the role of deferral scores in depth gains, we conduct
counterfactual experiments with simulated confidence scores. When these scores
have high AUROC and reflect only whether the current model answered
correctly, allowing more than two models improves test-set accuracy on four of five
benchmarks. However, these gains do not persist when the scores also
reflect query difficulty shared across models, even at the same AUROC.
 These results suggest that
gains from additional depth depend on how well the confidence score
separates correct from incorrect answers for the current model compared
with later models.
\end{abstract}

\section{Introduction}
\providecommand{\arxivcodefootnote}{}

The most capable Large Language Models (LLMs) are often prohibitively expensive to query at scale, while
cheaper alternatives frequently fall short on accuracy.
 \textit{LLM cascades} address this tradeoff by querying models sequentially, starting with a cheap model and escalating to a more expensive one when a confidence score falls below a threshold.
This cascade framework has become widely adopted
\citep{chen2023frugalgptuselargelanguage,gupta2024languagemodelcascadestokenlevel,
jitkrittum2024doesconfidencebasedcascadedeferral,
yue2024largelanguagemodelcascades}, but existing studies primarily treat
the thresholds as hyperparameters to be tuned, providing little guidance on how many stages a cascade should have or when additional depth improves accuracy at a given cost budget.

We study the value of additional depth in an LLM cascade. Optimizing a cascade involves two subproblems:
setting thresholds within a fixed model sequence, and choosing which sequence of models
to use under a maximum-depth constraint. For threshold allocation, we derive
first-order conditions showing that the ratio of expected accuracy gain to expected downstream
cost must be equalized across all deferral thresholds at an optimum.
To measure whether additional depth helps, we optimize model sequences and thresholds 
from a pool of eight models for each depth cap across a range of budgets, then evaluate the selected cascades on held-out queries.

Across five benchmarks and nine deferral scores, two-model cascades improve considerably
upon the best single model available at a given budget, increasing test-set accuracy by 2.1 to 8.2
percentage points. However, allowing more than two models does not improve mean test-set
accuracy in 118 of 135 comparisons across scorers, datasets, and depth
caps, and adds at most 0.43 percentage points. 
To understand when additional depth can improve performance, we run
counterfactual experiments with simulated confidence scores at AUROC
levels of $0.5, 0.6,...,0.9$. When the simulated scores have an AUROC of 0.9 and reflect only whether
the current model answered correctly, adding a third model improves
test-set accuracy on four of five benchmarks. However, these gains do
not persist at the same AUROC when the scores also reflect query
difficulty shared across models. Such scores escalate many queries that
later models also answer incorrectly. These results suggest that gains
from additional depth depend on how well the confidence score separates
correct from incorrect answers for the current model compared with later
models. Lastly, we vary how model prices are spaced between the cheapest
and most expensive models and find little effect on depth gains under
observed confidence scores.

Our contributions are as follows. First, we derive first-order optimality 
conditions for threshold placement in a fixed model sequence and
show that a local search built on these conditions reaches 
near-optimal thresholds in practice. We further derive an identity 
that decomposes the cascade accuracy gain 
from score-based escalation over random escalation into two AUROC terms, one for the current 
model and one for the downstream models, which clarifies when higher 
correctness AUROC does and does not translate into cascade value. 
Second, we formalize depth comparison by optimizing both model sequences
and thresholds at various depth caps, and show that allowing more than two
models adds little average test accuracy over optimized two-model cascades
across five benchmarks and nine deferral scores.
Third, we run counterfactual experiments showing that depth gains depend
on whether the confidence scores reflect only the current model's
correctness or also query difficulty shared across models.\footnote{Code is available at \href{https://github.com/dylanbouchard/llm-cascade-frontiers}{github.com/dylanbouchard/llm-cascade-frontiers}.}

\section{Related Work}
\label{sec:related}

\paragraph{LLM routing and cascading.}
A large body of recent work studies cost-efficient LLM deployment via 
model selection. Routing methods select a single model per query prior 
to inference, typically using learned routers trained on query 
representations or preference data 
\citep{ong2025routellmlearningroutellms, 
zhuang2024embedllmlearningcompactrepresentations, 
feng2025graphroutergraphbasedrouterllm, 
wang2025reasonsemanticroutervllm}. Several works incorporate explicit 
budget, latency, or capacity constraints 
\citep{mei2025omnirouterbudgetperformancecontrollable, 
markovicvoronov2026robustbatchlevelqueryrouting, 
ding2025bestrouteadaptivellmrouting}. Cascading methods instead query 
models sequentially and decide whether to accept or escalate based on 
signals observed after generation. Most analytical and empirical work 
studies two-model cascades with a fixed model pair, focusing on the 
deferral rule rather than model selection 
\citep{gupta2024languagemodelcascadestokenlevel,
jitkrittum2024doesconfidencebasedcascadedeferral,
zhang2024efficientcontextualllmcascades,
su2025cprouteruncertaintyawarerouterllm, yue2024largelanguagemodelcascades}. 
AutoMix \citep{aggarwal2025automixautomaticallymixinglanguage} extends 
cascades beyond two models using self-verification, and 
\citet{valkanas2025c3pooptimizedlargelanguage} optimize multi-model 
cascades under probabilistic cost constraints. 
\citet{chen2023frugalgptuselargelanguage} jointly optimize the model 
subsequence and thresholds from a pool of models but only experiment
with cascades of length three, without comparing against optimized 
two-model cascades. Other extensions consider deployment constraints 
such as privacy 
\citep{zhang2025privacypreservedllmcascadecotenhanced} and unified 
routing-cascading formulations 
\citep{dekoninck2025unifiedapproachroutingcascading}. To our knowledge, no prior work 
systematically varies the maximum cascade depth while jointly 
optimizing model sequences and thresholds to assess whether additional 
stages improve over the best pair at a given budget.

\paragraph{Analytical characterization of the cost-quality tradeoff.}
Several prior works formulate cascade design as constrained optimization
\citep{zhang2024efficientcontextualllmcascades,gupta2024languagemodelcascadestokenlevel,chen2023frugalgptuselargelanguage,jitkrittum2024doesconfidencebasedcascadedeferral},
but do not derive a characterization of the optimal thresholds or study how the achievable frontier changes with cascade depth.
\citet{chen2023frugalgptuselargelanguage} minimize expected cost subject
to an accuracy constraint, selecting thresholds by empirical search rather
than deriving first-order necessary conditions for optimality.
\citet{valkanas2025c3pooptimizedlargelanguage} introduce probabilistic
cost constraints and provide generalization guarantees.
\citet{dekoninck2025unifiedapproachroutingcascading} derive optimal
routing, cascading, and cascade-routing strategies from query- and
output-dependent quality estimates.
\citet{jitkrittum2024doesconfidencebasedcascadedeferral} characterize when a
confidence statistic is sufficient for deferral. 
We derive first-order conditions for optimal thresholds in a multi-model
cascade and an identity that relates the gain of score-based escalation
to AUROC terms for the current and later models.

\section{The $k$-Model Cascade Framework}
\label{sec:framework}

\subsection{Setup and Cascade Mechanism}
\label{sec:setup}

Let $\mathcal{X}$ denote the space of queries and
let $x \sim \mathcal{P}$ be drawn from the deployment
distribution. A $k$-model cascade consists of an ordered model
pool $\{\mathcal{M}_1, \ldots, \mathcal{M}_k\}$ together with a
threshold vector $\boldsymbol{\tau} = (\tau_1, \ldots, \tau_{k-1})$ that
governs escalation between adjacent stages.
When queried, model $i$ incurs a random per-query cost $C_i(x)$ and
produces a response $y_i(x)$ with quality
$U_i(x) := u(y_i(x), y^*(x)),$
where $y^*(x)$ is the ground-truth answer for $x$, available only offline.
 Let $c_i \equiv \mathbb{E}[C_i(x)]$ and $\mathbb{E}[U_i]$ denote the expected cost and quality, respectively.
We order the models by expected cost.
We do not require their unconditional accuracies to increase along the model sequence.
Each non-terminal model $i < k$ additionally produces a scalar confidence
score $s_i(x)$, obtained from the same generation call at no additional cost.
Given $\boldsymbol{\tau}$, the cascade escalates from
model $i$ to model $i+1$ whenever $s_i(x) < \tau_i$, while the terminal
model $k$ always returns its response if stage $k$ is reached.\footnote{Hereafter, we carry the query argument $x$ explicitly in per-query random variables. All expectations and probabilities are over $x \sim \mathcal{P}$ unless otherwise noted.}

\paragraph{Stopping index, cost, and quality.} Let
$I(x;\boldsymbol{\tau}) := \min\bigl(\{j < k : s_j(x) \ge \tau_j\} \cup \{k\}\bigr)$
denote the stopping index on query $x$.
The cascade returns $y_{I(x;\boldsymbol{\tau})}(x)$, with output quality
and total cost
\[
    U(x;\boldsymbol{\tau}) := U_{I(x;\boldsymbol{\tau})}(x),
    \qquad
    C(x;\boldsymbol{\tau}) := \sum_{j=1}^{I(x;\boldsymbol{\tau})} C_j(x).
\]

\paragraph{Joint score distribution.}
Let $\mathbf{s}(x) = (s_1(x), \ldots, s_{k-1}(x))$ denote the vector of confidence scores,
and let $\mathbf{s}_{<i}(x) := (s_1(x), \ldots, s_{i-1}(x))$ denote the prefix through stage $i-1$. Under $x \sim \mathcal{P}$, $\mathbf{s}(x)$ is a random vector with joint density $f_{\mathbf{s}}$, which need not factorize due to shared dependence on query difficulty. We write $f_{s_i \mid \mathbf{s}_{<i}}$ for the conditional density of $s_i$ given the prefix, and suppress the $x$-argument for notational simplicity.

\paragraph{Conditional quality.} For non-terminal models $j < k$, let
$m_j(\mathbf{s}_{1:j}) := \mathbb{E}[U_j(x) \mid s_1(x), \ldots, s_j(x)]$
denote the expected quality of model $j$ conditional on the scores observed
through stage $j$. For the terminal model,
$m_k(\mathbf{s}_{<k}) := \mathbb{E}[U_k(x) \mid s_1(x), \ldots, s_{k-1}(x)]$.\footnote{We assume that
all conditional expectations (e.g., continuation probabilities and conditional accuracies)
are well-defined and continuous in the threshold $\tau$. This ensures existence of optimal
thresholds and justifies differentiation under the expectation.}

\paragraph{Downstream expected quality and cost.} Once a query has reached stage $i{+}1$, the remaining cascade produces some expected quality and incurs some expected cost, both conditional on the scores observed so far. We refer to these as the \emph{continuation values}. Let $I^{>i}(x;\boldsymbol{\tau}) := \min\bigl(\{j : i < j < k,\; s_j(x) \ge \tau_j\}
\cup \{k\}\bigr)$ denote the stage that would answer the query if the cascade resumed from stage $i{+}1$ onward. The quality and cost continuation values at stage $i{+}1$ are:
\begin{align}
V_{i+1}(\mathbf{s}_{1:i};\boldsymbol{\tau}) &:= \mathbb{E}\bigl[U_{I^{>i}(x;\boldsymbol{\tau})}(x) \;\big|\; \mathbf{s}_{1:i}(x)=\mathbf{s}_{1:i} \bigr],
\label{eq:V_def} \\
W_{i+1}(\mathbf{s}_{1:i};\boldsymbol{\tau}) &:= \mathbb{E}\!\left[\sum_{\ell=i+1}^{I^{>i}(x;\boldsymbol{\tau})} C_\ell(x) \;\middle|\; \mathbf{s}_{1:i}(x)=\mathbf{s}_{1:i} \right].
\label{eq:W_def}
\end{align}
Because confidence scores correlate with query difficulty, thresholds affect cascade quality in two ways: by changing which queries reach each stage, and by changing the difficulty composition of queries at each stage.

\paragraph{The constrained optimization.} 
The practitioner wishes to maximize quality for a given budget, or conversely, minimize cost for a given quality target, which gives
 two constrained formulations:
\begin{align}
    \textbf{(P1)} \quad
        &\min_{\boldsymbol{\tau}} \;\mathbb{E}[C(x;\boldsymbol{\tau})]
        \quad \text{s.t.} \quad \mathbb{E}[U(x;\boldsymbol{\tau})] \geq Q,
        \label{eq:P1} \\
    \textbf{(P2)} \quad
        &\max_{\boldsymbol{\tau}} \;\mathbb{E}[U(x;\boldsymbol{\tau})]
        \quad \text{s.t.} \quad \mathbb{E}[C(x;\boldsymbol{\tau})] \leq B.
        \label{eq:P2}
\end{align}
Both describe the same efficient tradeoff geometrically. On concave
segments, standard Lagrangian duality applies with the shadow-price
interpretations developed in Appendix~\ref{app:two-model-geometry}. We focus on
\textbf{(P2)} below.

\paragraph{Pareto frontier.} The \emph{Pareto frontier} of the cascade is
the value function of \textbf{(P2)} as the budget varies:
\begin{equation}
    U^*(B) := \sup_{\boldsymbol{\tau} \in [0,1]^{k-1}}
    \bigl\{\mathbb{E}[U(x;\boldsymbol{\tau})] :
    \mathbb{E}[C(x;\boldsymbol{\tau})] \leq B\bigr\}.
    \label{eq:pareto_frontier}
\end{equation}

The structural results below follow from how expected escalation 
benefit, $V_{i+1}(\mathbf{s}_{1:i};\boldsymbol{\tau}) - 
m_i(\mathbf{s}_{1:i})$, varies with the boundary score $s_i$. 
The integral forms used in the proofs appear in 
Appendix~\ref{app:foc-proof}.

\subsection{First-Order Optimality Conditions}
\label{sec:foc}

For a fixed model sequence, the following theorem gives necessary
stationarity conditions for threshold allocation. Because the continuation values can be estimated from calibration data,
these conditions can guide threshold search in practice
(Section~\ref{sec:exp-foc-validation}). Appendix~\ref{app:p1_fonc}
gives the complete KKT conditions for \textbf{(P1)} and \textbf{(P2)}.

\begin{theorem}[First-Order Conditions for a Fixed Model Sequence]
\label{thm:foc}
Consider problem \textbf{(P2)} with Lagrangian
$$
\mathcal{L}(\boldsymbol{\tau},\lambda) = \mathbb{E}[U(x;\boldsymbol{\tau})] - \lambda(\mathbb{E}[C(x;\boldsymbol{\tau})] - B),
$$
so that $\lambda \ge 0$ is the Lagrange multiplier (units of quality per unit cost). Under regularity conditions on $f_{\mathbf{s}}$ (stated precisely in Appendix~\ref{app:supplementary_theory}), at an interior optimum each threshold $\tau_i$ satisfies:
\begin{equation}
\underbrace{\mathbb{E}\Bigl[
V_{i+1}(\mathbf{s}_{1:i};\boldsymbol{\tau}) - m_i(\mathbf{s}_{1:i})
\;\Big|\;
\mathbf{s}_{<i} < \boldsymbol{\tau}_{<i},\; s_i = \tau_i
\Bigr]}_{\text{decision-boundary expected escalation benefit at stage } i}
=\;
\lambda\;
\underbrace{\mathbb{E}\Bigl[
W_{i+1}(\mathbf{s}_{1:i};\boldsymbol{\tau})
\;\Big|\;
\mathbf{s}_{<i} < \boldsymbol{\tau}_{<i},\; s_i = \tau_i
\Bigr]}_{\text{decision-boundary expected downstream cost at stage } i}.
\label{eq:foc}
\end{equation}
\end{theorem}

The proof and boundary conditions appear in Appendix~\ref{app:foc-proof}. 
Appendix~\ref{app:monotonicity} gives a condition under which expected cost and quality are monotone in
each threshold. For two-model cascades, Appendix~\ref{app:two-model-geometry} gives conditions under which the
frontier is concave and the first-order conditions are sufficient for an optimum.
Appendix~\ref{app:theory-checks} checks these conditions
empirically.

\paragraph{Economic interpretation.}
At an interior optimum, the ratio of 
expected accuracy gain to expected downstream cost is the same at every threshold.
Intuitively, if one boundary offered a higher ratio than another, reallocating 
escalations toward the high-ratio boundary and away from the low-ratio one 
would improve accuracy without increasing cost. 

\subsection{Selection Among Subsequences}
\label{sec:depth-value}


\paragraph{Subsequence optimization.}
The first-order conditions characterize threshold allocation within a
fixed model sequence. Following \citet{chen2023frugalgptuselargelanguage}, we
also optimize the subsequence selected from the model pool. Let
$\mathcal{M} = \{\mathcal{M}_1, \ldots, \mathcal{M}_K\}$ be a fixed
pool ordered by expected generation cost.
For a maximum depth $d \leq K$, let $S_d$ contain every cascade
$(\mathbf{i}, \boldsymbol{\tau})$, where $\mathbf{i}$ lists the indices
of a cost-ordered subsequence of length one through $d$ and
$\boldsymbol{\tau}$ its thresholds, and let $U_{\mathbf{i}}(x;\boldsymbol{\tau})$
and $C_{\mathbf{i}}(x;\boldsymbol{\tau})$ denote the quality and cost of
Section~\ref{sec:setup} for that cascade. The optimal frontier under depth cap $d$ is
\begin{equation}
U_d^*(B) = \sup_{(\mathbf{i}, \boldsymbol{\tau}) \in S_d}
\left\{ \mathbb{E}[U_{\mathbf{i}}(x; \boldsymbol{\tau})] :
\mathbb{E}[C_{\mathbf{i}}(x; \boldsymbol{\tau})] \leq B \right\}.
\label{eq:depth-frontier}
\end{equation}
Both the chosen models and their thresholds may change with the budget or
depth cap, and shorter model sequences are always permitted.

\paragraph{Value of additional depth.}
Because $S_{d-1} \subset S_d$, calibration-set 
accuracy cannot decrease with depth. However, selecting from a larger 
policy class increases overfitting risk, so test-set accuracy can 
decrease. We measure the value of depth beyond pairs as
$G_d(B) = U_d^*(B) - U_2^*(B)$, which is nonnegative by this nesting
argument. The experiments in Section~\ref{sec:exp-frontier} measure the
held-out counterpart of $G_d$ across depth caps of three through five.

\subsection{Correctness Discrimination and Cascade Value}
\label{sec:signal-value}

Escalation improves accuracy only on queries
where the current model is incorrect and the downstream answer is correct. Below we derive an identity 
that relates score discrimination to cascade accuracy. In this subsection, we specialize to binary correctness,
$U_j(x)\in\{0,1\}$ for every model $j$. Fix a model sequence, a nonterminal stage $i$, and all thresholds other than $\tau_i$.
Let $p_i=\Pr(\mathbf s_{<i}<\boldsymbol\tau_{<i})>0$ be the probability that a
query reaches stage $i$, and let $U_{>i}(x):=U_{I^{>i}(x;\boldsymbol\tau)}(x)$
denote the quality of the answer returned if the query escalates from stage
$i$, so that $V_{i+1}=\mathbb E[U_{>i}\mid\mathbf s_{1:i}]$. Note that neither the set
of queries reaching stage $i$ nor $U_{>i}$ depends on $\tau_i$. Order the
queries reaching stage $i$ by increasing $s_i$ and let $a_i(t)$ denote overall
cascade accuracy when the bottom fraction $t\in[0,1]$ of these queries
escalate. Let $\bar a_i(t)=(1-t)\,a_i(0)+t\,a_i(1)$, which is the accuracy
obtained when the same fraction escalates at random. For a binary variable $Y$, let
\[
\Gamma_i(Y):=\pi_Y(1-\pi_Y)
\bigl(\operatorname{AUROC}\bigl(s_i,Y\mid\mathbf s_{<i}<\boldsymbol\tau_{<i}\bigr)-\tfrac12\bigr),
\qquad
\pi_Y:=\mathbb E\bigl[Y\mid\mathbf s_{<i}<\boldsymbol\tau_{<i}\bigr],
\]
if $0<\pi_Y<1$, and let $\Gamma_i(Y)=0$ otherwise.

\begin{proposition}[AUROC and Accuracy Gain over Random Escalation]
\label{prop:auc-gain}
For binary correctness and continuous scores,
\begin{equation}
\begin{aligned}
\int_0^1\bigl[a_i(t)-\bar a_i(t)\bigr]\,dt
&= p_i\bigl[\Gamma_i(U_i)-\Gamma_i(U_{>i})\bigr] \\
&= p_i\bigl[\Gamma_i\bigl(U_i(1-U_{>i})\bigr)
-\Gamma_i\bigl((1-U_i)U_{>i}\bigr)\bigr].
\end{aligned}
\label{eq:auc-gain}
\end{equation}
\end{proposition}

A proof is given in Appendix~\ref{app:signal-quality-proof}. 
The left side is the accuracy gain of score-based escalation over random
escalation, averaged over escalation rates.\footnote{The identity compares
policies at equal escalation rates, which need not imply equal expected costs.
Accuracy at a single threshold depends on how that threshold separates both
current-model and downstream correctness, so AUROC alone does not determine
the cost-accuracy frontier.} The first expression compares how well $s_i$
discriminates correctness for the current model and for the downstream
answer, each weighted by the variance of the corresponding correctness
variable. A score that reflects query difficulty shared by both can raise
both terms, so a high $\operatorname{AUROC}(s_i,U_i)$ does not by itself
imply a large gain. The second expression isolates queries where escalation
changes correctness. Ranking queries where escalation hurts ($U_i=1$,
$U_{>i}=0$) higher, or queries where it helps ($U_i=0$, $U_{>i}=1$) lower,
increases the average gain. For $k=2$, we have $p_1=1$ and $U_{>1}=U_2$.
Holding $p_i$,
both conditional accuracies, and $\operatorname{AUROC}(s_i,U_i)$ fixed, the
average gain decreases in $\operatorname{AUROC}(s_i,U_{>i})$.


\section{Experiments}
\label{sec:experiments}

\subsection{Experimental Setup}
\label{sec:exp-setup}
We conduct experiments using a pool of eight models (Llama 3.1-8B, Qwen2.5-7B,
GPT-4o mini, GPT-oss-20B, DeepSeek-V3, Llama 3.3-70B, GPT-4o, and MiniMax-M2.7)
on five benchmark datasets (MATH levels 3--5 \citep{hendrycks2021measuringmathematicalproblemsolving}, MMLU \citep{hendrycks2021measuring}, TriviaQA \citep{joshi2017triviaqalargescaledistantly}, SimpleQA \citep{wei2024measuringshortformfactualitylarge}, and LiveCodeBench \citep{jain2024livecodebenchholisticcontaminationfree}).
Llama 3.1-8B has the lowest average cost per query on all five datasets, while the highest-accuracy model varies by dataset.
Single-model operating points and grading are described in Appendices~\ref{app:costs} and~\ref{app:grading}.
We evaluate nine deferral scores.\footnote{We do not consider confidence scores based on self-consistency
\citep{manakul2023selfcheckgptzeroresourceblackboxhallucination},
post-hoc verbalized confidence \citep{tian2023justaskcalibrationstrategies},
or LLM-as-a-Judge, because they require additional generation calls,
which add to the cost that cascading aims to reduce. Section~\ref{sec:exp-counterfactual}
instead studies higher-quality scores through simulation.} Five are derived from log-probabilities, as used by many prior studies
 \citep{gupta2024languagemodelcascadestokenlevel, jitkrittum2024doesconfidencebasedcascadedeferral,zellinger2025rational,zellinger2025costsavingllmcascadesearly}: mean token
negentropy, minimum token negentropy, probability margin, minimum token probability,
and length-normalized probability (definitions in Appendix~\ref{app:scores}). The remaining four are learned. Two are logistic regressions that predict
the current model's correctness, one over the five token-probability-based
scores, which combines token-level uncertainty as in
\citet{gupta2024languagemodelcascadestokenlevel}, and one over frozen
embeddings of the prompt and response, following
\citet{chen2023frugalgptuselargelanguage}. The other two use the same
inputs but fit ridge regressions to the difference in correctness between
the next model and the current model, which is the Diff-01 target of
\citet{jitkrittum2024doesconfidencebasedcascadedeferral}. We refer to the
first seven as confidence-based scores.

We compare the nested classes $S_1 \subseteq \cdots \subseteq S_5$ from
Section~\ref{sec:depth-value} and, with a slight abuse of notation, write $S_d - S_{d'}$ for the difference
in test-set accuracy between the policies selected from $S_d$ and
$S_{d'}$. At each budget, we select
the feasible policy in each class with the highest calibration-set
accuracy and freeze its sequence and thresholds before evaluating on
test-set data. Selection from $S_1$ therefore gives the best
budget-feasible single model.
We use 50 stratified random splits, assigning half the queries to the calibration set and half to the test set.
We evaluate at 500 evenly spaced budgets from the average cost of the cheapest
model to that of the highest-accuracy model.
Threshold selection uses exhaustive search over all calibration-set threshold placements, detailed in Appendix~\ref{app:exact-depth}. 
Section~\ref{sec:exp-foc-validation} evaluates this against a local search guided by the first-order conditions. 
Costs use the token prices in Appendix~\ref{app:costs}.

\subsection{Empirical Check of the First-Order Conditions}
\label{sec:exp-foc-validation}

Here, we check whether the first-order conditions in Theorem~\ref{thm:foc} are a usable guide to threshold placement (details in Appendix~\ref{app:foc-optimizer}). 
In particular, we build a local search that estimates the boundary gain-to-cost ratio in Equation~\ref{eq:foc} at each threshold from calibration queries with scores near that threshold. 
The search proposes moves that shift escalations from low-ratio to high-ratio boundaries and accepts a move when it improves calibration accuracy or holds accuracy at lower cost.
To ensure fair comparison, we retain the model pool, calibration splits, budgets, policy classes, and selection objective of the exhaustive search. Across datasets and depth caps, 
its mean calibration-set accuracy falls short of the exhaustive optimum by at most 0.07 percentage points, and its mean test-set accuracy remains within 0.12 percentage points of exhaustive search in all 20 comparisons (Table~\ref{tab:foc-optimizer-heldout}). 

\begin{table}[t]
\centering
\small
\caption{Mean test-set accuracy gains across deferral scores, in percentage
points. The cascade-gain panel reports $S_2-S_1$, relative to the best
budget-feasible single model. The depth-gain panel reports $S_d-S_2$ for
$d=3,4,5$, relative to two-model selection. Differences are integrated over
500 assigned budget positions within each split, then averaged equally across
50 splits. Positive values indicate accuracy gains. Within-dataset comparisons use matched
frozen splits and budget positions with exhaustive threshold optimization.}
\label{tab:summary}
\setlength{\tabcolsep}{3pt}
\resizebox{\linewidth}{!}{\begin{tabular}{llrrrrrrrrr}
\toprule
 & & \multicolumn{5}{c}{Token-probability scores}
   & \multicolumn{2}{c}{Correctness-trained}
   & \multicolumn{2}{c}{Benefit-trained} \\
\cmidrule(lr){3-7}\cmidrule(lr){8-9}\cmidrule(lr){10-11}
Dataset & Comparison
 & \shortstack{Mean\\negentropy} & \shortstack{Min\\negentropy}
 & \shortstack{Prob.\\margin} & \shortstack{Min\\probability}
 & \shortstack{Length-norm.\\probability}
 & \shortstack{Token-\\prob.} & Embedding
 & \shortstack{Token-\\prob.} & Embedding \\
\midrule
\multicolumn{11}{l}{\textbf{Cascade gain}} \\
MMLU & $S_2-S_1$ & $+3.93$ & $+3.58$ & $+3.88$ & $+3.50$ & $+3.88$ & $+3.90$ & $+2.33$ & $+3.89$ & $+2.30$ \\
TriviaQA & $S_2-S_1$ & $+5.72$ & $+6.38$ & $+5.27$ & $+6.50$ & $+6.00$ & $+6.61$ & $+5.27$ & $+6.62$ & $+7.20$ \\
MATH & $S_2-S_1$ & $+2.80$ & $+2.50$ & $+2.81$ & $+2.33$ & $+2.80$ & $+2.44$ & $+2.13$ & $+3.02$ & $+3.46$ \\
SimpleQA & $S_2-S_1$ & $+8.22$ & $+6.97$ & $+8.18$ & $+7.08$ & $+7.73$ & $+7.45$ & $+6.18$ & $+7.79$ & $+7.60$ \\
LiveCodeBench & $S_2-S_1$ & $+7.01$ & $+6.12$ & $+7.09$ & $+5.79$ & $+6.60$ & $+5.98$ & $+4.65$ & $+6.48$ & $+7.32$ \\
\midrule
\multicolumn{11}{l}{\textbf{Depth gain}} \\
MMLU & $S_3-S_2$ & $-0.05$ & $+0.27$ & $-0.07$ & $-0.08$ & $-0.04$ & $-0.07$ & $-0.02$ & $-0.13$ & $-0.95$ \\
 & $S_4-S_2$ & $-0.17$ & $+0.20$ & $-0.15$ & $-0.19$ & $-0.15$ & $-0.18$ & $-0.10$ & $-0.21$ & $-1.33$ \\
 & $S_5-S_2$ & $-0.24$ & $+0.12$ & $-0.20$ & $-0.19$ & $-0.19$ & $-0.23$ & $-0.33$ & $-0.22$ & $-1.57$ \\
\addlinespace[3pt]
TriviaQA & $S_3-S_2$ & $+0.05$ & $+0.19$ & $-0.11$ & $+0.13$ & $+0.04$ & $+0.16$ & $-0.50$ & $+0.43$ & $-0.31$ \\
 & $S_4-S_2$ & $-0.03$ & $+0.11$ & $-0.16$ & $+0.03$ & $-0.04$ & $+0.07$ & $-0.74$ & $+0.40$ & $-0.96$ \\
 & $S_5-S_2$ & $-0.03$ & $+0.11$ & $-0.17$ & $+0.03$ & $-0.05$ & $+0.06$ & $-0.81$ & $+0.39$ & $-1.05$ \\
\addlinespace[3pt]
MATH & $S_3-S_2$ & $-0.08$ & $-0.10$ & $-0.11$ & $-0.11$ & $-0.08$ & $-0.10$ & $-0.25$ & $-0.07$ & $-1.52$ \\
 & $S_4-S_2$ & $-0.09$ & $-0.16$ & $-0.15$ & $-0.18$ & $-0.09$ & $-0.14$ & $-0.42$ & $-0.09$ & $-2.01$ \\
 & $S_5-S_2$ & $-0.10$ & $-0.17$ & $-0.15$ & $-0.21$ & $-0.09$ & $-0.16$ & $-0.51$ & $-0.11$ & $-2.32$ \\
\addlinespace[3pt]
SimpleQA & $S_3-S_2$ & $-0.10$ & $-0.10$ & $-0.02$ & $-0.12$ & $-0.13$ & $-0.08$ & $-1.02$ & $-0.16$ & $-2.24$ \\
 & $S_4-S_2$ & $-0.10$ & $-0.12$ & $-0.03$ & $-0.13$ & $-0.15$ & $-0.11$ & $-1.59$ & $-0.18$ & $-3.56$ \\
 & $S_5-S_2$ & $-0.10$ & $-0.12$ & $-0.03$ & $-0.13$ & $-0.15$ & $-0.11$ & $-1.76$ & $-0.18$ & $-4.49$ \\
\addlinespace[3pt]
LiveCodeBench & $S_3-S_2$ & $-0.03$ & $-0.04$ & $-0.02$ & $-0.01$ & $-0.04$ & $-0.03$ & $-0.06$ & $-0.07$ & $-0.21$ \\
 & $S_4-S_2$ & $-0.03$ & $-0.05$ & $-0.02$ & $-0.01$ & $-0.05$ & $-0.03$ & $-0.06$ & $-0.07$ & $-0.21$ \\
 & $S_5-S_2$ & $-0.03$ & $-0.05$ & $-0.02$ & $-0.01$ & $-0.05$ & $-0.03$ & $-0.06$ & $-0.07$ & $-0.21$ \\
\bottomrule
\end{tabular}
}
\end{table}

\subsection{The Value of Depth}
\label{sec:exp-frontier}

Table~\ref{tab:summary} reports budget-averaged differences in test-set
accuracy between depth caps for all nine deferral scores.
Across scorers and datasets, two-model cascades improve mean test-set accuracy
over the best budget-feasible single model by 2.1 to 8.2 percentage points, 
consistent with the large cascade gains reported in prior work 
\citep{chen2023frugalgptuselargelanguage,gupta2024languagemodelcascadestokenlevel,aggarwal2025automixautomaticallymixinglanguage}.

We measure the value of additional depth by averaging the accuracy
difference between each depth cap and two-model selection over the budget
range within each split. Across scorer-dataset combinations, mean test-set
accuracy differences between $S_3$ and $S_2$ range from $-2.24$ to $+0.43$
percentage points and are negative in 38 of 45 cases. Allowing four or
five models gives negative mean differences in 80 of 90 cases, and the
positive differences never exceed 0.40 percentage points. Scores trained to
predict escalation benefit generally do not increase depth gains, except
for the token-probability version on TriviaQA.

Calibration search frequently selects longer sequences when deeper
classes are available, with three-model sequences accounting for
8--89\% of budgets under $S_3$ depending on the dataset
(Table~\ref{tab:exact-usage}). Because the classes are nested, calibration-set accuracy weakly
increases with the depth cap, but the deeper policies generally do not
improve on $S_2$ in test-set accuracy.
This pattern suggests some degree of overfitting during policy selection. However, when 
we select and evaluate policies on the same full sample, so that
overfitting cannot reduce the gain, a third model still adds at most 0.9
percentage points at any single budget and at most 0.5 points on more than
99\% of budgets on four of the five datasets (Appendix~\ref{app:oracle-depth}).

\subsection{Counterfactual Experiments: Confidence Scores and Model Prices}
\label{sec:exp-counterfactual}

We next study two factors that can affect the value of cascade depth:
how the confidence scores relate to model correctness, and how model
prices are spaced within the pool.

\paragraph{Synthetic confidence scores.}
We simulate confidence scores at target correctness AUROC levels of $0.5, 0.6, \ldots, 0.9$, 
preserving the cross-model correlation of
the observed mean token negentropy scores
(Appendix~\ref{app:synthetic-scores}).\footnote{For comparison, across datasets and the seven confidence-based scores,
the first-stage scores of selected two-model policies have test-set
correctness AUROC between $0.63$ and $0.80$
(Table~\ref{tab:observed-stage}).}
We then repeat the exhaustive search over $S_1$ through $S_4$ using the same 50 calibration-test splits and 500 budget
positions as the primary experiments. Figure~\ref{fig:kmodel_vs_envelope} shows the cascade frontiers under 
observed scores and synthetic scores at AUROC $0.9$, where the deeper 
cascades visibly separate from $S_2$ on four benchmarks.

Under these synthetic scores, raising correctness AUROC makes a 
third stage beneficial on four of five benchmarks 
(Figure~\ref{fig:signal-quality}). At low AUROC levels ($0.5$ and $0.6$), 
the mean $S_3 - S_2$ gain is negative on every dataset. 
However, at higher AUROC levels, the gain becomes positive on four benchmarks, ranging
at AUROC $0.9$ from $+0.44$ percentage points on MATH to $+1.07$
percentage points on TriviaQA. LiveCodeBench remains negative at every 
tested level. At AUROC $0.9$, $S_4$ also improves on $S_2$ on the same 
four benchmarks, with mean gains of $0.43$ to $1.42$ percentage points, 
though the marginal gain from a fourth model beyond the third is smaller 
and inconsistent in sign across datasets. These findings are largely consistent across five additional independent
score draws (Appendix~\ref{app:synthetic-scores}), where the mean
third-stage gain at AUROC $0.9$ is positive in every draw on the same
four benchmarks.

\begin{figure}[t]
\centering
\includegraphics[width=\linewidth]{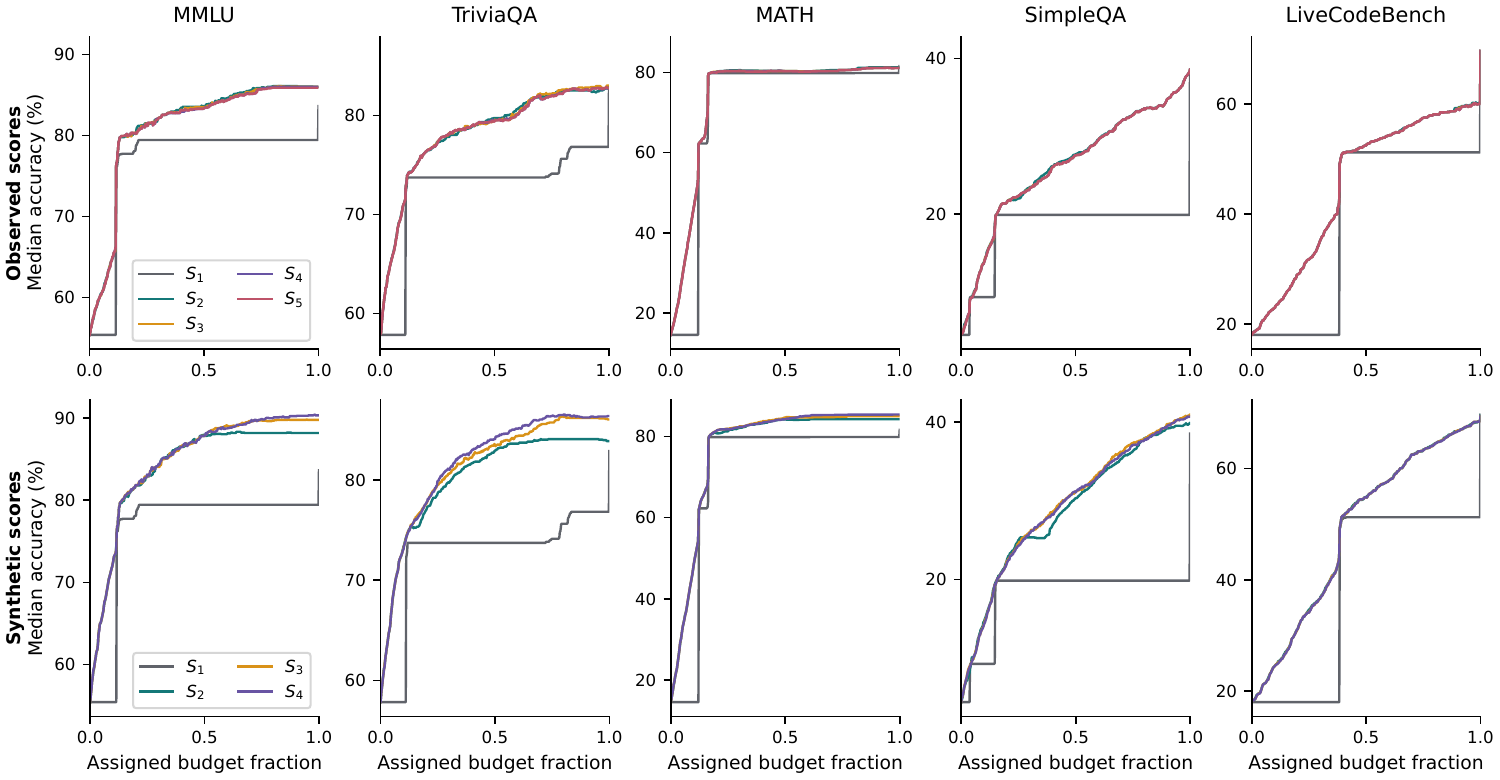}
\caption{\textbf{Synthetic scores built from each model's own correctness make additional depth useful on four of five benchmarks at AUROC $0.9$.}
Top row shows observed mean token negentropy scores, with median test accuracy
for the best budget-feasible single model ($S_1$) and cascades of at most two
through five models ($S_2$--$S_5$). Bottom row shows synthetic confidence
scores with target correctness AUROC $0.9$, evaluated through $S_4$. Policies are selected on calibration data and evaluated at their assigned budgets.
Accuracy scales match within each dataset.}
\label{fig:kmodel_vs_envelope}
\end{figure}

\paragraph{Source of the confidence signal.}
The scores above are each constructed to predict whether their own model
answered correctly, so they relate to later models' correctness only
through the correlation of correctness across models. They should
therefore separate correct from incorrect answers better for the current model than for later models, and by
Proposition~\ref{prop:auc-gain}, this favors escalation.
To test this, we add a query difficulty term to the construction, measured as the error rate of the other seven models on the query. A weight $w$ sets the share of the signal that comes from difficulty, and $w=0$ recovers the scores above. 
We recalibrate each score to the same target correctness AUROC and repeat the exhaustive search over $S_2$ and $S_3$ (Appendix~\ref{app:shared-difficulty}). 
At target AUROC $0.9$ (Table~\ref{tab:shared-difficulty}), the
first-stage score's AUROC for the downstream answer's correctness rises
with $w$, and the gain over random escalation falls on every dataset.
Table~\ref{tab:observed-stage} reports the same quantities for the seven
observed confidence-based scores.
On the four benchmarks with positive gains at $w=0$, the mean third-stage gain 
falls from between $+0.44$ and $+1.07$ percentage points to 
between $-0.15$ and $-0.04$ at the largest feasible $w$. 
We also find that the pair gain $S_2-S_1$ falls as $w$ increases on every dataset, but it remains positive.
These results are
consistent with Proposition~\ref{prop:auc-gain}, which shows that a 
cascade's gain over random escalation depends on how well the confidence score separates correct from
incorrect answers for the current model compared with later models.

\begin{table}[t]
\centering
\small
\caption{Third-stage gains at target correctness AUROC $0.9$ as the
score's signal shifts from the model's own correctness ($w=0$) toward
query difficulty shared across models ($w>0$). $S_3-S_2$ is the
mean test-set accuracy difference in percentage points between $S_3$ and $S_2$,
integrated over budgets and averaged over 50 splits. Downstream AUROC is
the AUROC of the first-stage score for the downstream answer's
correctness, $\operatorname{AUROC}(s_1,U_{>1})$. Gain over random
escalation is the right side of Equation~\ref{eq:auc-gain} in percentage
points. Both are computed at the first stage of each selected $S_3$ policy on
test-set queries and averaged over budgets and splits. The first-stage score's AUROC for its own model's
correctness rounds to $0.90$ in every cell. The largest feasible $w$
ranges from $0.70$ to $0.87$ across datasets
(Appendix~\ref{app:shared-difficulty}). Results use one score draw.}
\label{tab:shared-difficulty}
\setlength{\tabcolsep}{3pt}
\begin{tabular*}{\linewidth}{@{\extracolsep{\fill}}lrrrrrrrrr@{}}
\toprule
& \multicolumn{3}{c}{$S_3-S_2$ (pp)} & \multicolumn{3}{c}{Downstream AUROC} & \multicolumn{3}{c}{\shortstack{Gain over random\\escalation (pp)}} \\
\cmidrule(lr){2-4}\cmidrule(lr){5-7}\cmidrule(l){8-10}
Dataset \quad $w$ & $0$ & $0.5$ & max & $0$ & $0.5$ & max & $0$ & $0.5$ & max \\
\midrule
MMLU & $+0.81$ & $+0.20$ & $-0.15$ & $0.68$ & $0.79$ & $0.89$ & $4.5$ & $2.8$ & $1.4$ \\
TriviaQA & $+1.07$ & $+0.03$ & $-0.04$ & $0.71$ & $0.83$ & $0.91$ & $5.7$ & $3.0$ & $2.0$ \\
MATH & $+0.44$ & $-0.11$ & $-0.05$ & $0.67$ & $0.79$ & $0.93$ & $3.9$ & $1.7$ & $-1.1$ \\
SimpleQA & $+0.57$ & $-0.21$ & $-0.08$ & $0.59$ & $0.69$ & $0.79$ & $3.2$ & $0.4$ & $-1.8$ \\
LiveCodeBench & $-0.07$ & $-0.07$ & $-0.03$ & $0.64$ & $0.76$ & $0.89$ & $5.3$ & $2.5$ & $-0.1$ \\
\bottomrule
\end{tabular*}

\end{table}

\paragraph{Cost structure.}
Our model pool has a particular set of prices, and the value of a third stage
could depend on how the intermediate models are spaced in cost between
the cheapest and most expensive model.  We test this by replacing observed prices with
synthetic ones, holding the correctness labels, confidence scores, and
recorded token counts fixed. The cheapest and most expensive models keep
their observed expected costs, and we vary the positions of the six
intermediate models in between. 
We set positions by equally spaced quantiles of a $\mathrm{Beta}(a,b)$ distribution on normalized log cost, 
which moves the intermediate models toward the endpoints, toward the middle, or toward the cheap or expensive end. We test
nine settings of $(a,b)$ and repeat the exhaustive search over $S_2$ and $S_3$ using the
same 50 splits and 500 budget positions as the primary experiments.
Appendix~\ref{app:cost-structure} gives the construction, the settings,
and the fitted positions of the observed pools.

The results indicate that the spacing of model prices matters less with observed scores than with
high-AUROC synthetic scores (Figure~\ref{fig:cost-structure}). 
With observed scores, the mean $S_3 - S_2$ difference on MMLU, MATH, SimpleQA, and
LiveCodeBench falls between $-0.17$ and $-0.01$ percentage points across all nine
settings, while TriviaQA is positive at every setting. We repeat the sweep with the AUROC $0.9$ synthetic scores, which
made a third stage beneficial on four of five benchmarks in the
synthetic-score experiments above.
With these scores, the nine settings change the third-stage gain by $0.67$ to $1.40$
percentage points on MMLU, TriviaQA, MATH, and SimpleQA, and the gain on
SimpleQA changes sign. LiveCodeBench remains negative in every setting.
The  setting with the largest gain differs across datasets, so we do not
identify a cost spacing that favors depth in general.

\begin{figure}[t]
\centering
\includegraphics[width=\linewidth]{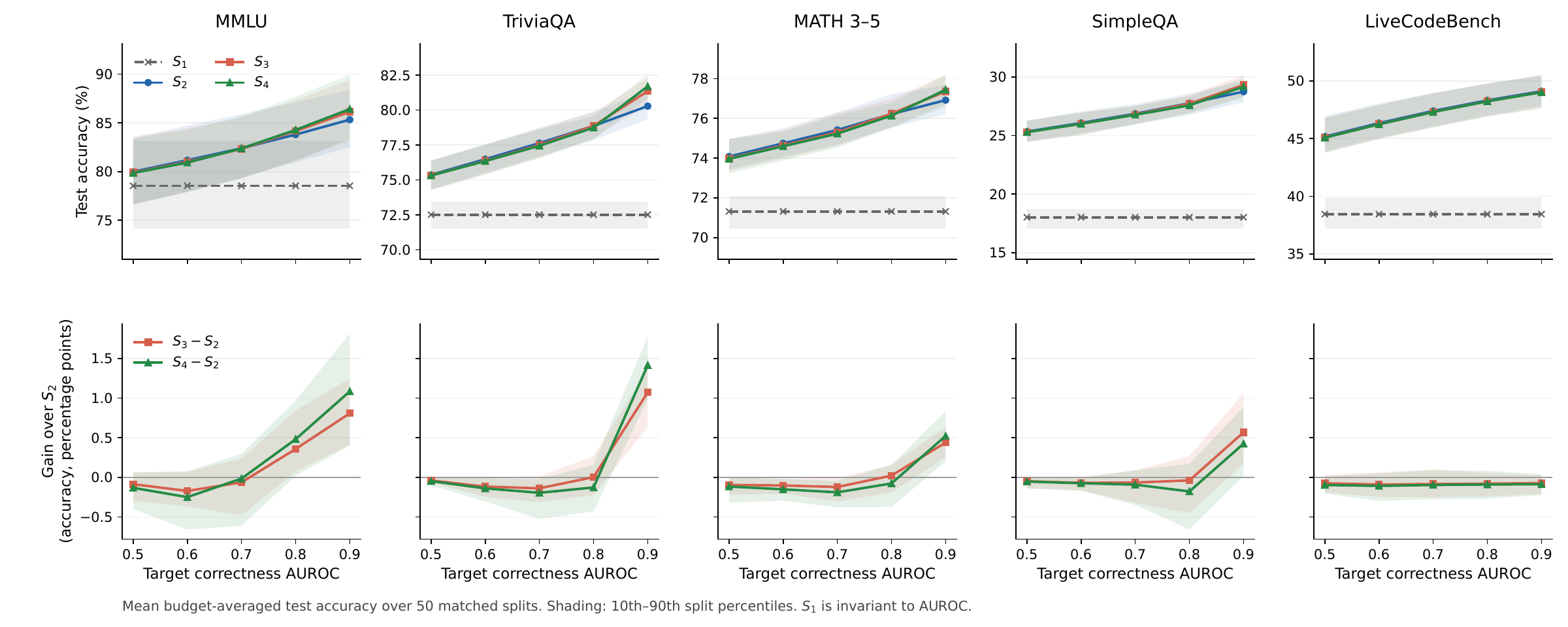}
\caption{\textbf{Synthetic scores built from each model's own correctness make a third stage beneficial on four of five benchmarks at high correctness AUROC.}
Top row shows mean test accuracy for $S_1$ through $S_4$ by target
correctness AUROC, with a separate accuracy scale for each dataset.
Bottom row shows mean accuracy gains of $S_3$ and $S_4$ over $S_2$.
Accuracy is averaged over assigned budgets and 50 matched splits.
Shading shows 10th--90th split percentiles. $S_1$ is constant across AUROC.
Scores match the observed
full-sample Pearson correlation matrix of mean token negentropy.}
\label{fig:signal-quality}
\end{figure}

\FloatBarrier

\section{Conclusion and Limitations}

We study the value of additional depth in an LLM cascade. We derive first-order optimality 
conditions for threshold placement and show that a local search using the 
gain-to-cost ratio from these conditions closely matches exhaustive optimization. 
We also derive an identity 
decomposing cascade accuracy gain into two AUROC terms. We conduct experiments across five 
benchmarks and nine deferral scores and find that two-model cascades substantially 
improve over single-model selection, but allowing more than two models 
does not meaningfully improve average test accuracy. 
Varying the spacing of model prices has
little effect on this result under observed deferral scores.
When deferral scores are simulated with an AUROC of 0.9 for predicting
whether their own model answered correctly, a third stage improves
test-set accuracy on four of five benchmarks, but these gains do not
persist when the scores also reflect query difficulty shared across
models.
Overall, our findings suggest that depth gains depend on the tradeoff
between the two AUROC terms, rather than on correctness AUROC or cost
spacing alone, and that with conventional deferral scores, a
well-optimized two-model cascade captures nearly all of the available
gain.

We note several limitations. First, our results concern deterministic
threshold cascades that defer on a score computed from each model's own
output. Adaptive routing and router-cascade hybrid policies may behave
differently. Second, our experiments use a single pool of eight models and
five benchmarks with binary correctness, and results
may differ for other model pools and for tasks with continuous output quality.
Third, the synthetic scores 
follow one construction, which uses
full-sample correctness labels. Other constructions may produce different
depth gains. Lastly, our depth comparisons focus on budget-averaged
accuracy. Practitioners operating at a single budget may find that a deeper
cascade outperforms the best pair at that budget.

\bibliographystyle{iclr2026_conference}
\bibliography{../main/tmlr}

\section*{Disclosure of LLM usage}
LLM-based tools assisted with manuscript drafting and editing, code development,
and preparation of reproducibility materials. The authors take responsibility
for the final manuscript, experimental procedures, reported results, and
conclusions.

\appendix

\section{Supplementary Theoretical Analysis}
\label{app:supplementary_theory}

\subsection{Regularity conditions and proof of Theorem~\ref{thm:foc}}
\label{app:foc-proof}
We assume the joint density $f_{\mathbf{s}}$ is continuous and strictly
positive on the interior of its support, and that $m_i$, $V_{i+1}$, and
$W_{i+1}$ are continuous in $(s_{1:i}, \boldsymbol{\tau})$. Boundary optima require replacing stationarity with one-sided inequalities.

\paragraph{Integral form over the confidence support.} The expectations
$\mathbb{E}[U(x;\boldsymbol{\tau})]$ and
$\mathbb{E}[C(x;\boldsymbol{\tau})]$ can be written as integrals over
the confidence support. First, define the
stopping regions
\[
R_i(\boldsymbol{\tau}) =
\begin{cases}
\{\mathbf{s}: s_1 \ge \tau_1\}, & i = 1,\\
\{\mathbf{s}: \mathbf{s}_{<i} < \boldsymbol{\tau}_{<i},\; s_i \ge \tau_i\}, & 1 < i < k,\\
\{\mathbf{s}: \mathbf{s}_{<k} < \boldsymbol{\tau}_{<k}\}, & i = k,
\end{cases}
\]
where vector inequalities are elementwise. Then expected
quality decomposes stagewise as
\begin{equation}
\mathbb{E}[U(x;\boldsymbol{\tau})]
=
\sum_{i=1}^{k-1}
\int_{R_i(\boldsymbol{\tau})}
m_i(\mathbf{s}_{1:i}) f_{\mathbf{s}_{1:i}}(\mathbf{s}_{1:i})\,d\mathbf{s}_{1:i}
+
\int_{R_k(\boldsymbol{\tau})}
m_k(\mathbf{s}_{<k}) f_{\mathbf{s}_{<k}}(\mathbf{s}_{<k})\,d\mathbf{s}_{<k},
    \label{eq:integral_quality}
\end{equation}
and expected cost decomposes analogously as the sum of per-stage costs over reached regions:
\begin{equation}
    \mathbb{E}[C(x;\boldsymbol{\tau})] = c_1 + \sum_{i=2}^{k} \int_{\{\mathbf{s}_{<i} < \boldsymbol{\tau}_{<i}\}} \mathbb{E}[C_i(x) \mid \mathbf{s}_{<i}]\, f_{\mathbf{s}_{<i}}(\mathbf{s}_{<i}) \, d\mathbf{s}_{<i}.
    \label{eq:integral_cost}
\end{equation}

\begin{proof}[Proof sketch]
The threshold $\tau_i$ enters $\mathbb{E}[U(x;\boldsymbol{\tau})]$ and $\mathbb{E}[C(x;\boldsymbol{\tau})]$ only through the indicator $\mathbf{1}[s_i(x) < \tau_i]$ in the stopping regions $R_i, \ldots, R_k$.\footnote{The conditioning event $\{s_i(x)=\tau_i\}$ has measure zero.
Conditional expectations at the boundary are defined using the regular
conditional distribution induced by the joint density of
$(\mathbf{s}_{<i},s_i)$, equivalently by disintegration via the
conditional density of $s_i$ given the prefix-reaching event
$\mathbf{s}_{<i}<\boldsymbol{\tau}_{<i}$. The required densities are assumed
continuous and positive at the boundary under the regularity conditions of
Appendix~\ref{app:supplementary_theory}.}
Let $\mathcal{S}_{<i}(\boldsymbol{\tau}) = \{\mathbf{s}_{<i} : s_1 < \tau_1, \ldots, s_{i-1} < \tau_{i-1}\}$ denote the score prefix region of queries reaching stage $i$. Applying the Leibniz integral rule gives
\begin{align}
\frac{\partial \mathbb{E}[U(x;\boldsymbol{\tau})]}{\partial \tau_i} &= \int_{\mathcal{S}_{<i}} \bigl[V_{i+1}(\mathbf{s}_{<i},\tau_i;\boldsymbol{\tau}) - m_i(\mathbf{s}_{<i},\tau_i)\bigr]\, f_{\mathbf{s}_{1:i}}(\mathbf{s}_{<i},\tau_i)\,d\mathbf{s}_{<i}, \label{eq:grad_U} \\
\frac{\partial \mathbb{E}[C(x;\boldsymbol{\tau})]}{\partial \tau_i} &= \int_{\mathcal{S}_{<i}} W_{i+1}(\mathbf{s}_{<i},\tau_i;\boldsymbol{\tau})\, f_{\mathbf{s}_{1:i}}(\mathbf{s}_{<i},\tau_i)\,d\mathbf{s}_{<i}. \label{eq:grad_C}
\end{align}
where $f_{\mathbf{s}_{1:i}}$ is the marginal density of $(s_1(x),\ldots,s_i(x))$.\footnote{The optimal policy may lie at the boundary (i.e., always or never escalate). In such cases, the first-order condition is replaced by a one-sided condition. For clarity, we focus on interior optima.} Setting $\partial \mathcal{L}/\partial \tau_i = 0$ and dividing both sides by $\Pr(\mathbf{s}_{<i}(x) < \boldsymbol{\tau}_{<i}) \cdot f_{s_i \mid \mathbf{s}_{<i} < \boldsymbol{\tau}_{<i}}(\tau_i)$, which appears as a common factor, yields the first-order condition (FOC) in Equation~\ref{eq:foc}.
\end{proof}

\paragraph{Complete first-order conditions for the constrained problems.}
\label{app:p1_fonc}
Section~\ref{sec:foc} states the stationarity component of the first-order conditions for
the budget-constrained problem \textbf{(P2)}. For completeness, the KKT
conditions for an interior optimum of \textbf{(P2)} are
\begin{gather*}
    \mathbb{E}[C(x;\boldsymbol{\tau})] \leq B,
    \qquad
    \lambda \geq 0,\\
    \lambda\!\left(\mathbb{E}[C(x;\boldsymbol{\tau})]-B\right)=0.
\end{gather*}
These are primal feasibility, dual feasibility, and complementary slackness.
For each active threshold,
\[
\mathbb{E}\Bigl[
V_{i+1}(\mathbf{s}_{1:i};\boldsymbol{\tau}) - m_i(\mathbf{s}_{1:i})
\;\Big|\;
\mathbf{s}_{<i} < \boldsymbol{\tau}_{<i},\; s_i = \tau_i
\Bigr]
=
\lambda\,
\mathbb{E}\Bigl[
W_{i+1}(\mathbf{s}_{1:i};\boldsymbol{\tau})
\;\Big|\;
\mathbf{s}_{<i} < \boldsymbol{\tau}_{<i},\; s_i = \tau_i
\Bigr].
\]
Boundary optima replace stationarity with the appropriate one-sided
inequalities.

The quality-constrained problem \textbf{(P1)} uses the same marginal
objects with the reciprocal shadow price. Write its Lagrangian as
\[
    \mathcal{L}_{P1}(\boldsymbol{\tau},\mu)
    =
    \mathbb{E}[C(x;\boldsymbol{\tau})]
    + \mu\left(Q-\mathbb{E}[U(x;\boldsymbol{\tau})]\right),
\]
where $\mu \geq 0$ has units of cost per unit quality. Its KKT conditions
are
\begin{gather*}
    \mathbb{E}[U(x;\boldsymbol{\tau})] \geq Q,
    \qquad
    \mu \geq 0,\\
    \mu\!\left(Q-\mathbb{E}[U(x;\boldsymbol{\tau})]\right)=0.
\end{gather*}
These are primal feasibility, dual feasibility, and complementary slackness.
For each active threshold,
\begin{equation}
\mathbb{E}\Bigl[
W_{i+1}(\mathbf{s}_{1:i};\boldsymbol{\tau})
\;\Big|\;
\mathbf{s}_{<i} < \boldsymbol{\tau}_{<i},\; s_i = \tau_i
\Bigr]
=
\mu\,
\mathbb{E}\Bigl[
V_{i+1}(\mathbf{s}_{1:i};\boldsymbol{\tau}) - m_i(\mathbf{s}_{1:i})
\;\Big|\;
\mathbf{s}_{<i} < \boldsymbol{\tau}_{<i},\; s_i = \tau_i
\Bigr].
\label{eq:foc_p1}
\end{equation}
Thus the marginal cost-per-quality ratio is equalized across active stage
boundaries. Whenever the same interior frontier point solves
both constrained formulations, comparison with Equation~\ref{eq:foc} gives
$\mu = 1/\lambda$. In the two-model case, Equation~\ref{eq:foc_p1} reduces to
\[
    \gamma(\tau) = \mu\,(m_H(\tau)-m_L(\tau)),
\]
which is the P1 counterpart to Equation~\ref{eq:foc_2model}. Under score-independent
expected cost, $\gamma(\tau)=c_H$.

\subsection{Monotonicity}
\label{sec:monotone}
\label{app:monotonicity}

The FOC characterizes interior optima but does not by itself imply that the
cost-quality curve is monotone. The following \emph{decision-boundary
dominance} condition is a local sufficient condition requiring that, at each
active stage, marginally escalated queries benefit in expectation from continuation.

\begin{proposition}[Stagewise Monotonicity]
\label{prop:monotone}
Assume the regularity conditions of Theorem~\ref{thm:foc}, positive density at each reached decision boundary, and positive expected downstream cost. Suppose that, at a threshold vector $\boldsymbol{\tau}$,
\begin{equation}
\mathbb{E}\bigl[
V_{i+1}(\mathbf{s}_{1:i};\boldsymbol{\tau})
- m_i(\mathbf{s}_{1:i})
\;\big|\;
\mathbf{s}_{<i} < \boldsymbol{\tau}_{<i},\; s_i = \tau_i
\bigr] > 0,
\label{eq:stagewise_dominance}
\end{equation}
for every $i = 1, \ldots, k-1$. Then increasing any single threshold locally, 
with the others fixed, strictly increases both expected cost and expected quality. 
This extends to any connected region of threshold space where the condition holds at all stages.
\end{proposition}

\begin{proof}[Proof sketch]
The condition in Equation~\ref{eq:stagewise_dominance} is the quality-gain term in
the FOC in Equation~\ref{eq:foc}, and states that at the stage-$i$ boundary, continuing
the cascade has strictly higher expected quality than stopping at model $i$. By the
gradient expressions in Equations~\plaineqref{eq:grad_U}--\plaineqref{eq:grad_C}, both expected
quality and expected cost therefore increase locally with $\tau_i$.
\end{proof}

\subsection{Two-Model Specialization}
\label{app:two_model}
\label{sec:twomodel}
\label{sec:twomodel_foc}

Let $L$ and $H$ denote the lower-cost and higher-cost models, with expected
costs $c_L$ and $c_H$. Write $s_L$ for the lower-cost model's confidence
score and $\tau$ for its escalation threshold. Define
$m_L(s)=\mathbb{E}[U_L(x)\mid s_L(x)=s]$,
$m_H(s)=\mathbb{E}[U_H(x)\mid s_L(x)=s]$, and
$\gamma(s)=\mathbb{E}[C_H(x)\mid s_L(x)=s]$.
The stagewise integral form in Equations~\plaineqref{eq:integral_quality}--\plaineqref{eq:integral_cost} reduces to a pair of one-dimensional integrals over the confidence support $\mathcal{T} = \{\tau \in (0,1) : f_{s_L}(\tau) > 0\}$:
\begin{align*}
    \mathbb{E}[C(x;\tau)]
    &= c_L + \int_0^\tau \mathbb{E}[C_H(x) \mid s_L(x) = s] f_{s_L}(s)\, ds, \\
    \mathbb{E}[U(x;\tau)]
    &= \int_0^\tau m_H(s) f_{s_L}(s)\, ds + \int_\tau^1 m_L(s) f_{s_L}(s)\, ds.
\end{align*}
Specializing Theorem~\ref{thm:foc} yields the two-model first-order condition
\begin{equation}
m_H(\tau) - m_L(\tau) = \lambda\, \gamma(\tau),
\label{eq:foc_2model}
\end{equation}
which reduces to $m_H(\tau)-m_L(\tau)=\lambda c_H$ when
$\gamma(\tau) \equiv c_H$. The stagewise dominance condition in Equation~\ref{eq:stagewise_dominance} reduces to $m_H(\tau) - m_L(\tau) > 0$ for $\tau \in \mathcal{T}^\circ$.

\subsection{Frontier Geometry}
\label{app:two-model-geometry}
\label{sec:frontier_geometry}

\begin{definition}[Decreasing Benefit-Cost Ratio Region]
\label{def:ratio_region}
An interval $I\subseteq\mathcal{T}^\circ$ is a \emph{decreasing-ratio region}
if $\gamma$ is positive and bounded away from zero on $I$ and
$(m_H-m_L)/\gamma$ is weakly decreasing in the confidence score.
\end{definition}

This condition requires lower confidence to identify greater expected
escalation benefit per unit cost. With score-independent expected cost
$\gamma(s)=c_H>0$, it reduces to weakly decreasing escalation benefit.

\begin{proposition}[Piecewise Concavity and Reciprocal Shadow Prices]
\label{prop:duality}
Suppose two-model stagewise dominance holds throughout $\mathcal{T}^\circ$,
so $m_H(s)-m_L(s)>0$ for every $s\in\mathcal{T}^\circ$, and let
$I \subseteq \mathcal{T}^\circ$ be a decreasing-ratio region. Then every
threshold-parametrized point on $I$ is Pareto efficient and the Pareto frontier $U^*$ is
concave on the cost interval $\{\mathbb{E}[C(x;\tau)] : \tau \in I\}$.
If the optimum of \textbf{(P1)} or \textbf{(P2)} for a given target lies
in the interior of $I$, it is characterized by Equation~\ref{eq:foc_2model}, and
the optimal multipliers at $\tau^* \in I^\circ$ are reciprocals:
\[
    \lambda^*_{P1} = \frac{\gamma(\tau^*)}{m_H(\tau^*) - m_L(\tau^*)},
    \qquad
    \lambda^*_{P2} = \frac{m_H(\tau^*) - m_L(\tau^*)}{\gamma(\tau^*)}
    = \frac{1}{\lambda^*_{P1}}.
\]
The P1 multiplier $\lambda^*_{P1}$ has units of cost per unit quality and equals the local marginal cost of tightening the quality constraint by one unit. The P2 multiplier $\lambda^*_{P2}$ has units of quality per unit cost and equals the local marginal quality obtainable per additional unit of budget. If $\mathcal{T}^\circ$ is itself a single decreasing-ratio region, $U^*$ is globally concave, the Lagrangian relaxations of \textbf{(P1)} and \textbf{(P2)} have zero duality gap, and the local marginals are global shadow prices.
\end{proposition}

The slope identity $dU^*/dC=(m_H-m_L)/\gamma$ explains the
concavity condition. A decreasing benefit-cost ratio makes marginal quality
non-increasing in cost. Without positive
escalation benefit, this concave threshold curve need not be Pareto efficient.

\label{app:proof_piecewise}
\begin{proof}[Proof of Proposition~\ref{prop:duality}]
Assume two-model stagewise dominance on $\mathcal{T}^\circ$ and fix a
decreasing benefit-cost ratio region $I$. Differentiating the two-model
integral expressions gives
\[
\begin{aligned}
\frac{d}{d\tau}\mathbb{E}[U(x;\tau)]
&=
\bigl(m_H(\tau)-m_L(\tau)\bigr) f_{s_L}(\tau),\\
\frac{d}{d\tau}\mathbb{E}[C(x;\tau)]
&=
\gamma(\tau) f_{s_L}(\tau).
\end{aligned}
\]
The slope of the cost-quality frontier at an interior point parameterized
by $\tau$ is therefore
\[
\frac{d\mathbb{E}[U(x;\tau)]}{d\mathbb{E}[C(x;\tau)]}
=
\frac{m_H(\tau)-m_L(\tau)}{\gamma(\tau)}.
\]
Stagewise dominance and positive $\gamma$ make expected quality and expected
cost strictly increasing in $\tau$. Hence no lower-cost threshold can dominate
a higher-threshold point, and the threshold-parametrized curve coincides with
$U^*$ on the corresponding cost interval.
On a decreasing benefit-cost ratio region this slope is non-increasing in
$\tau$ and, because $\gamma$ is positive and bounded away from zero on $I$,
expected cost is strictly increasing in $\tau$. The slope is therefore
non-increasing in cost.
Hence $U^*$ restricted to the corresponding cost interval is concave.

For a target whose unrestricted optimum $\tau^*$ lies in $\mathrm{int}(I)$, stationarity of the Lagrangian on $I$ is necessary and sufficient. The envelope theorem applied to the value functions $V_{P1}(Q) = \min_\tau \mathbb{E}[C(x;\tau)]$ s.t.\ $\mathbb{E}[U(x;\tau)] \geq Q$ and $V_{P2}(B) = U^*(B)$ yields $\lambda^*_{P1} = V_{P1}'(Q)$ and $\lambda^*_{P2} = \frac{dU^*(B)}{dB}$. Since $V_{P1}'(Q) = \gamma(\tau^*)/(m_H(\tau^*) - m_L(\tau^*))$ and $\frac{dU^*(B)}{dB} = (m_H(\tau^*) - m_L(\tau^*))/\gamma(\tau^*)$, these are reciprocals.

When $I = \mathcal{T}^\circ$, the value function $U^*$ is globally
concave. For any budget strictly inside the attainable cost range, continuity
and strict monotonicity of $\tau \mapsto \mathbb{E}[C(x;\tau)]$ imply a
strictly feasible threshold for (P2). Standard Lagrangian duality for the
resulting concave program gives zero duality gap, and the local marginals are
global shadow prices at interior optima. Boundary budgets reduce to one-sided
derivative cases.

If the optimum lies at the boundary of $I$, the stationarity condition is replaced by a one-sided inequality and the reciprocal relationship holds as an inequality between the left- and right-derivatives of the value functions.
\end{proof}

\paragraph{Score-dependent per-query cost.}
\label{app:variable_cost}

Proposition~\ref{prop:duality} does not require constant per-query costs.
It allows score-dependent expected escalation cost through
$\gamma(s)=\mathbb{E}[C_H(x) \mid s_L(x)=s]$ and requires the benefit-cost
ratio $(m_H(s)-m_L(s))/\gamma(s)$ to be weakly decreasing on the relevant
confidence region, together with two-model stagewise dominance so that the
threshold-parametrized curve is Pareto efficient. Score-independence,
$\gamma(s)\equiv c_H$, is the
special case in which this condition reduces to decreasing
escalation benefit. Appendix~\ref{app:cost-variability} reports diagnostics
for both the score-independence approximation and the direct benefit-cost
ratio condition.

\subsection{Proof of Proposition~\ref{prop:auc-gain}}
\label{app:signal-quality-proof}

All expectations and covariances in this proof are taken among queries
reaching stage $i$. Let $\rho_i$ denote the percentile rank of $s_i$ among
these queries, which is uniform on $[0,1]$ for continuous scores.

Queries that stop before stage $i$ contribute the same amount to $a_i(t)$ and
$\bar a_i(t)$. A query reaching stage $i$ escalates when $\rho_i<t$, so
\[
a_i(t)-\bar a_i(t)
= p_i\,\mathbb E\bigl[(U_{>i}-U_i)(\mathbf 1\{\rho_i<t\}-t)\bigr].
\]
Since $\int_0^1(\mathbf 1\{\rho_i<t\}-t)\,dt=\tfrac12-\rho_i$ and
$\mathbb E[\rho_i]=\tfrac12$,
\begin{equation}
\int_0^1\bigl[a_i(t)-\bar a_i(t)\bigr]\,dt
= p_i\bigl[\operatorname{Cov}(U_i,\rho_i)-\operatorname{Cov}(U_{>i},\rho_i)\bigr].
\label{eq:cov-form}
\end{equation}
We show that $\Gamma_i(Y)=\operatorname{Cov}(Y,\rho_i)$ for any binary $Y$.
If $\mathbb E[Y]\in\{0,1\}$, both sides are zero. Otherwise, let
$\pi=\mathbb E[Y]$, let $F_y$ be the distribution function of $s_i$ given
$Y=y$, and let $A=\operatorname{AUROC}(s_i,Y)$. Then
$\rho_i=\pi F_1(s_i)+(1-\pi)F_0(s_i)$, which gives
$\mathbb E[\rho_i\mid Y=1]=\tfrac{\pi}{2}+(1-\pi)A$ and
$\mathbb E[\rho_i\mid Y=0]=\pi(1-A)+\tfrac{1-\pi}{2}$. The difference is
$A-\tfrac12$, so $\operatorname{Cov}(Y,\rho_i)=\pi(1-\pi)(A-\tfrac12)=\Gamma_i(Y)$.
Substituting into Equation~\ref{eq:cov-form} gives the first equality in
Equation~\ref{eq:auc-gain}. For the second, write
$U_i-U_{>i}=U_i(1-U_{>i})-(1-U_i)U_{>i}$ and apply linearity of covariance to
Equation~\ref{eq:cov-form}.

\section{Experimental Setup}
\label{app:setup}

\subsection{Models, Token Prices, and Cost Computation}
\label{app:costs}
\label{app:model-descriptives}

Per-query costs use prompt token counts estimated with \texttt{cl100k\_base}
and output counts from the stored log-probability sequence, multiplied by the
prices in Table~\ref{tab:token_prices}. Input counts therefore approximate
provider-specific tokenization. Prices reflect public
provider pricing at the time experiments were run (May 2026), using OpenAI
prices for GPT-4o and GPT-4o mini and Together AI prices for all other models.
All empirical frontiers and tables use these per-query cost estimates.

\begin{table}[h]
\centering
\tiny
\caption{Token prices used for cost computation, in dollars per million input and output tokens.}
\label{tab:token_prices}
\begin{tabular}{lrr}
\toprule
Model & Input & Output \\
\midrule
Llama 3.1-8B & 0.10 & 0.10 \\
GPT-oss-20B & 0.05 & 0.20 \\
GPT-4o mini & 0.15 & 0.60 \\
Qwen2.5-7B & 0.30 & 0.30 \\
Llama 3.3-70B & 0.88 & 0.88 \\
MiniMax-M2.7 & 0.30 & 1.20 \\
DeepSeek-V3 & 0.60 & 1.70 \\
GPT-4o & 2.50 & 10.00 \\
\bottomrule
\end{tabular}
\end{table}

Table~\ref{tab:model-descriptives} reports single-model frontier endpoints for
the lowest-cost and highest-accuracy models on each dataset. Cost is mean
dollars per query multiplied by $10^6$, and accuracy is the dataset-specific
correctness metric.

\begin{table}[h]
\centering
\tiny
\caption{Single-model frontier endpoints.}
\label{tab:model-descriptives}
\resizebox{\linewidth}{!}{%
\begin{tabular}{l l rr l rr}
\toprule
Dataset & Lowest-cost model & Acc. & Cost & Highest-accuracy model & Acc. & Cost \\
\midrule
MMLU & Llama 3.1-8B & 0.554 & 12.4 & GPT-oss-20B & 0.843 & 79.9 \\
TriviaQA & Llama 3.1-8B & 0.579 & 4.4 & Llama 3.3-70B & 0.830 & 40.2 \\
MATH & Llama 3.1-8B & 0.145 & 60.1 & DeepSeek-V3 & 0.818 & 1398.7 \\
SimpleQA & Llama 3.1-8B & 0.044 & 7.0 & GPT-4o & 0.382 & 356.4 \\
LiveCodeBench & Llama 3.1-8B & 0.179 & 70.8 & GPT-oss-20B & 0.700 & 387.0 \\
\bottomrule
\end{tabular}

}
\end{table}

\subsection{Dataset Grading Details}
\label{app:grading}

We use dataset-specific binary correctness labels. MMLU responses are graded
by exact match to the selected multiple-choice letter. TriviaQA uses normalized
exact match after lowercasing and removing punctuation, articles, and excess
whitespace. MATH is graded with a symbolic-equivalence checker based on SymPy
after extracting the final answer. SimpleQA is graded with the official
short-answer factuality rubric using an LLM judge, mapped to binary
correctness. LiveCodeBench is graded by executing submitted solutions against
the benchmark test cases. Grading is completed before calibration-test
splitting and cascade optimization.

\subsection{Deferral Scores}
\label{app:scores}

We compute the token-probability scores using the UQLM library
\citep{bouchard2025uqlm, bouchard2025uncertainty}. Seven confidence-based scores are compared, namely
five off-the-shelf log-probability scorers (mean token negentropy, min token
negentropy, probability margin, min probability, length-normalized probability), a
logistic regression trained on the five base scorers (correctness-trained
token-probability score), and a
logistic regression trained on frozen embeddings of the prompt and cheap-model
response (correctness-trained embedding score). Both learned scorers use calibration-set data
only and predict cheap-model correctness $U_L(x)$. We additionally evaluate two
pair-specific Diff-01 ridge scores that predict $U_H(x)-U_L(x)$ from either the
five base scores (benefit-trained token-probability score) or the frozen
prompt-and-response embeddings (benefit-trained embedding score). Their confidence
score is the negative predicted difference, so lower values trigger escalation.
For a generated response
$y=(t_1,\ldots,t_L)$, let $p_j$ denote the probability assigned to the generated
token $t_j$.

\paragraph{Length-Normalized Probability (LNP).} This score uses the geometric mean of the probabilities assigned to the generated tokens, which removes the mechanical penalty faced by longer responses \citep{malinin2021uncertaintyestimationautoregressivestructured}:
\[
\mathrm{LNP}(y)=\left(\prod_{j=1}^{L}p_j\right)^{1/L}.
\]

\paragraph{Minimum Token Probability (MTP).} This score records the least confident generated token in the response \citep{manakul2023selfcheckgptzeroresourceblackboxhallucination}:
\[
\mathrm{MTP}(y)=\min_{j\in\{1,\ldots,L\}}p_j.
\]

The remaining scorers use the top-$K$ alternatives at each output
position. Let $\{p_{j,1},\ldots,p_{j,K}\}$ be the top-$K$ probabilities
at position $j$, sorted from largest to smallest, and define the
renormalized top-$K$ probabilities
$p_{j,k}^{(K)}=p_{j,k}/\sum_{\ell=1}^{K}p_{j,\ell}$ for
$k=1,\ldots,K$.
In all experiments, we request $K=15$ top log-probabilities when available.
For the Llama models served through the Together AI API, which exposes only
$K=5$ in our setup, the top-$K$ scorers use $K=5$.

\paragraph{Probability Margin (PM).} This margin averages the probability gap between the most likely and second-most-likely token at each position \citep{farr2024redctsystemsdesignmethodology}:
\[
\mathrm{PM}(y)=\frac{1}{L}\sum_{j=1}^{L}(p_{j,1}^{(K)}-p_{j,2}^{(K)}).
\]

\paragraph{Average Token Negentropy (ATN@$K$).} This score averages a normalized negentropy transformation over token positions \citep{scalena2025eagerentropyawaregenerationadaptive,manakul2023selfcheckgptzeroresourceblackboxhallucination}. First define the top-$K$ entropy at position $j$ as
\[
\mathrm{TE@}K(t_j)=-\sum_{k=1}^{K}p_{j,k}^{(K)}\log p_{j,k}^{(K)}.
\]
We convert entropy to a confidence-oriented score in $[0,1]$ by
\[
\begin{aligned}
\mathrm{TN@}K(t_j)
&=1-\frac{\mathrm{TE@}K(t_j)}{\log K},\\
\mathrm{ATN@}K(y)
&=\frac{1}{L}\sum_{j=1}^{L}\mathrm{TN@}K(t_j).
\end{aligned}
\]

\paragraph{Minimum Token Negentropy (MTN@$K$).} This score takes the least confident normalized-negentropy value across the generated positions \citep{scalena2025eagerentropyawaregenerationadaptive,manakul2023selfcheckgptzeroresourceblackboxhallucination}:
\[
\mathrm{MTN@}K(y)=\min_{j\in\{1,\ldots,L\}}\mathrm{TN@}K(t_j).
\]

\subsection{Exact Calibration-Set Selection and Test-Set Evaluation}
\label{app:exact-depth}
\label{app:exact-heldout-three}

\paragraph{Policy classes and samples.}
We implement the constrained sequence optimization in
Section~\ref{sec:depth-value}. For maximum allowed depth $d$, $S_d$ is the union of standalone models and
all cost-ordered sequences containing two through $d$ models, with a scalar
threshold at every nonterminal stage. All eight models are included.
Calibration mean costs determine the order, with model name breaking cost
ties. Neither standalone dominance nor increasing accuracy restricts the
pool or escalation edges. The best feasible single model $S_1$ varies with
budget and can differ from the highest-accuracy model used as the upper
budget anchor.

Rows are jointly filtered for finite scores, correctness, and costs, and
prompt alignment is checked across models. The 50 splits use seeds 42--91
and stratify on the correctness of the first available model, rounding down
the calibration count within each stratum. Calibration/test counts are
999/1,000 for MMLU, 999/1,001 for TriviaQA, 985/986 for MATH, 945/946 for
SimpleQA, and 527/528 for LiveCodeBench. Within split $r$, the 500 assigned
budgets are
\begin{equation}
 B_{rj}=\widehat C_{\mathrm{cal}}(L_r)
 +\frac{j}{499}\left[\widehat C_{\mathrm{cal}}(H_r)
 -\widehat C_{\mathrm{cal}}(L_r)\right],\qquad j=0,\ldots,499,
 \label{eq:deployment-budget}
\end{equation}
where $L_r$ is the calibration-cheapest model and $H_r$ has highest calibration-set
accuracy, breaking accuracy ties by lower cost and then model name. All
methods share these budgets, rows, and splits.

\paragraph{Threshold convention and exact search.}
Escalation occurs when the score is strictly below the threshold. Distinct
calibration-set scores define routing partitions, with tied observations moved
together. An interior rank maps to the next distinct calibration-set score, and
explicit negative and positive infinity sentinels implement never and always
escalation. Thresholds do not change when applied to unseen scores. The search
is exact over these calibration-set partitions, the specified model order, and
the assigned budget range. It does not optimize every possible off-sample
cutpoint inside a calibration-set score gap.

Pairs and triples are enumerated using one- and two-dimensional cumulative
sums. The four-model compiled kernel streams threshold combinations, skips
changes to unreachable thresholds, and prunes branches that exceed the
maximum assigned budget. Binary correctness permits retaining the cheapest
witness at each attainable integer correctness count. The five-model kernel
additionally uses an upper bound on downstream correctness. A branch can be
pruned when a strictly cheaper incumbent attains at least that upper bound.
Incumbents may come from shallower classes or previously searched sequences,
which certifies the union optimum. Such a cache does not certify an individual
fixed sequence's optimum.

At each budget, calibration-set accuracy has first priority and lower calibration-set
cost has second priority. For depths one through three, the selector
retains the first enumerated witness on exact cost ties. The depth-four and
five extensions use a $10^{-15}$ dollar cost-tie tolerance and retain the
existing shallower incumbent on numerical ties. Budget feasibility uses
$10^{-12}$ dollars per query.
The nested classes include shallower policies explicitly. A selected nominal
five-model sequence need not call every model on every query. Always
escalating at an intermediate stage still incurs that stage's generation cost.

\paragraph{Freeze, replay, and evaluate.}
Every selected calibration witness is independently replayed before the
manifest is saved. The manifest records model identities, thresholds,
calibration quantities, split rows, budgets, and data and source fingerprints.
Only then are test rows evaluated. All test operating points remain at their
assigned budget indices. Overshoots, dominated points, and nonmonotone accuracy
are retained. The depth-four and five audits verify 250 completed splits and
125,000 budget positions each, baseline preservation, source hashes, split
separation, calibration nesting and feasibility, and completed model sequence
checkpoints. Kernel and selector tests compare with exhaustive enumeration
on small synthetic pools, including tied scores, endpoint policies, numerical
ties, incumbent pruning, and planted cases where the fifth stage is necessary.
Timeouts are incomplete searches and cannot certify an optimum.

\paragraph{Budget summaries and uncertainty.}
All primary summaries use trapezoidal integration over normalized budget
fraction followed by equal weighting across splits. Positive accuracy
differences favor the deeper class named first. Cost differences
use dollars per 1,000 incoming queries and percentages of assigned budget.
For split $r$ and normalized budget position $t$, let $B_r(t)$ denote the
assigned budget and let $\Delta C_r(t)$ denote the paired difference in mean
realized cost per query. We define two cost summaries,
\begin{equation}
 D_{\$}=\frac{1}{50}\sum_{r=1}^{50}\int_0^1 1000\Delta C_r(t)\,dt,
 \qquad
 D_{\%}=\frac{1}{50}\sum_{r=1}^{50}\int_0^1
 \frac{100\Delta C_r(t)}{B_r(t)}\,dt.
\end{equation}
Both integrals use the same trapezoidal budget weights. The percentage
summary divides by the assigned budget before averaging, giving a dollar
of cost difference greater weight at a smaller budget. The summaries can
therefore have opposite signs when cost differences change sign across
budgets. For MMLU, $S_4-S_2$ has a mean dollar-cost difference of
$+0.00475$ per 1,000 queries in the lowest budget third and $-0.00993$
in the highest third. Across the full range, its mean dollar-cost difference
is $-0.00232$ per 1,000 queries, while its mean budget-normalized difference
is $+1.36\%$. The corresponding $S_3-S_2$ summaries are $-0.00461$
and $+0.21\%$. A positive percentage summary does not imply a positive
mean dollar-cost difference.

We also compute arithmetic-mean cross-checks and paired
adjacent-depth comparisons. Low, middle, and high budget thirds have exact
boundaries $[0,1/3]$, $[1/3,2/3]$, and $[2/3,1]$. For these descriptive
integrals, the paired statistic is linearly interpolated at the two interior
boundaries, rather than selecting or interpolating a new test policy.
Each third is normalized by its width.

Repeated splits overlap. The 10th--90th percentiles measure split sensitivity
and are not confidence intervals or evidence of practical equivalence.
Fractions of split-budget cells exceeding $0.5$ or $1.0$ percentage points are
descriptive. No quality-target cost reduction is inferred from the 500
budget-capped selections, since a separate cheapest-target-attaining search
would be needed to certify that quantity.

The depth-five search used four independent worker processes (two for
LiveCodeBench) and one search thread per worker. Each eight-model split
searches 56 five-model sequences. Threshold
witnesses, branch counts, pruning counts, and timings are retained with each
model sequence checkpoint. 

\FloatBarrier

\subsection{FOC-Guided Threshold Optimization}
\label{app:foc-optimizer}

\paragraph{Comparison and policy classes.}
We replace exhaustive threshold enumeration with FOC-guided search while
preserving the sequence optimization in Section~\ref{sec:depth-value}.
All eight models, all cost-ordered subsequences, and standalone policies
remain available. There are 28 pairs, 56 triples, 70 four-model sequences,
and 56 five-model sequences per split. Each depth cap includes the selected
policies from the preceding class. We use the same inputs, 50
calibration-test splits, and 500 assigned budgets as the exact comparison
in Appendix~\ref{app:exact-depth}. Exact policies and test outcomes do not
initialize or guide the FOC search.

\paragraph{Estimating the boundary returns.}
At stage $i$, let $r_i$ indicate that the current upstream policy reaches
the stage. The FOC equates the boundary return
\[
R_i(\boldsymbol\tau)=
\frac{\mathbb E[U_{>i}-U_i\mid r_i=1,s_i=\tau_i]}
     {\mathbb E[C_{>i}\mid r_i=1,s_i=\tau_i]}
\]
with the common multiplier at an interior optimum. The downstream quality
$U_{>i}$ and cost $C_{>i}$ use the current downstream thresholds.
We estimate $R_i$ by the ratio of summed quality gain to summed downstream
cost in a window of up to 100 reached calibration observations, ordered
by the current model's score. The window is centered on the threshold's
position in that ordering, with one-sided windows at the endpoints.
Threshold decisions preserve tied scores. The estimation window can include
part of a tied score group, using stable observation order. For longer
sequences, at least 50 reached observations are required to estimate a
ratio. Sparse and unreachable boundaries retain feasible endpoint handling
but do not support estimated-ratio transfers. The pairwise runs have more
than 100 calibration observations in every split. No score-cost independence
assumption is imposed.

\paragraph{Individual threshold updates.}
Holding the other thresholds fixed, cumulative sums give the calibration
cost of every distinct threshold partition. Cost is nondecreasing in the
threshold, so the budget determines its largest feasible partition. The
candidate set includes the current threshold, feasible always-stop and
always-continue endpoints, the largest budget-feasible threshold, and
partitions bracketing positive-to-negative crossings of the estimated
benefit-cost ratio. The ends of a zero plateau are also considered.
For a pair, this candidate calculation gives the full inner search.
A binding budget need not have zero marginal benefit, so the budget boundary
is included separately from stationary-point candidates. For longer
sequences, the calculation updates one threshold at a time.

\paragraph{Reallocating computation.}
The FOC also supplies a direction for transferring cost between thresholds.
To first order, a budget-preserving transfer from boundary $j$ to boundary
$i$ changes expected quality by $(R_i-R_j)\,dC_i$, where $dC_i$ denotes the
cost increase attributable to changing threshold $i$.
After a sweep of individual updates, we estimate all available boundary
returns. For each donor and receiver with
$\widehat R_i>\max\{0,\widehat R_j\}$, we propose reducing the donor
threshold by 1\%, 5\%, or 20\% of its model's score groups, rounded up to at
least one group. We then recompute the receiving threshold's candidates
under the same budget. A transfer must increase the receiver's threshold.
We accept the best proposed policy only if it improves actual calibration
accuracy, or preserves accuracy while reducing cost. Thus, estimated returns
guide proposals, while observed calibration quality and cost determine
acceptance. This is a local search with acceptance checks rather than a
solver that assumes stationarity is sufficient for global optimality.

\paragraph{Initialization and termination.}
For each sequence, budgets are processed in increasing order, using the
preceding budget's selected thresholds as a warm start. The first start
sets all thresholds to never escalate. At the first feasible budget and
every 25th budget index, four additional starts set thresholds to common
25\%, 50\%, 75\%, and 100\% score-group ranks. The first threshold is updated
to make each start feasible before further optimization. Each start receives
at most six sweeps, terminating earlier if a sweep produces no accepted
improvement. Additional starts therefore occur at scheduled budget positions,
rather than at every budget. The same settings are used for all datasets
and depths and are not tuned on test results.

\paragraph{Selection and evaluation.}
Within and across sequences, higher calibration accuracy has priority, with
lower cost breaking ties. The longer-sequence optimizer requires a cost
reduction greater than $10^{-15}$ dollars per query to accept an accuracy
tie. Budget feasibility permits a numerical tolerance of $10^{-12}$ dollars
per query. The pairwise candidate search uses strict cost improvement for
ties. Every returned longer-sequence policy is independently replayed on
calibration data. Selected policies are checked for feasibility, nested
calibration accuracy across depth caps, and the exhaustive accuracy upper
bound. Policies and source fingerprints are saved before held-out evaluation.
The exact comparators are independently replayed on the same test inputs.

\begin{table}[ht]
\centering
\small
\caption{Calibration and held-out comparison of threshold optimizers. Accuracy and cost
differences are FOC-guided minus exhaustive. Means integrate over assigned
budget positions within each split and average splits equally. The
10th--90th percentiles describe split variability, not confidence intervals.
The last column gives the budget-weighted frequency of attaining exhaustive
calibration accuracy, without requiring the same selected policy or cost.}
\label{tab:foc-optimizer-heldout}
\resizebox{\linewidth}{!}{\begin{tabular}{lrrrrrrr}
\toprule
Dataset & Class & Cal. $\Delta$ acc. (pp) & Test $\Delta$ acc. (pp) & 10th--90th & $\Delta$ cost (\$/1,000) & $\Delta$ cost (\% budget) & Cal. optimum (\%) \\
\midrule
MMLU & $S_2$ & -0.069 & +0.049 & [-0.042, +0.129] & +0.001711 & +1.645 & 55.1 \\
 & $S_3$ & -0.027 & +0.015 & [-0.028, +0.122] & +0.001924 & +0.440 & 79.5 \\
 & $S_4$ & -0.036 & +0.009 & [-0.078, +0.074] & +0.000553 & +0.473 & 71.6 \\
 & $S_5$ & -0.050 & +0.011 & [-0.082, +0.135] & -0.002215 & -0.415 & 61.4 \\
TriviaQA & $S_2$ & -0.035 & +0.054 & [+0.002, +0.102] & +0.000415 & +1.779 & 77.0 \\
 & $S_3$ & -0.033 & +0.012 & [-0.049, +0.077] & +0.000004 & +0.115 & 77.6 \\
 & $S_4$ & -0.039 & +0.026 & [-0.024, +0.072] & +0.000015 & +0.135 & 71.9 \\
 & $S_5$ & -0.037 & +0.027 & [-0.025, +0.079] & +0.000020 & +0.149 & 72.9 \\
MATH & $S_2$ & -0.051 & +0.045 & [-0.035, +0.108] & +0.024203 & +3.300 & 61.2 \\
 & $S_3$ & -0.013 & +0.023 & [-0.009, +0.077] & -0.001749 & -0.020 & 88.1 \\
 & $S_4$ & -0.014 & +0.033 & [-0.004, +0.096] & -0.001754 & -0.039 & 87.9 \\
 & $S_5$ & -0.014 & +0.036 & [-0.001, +0.106] & -0.001672 & -0.021 & 87.8 \\
SimpleQA & $S_2$ & -0.024 & +0.093 & [+0.055, +0.121] & +0.001512 & +1.017 & 82.5 \\
 & $S_3$ & -0.024 & +0.033 & [+0.009, +0.059] & +0.000180 & +0.223 & 85.2 \\
 & $S_4$ & -0.023 & +0.034 & [+0.010, +0.060] & +0.000171 & +0.218 & 84.7 \\
 & $S_5$ & -0.023 & +0.033 & [+0.010, +0.060] & +0.000170 & +0.217 & 84.7 \\
LiveCodeBench & $S_2$ & -0.019 & +0.118 & [+0.077, +0.171] & +0.001832 & +0.725 & 90.9 \\
 & $S_3$ & -0.003 & +0.016 & [+0.003, +0.031] & +0.000048 & +0.029 & 98.5 \\
 & $S_4$ & -0.005 & +0.023 & [+0.003, +0.054] & -0.000183 & -0.032 & 97.4 \\
 & $S_5$ & -0.005 & +0.023 & [+0.003, +0.054] & -0.000164 & -0.027 & 97.4 \\
\bottomrule
\end{tabular}
}
\end{table}

\paragraph{Accuracy and cost.}
Table~\ref{tab:foc-optimizer-heldout} reports mean calibration accuracy differences and the complete held-out comparison.
Higher mean test accuracy can coexist with a calibration gap because both
methods select policies from finite calibration samples. It does not establish
an equal-cost improvement, since realized test costs differ. Under FOC-guided
selection, moving from $S_4$ to $S_5$ changes mean test accuracy by
$-0.063$, $-0.006$, $-0.001$, $-0.001$, and $0.000$ percentage points on
MMLU, TriviaQA, MATH, SimpleQA, and LiveCodeBench, respectively.

\begin{table}[ht]
\centering
\small
\caption{Timing comparison on split 0 of each dataset. Times include
selection across all sequences through the stated depth cap, including
shorter policies, and exclude file loading, compilation, and held-out
evaluation. Runs are sequential on the same machine. Ratios greater than one
favor FOC-guided search. A single split does not establish typical runtime.}
\label{tab:foc-optimizer-timing}
\begin{tabular}{lrrrr}
\toprule
Dataset & Class & Exhaustive (s) & FOC-guided (s) & Exhaustive / FOC \\
\midrule
MMLU & $S_3$ & 0.692 & 1.427 & 0.48 \\
 & $S_4$ & 17.966 & 4.904 & 3.66 \\
TriviaQA & $S_3$ & 0.401 & 0.888 & 0.45 \\
 & $S_4$ & 2.164 & 3.080 & 0.70 \\
MATH & $S_3$ & 0.671 & 0.954 & 0.70 \\
 & $S_4$ & 11.602 & 2.912 & 3.98 \\
SimpleQA & $S_3$ & 0.659 & 1.217 & 0.54 \\
 & $S_4$ & 16.222 & 4.156 & 3.90 \\
LiveCodeBench & $S_3$ & 0.188 & 0.330 & 0.57 \\
 & $S_4$ & 0.425 & 0.999 & 0.42 \\
\bottomrule
\end{tabular}

\end{table}

\paragraph{Runtime and limitations.}
Table~\ref{tab:foc-optimizer-timing} reports the timing comparison for
$S_3$ and $S_4$. The matched timing benchmark covers only these
depths, so it does not establish an $S_5$ speedup. At $S_2$, the FOC
implementation was 7--11\%
slower in mean measured selection time across the 50 splits. Both pairwise
methods construct cumulative sums, while the guided method adds boundary
estimation. Fewer candidate thresholds therefore do not necessarily reduce
runtime. More generally, finite estimation windows, threshold steps, local
optima, scheduled restarts, and the six-sweep limit can leave calibration
gaps. The experiment does not isolate the benefit of FOC guidance relative
to a generic local-search method with a matched computational budget.





\section{Additional Results}
\label{app:results}

\subsection{Checks of the Theoretical Conditions}
\label{app:theory-checks}

\subsubsection{Representative Escalation Benefit Curves}
\label{app:segments}

Figure~\ref{fig:escben_representative} shows the escalation benefit $m_H(s) - m_L(s)$ for one representative pair per dataset, defined as the pair that attains $U_2^*(B)$ at the largest number of budget grid points. These are the same pairs used in Appendices~\ref{app:concavity-check} and~\ref{app:cost-variability}. Each panel annotates two diagnostics of the median spline curve: the expected-dominance fraction (\texttt{dom}), the share of the confidence grid on which the median benefit is positive, and the decreasing-benefit fraction (\texttt{dec}), the share of adjacent grid steps on which it is non-increasing. Across all pairs, the median \texttt{dom} ranges from 0.87 to 1.00 and the median \texttt{dec} from 0.71 to 0.99 by dataset. Curves for all 28 pairs on each dataset are provided in the supplementary material.

\begin{figure}[!htb]
\centering
\includegraphics[width=0.65\linewidth]{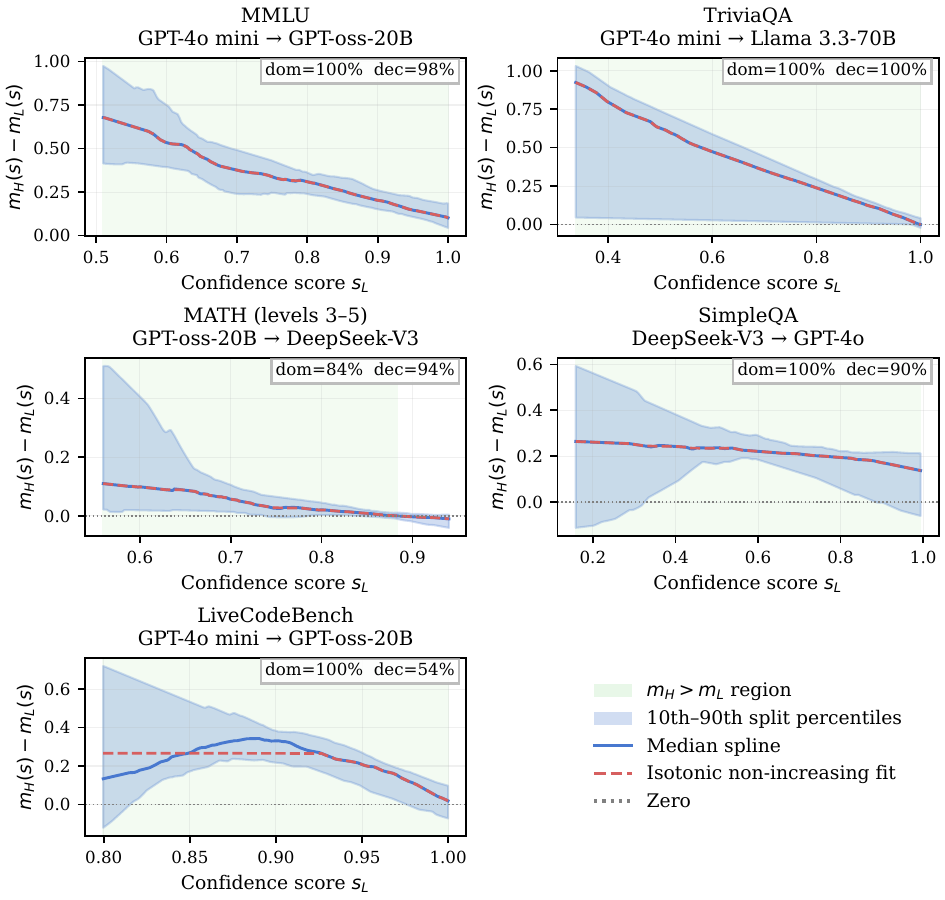}
\caption{Escalation benefit curves for the representative pair on each dataset. Bands show the 10th--90th percentile of per-split spline fits across 50 splits, the solid line is the median spline, and the dashed line is its isotonic non-increasing fit. Green shading marks confidence regions where the median benefit is positive.}
\label{fig:escben_representative}
\end{figure}

\subsubsection{Realized Frontier Concavity Check}
\label{app:concavity-check}

Proposition~\ref{prop:duality} characterizes concavity of the population
two-model frontier under stagewise dominance and a decreasing benefit-cost
ratio condition. As an
empirical check on the realized frontier object, Table~\ref{tab:concavity-check}
compares each representative pair's test-set threshold frontier to its concave
hull across the 50 calibration-test splits. Concave share is the fraction of
frontier points lying within a tolerance of the hull, where the tolerance is
fixed at 1\% of the split-specific frontier quality range. Max gap is the
largest vertical distance from the frontier to the hull in accuracy units. We
also report the same quantities after averaging quality within cost deciles.

\begin{table}[h]
\centering
\small
\caption{Concavity check for representative realized two-model frontiers. Raw
frontiers use the threshold-sweep operating points with realized token costs.
Binned frontiers average quality within cost deciles before computing the
concave hull.}
\label{tab:concavity-check}
\resizebox{\linewidth}{!}{%
\begin{tabular}{llrrrr}
\toprule
Dataset & Pair & Raw share & Binned share & Raw max gap & Binned max gap \\
\midrule
MMLU & GPT-4o mini$\to$GPT-oss-20B & 0.19 & 0.83 & 0.007 & 0.002 \\
TriviaQA & GPT-4o mini$\to$Llama 3.3-70B & 0.27 & 0.80 & 0.012 & 0.003 \\
MATH & GPT-oss-20B$\to$DeepSeek-V3 & 0.37 & 0.55 & 0.005 & 0.003 \\
SimpleQA & DeepSeek-V3$\to$GPT-4o & 0.39 & 0.80 & 0.009 & 0.004 \\
LiveCodeBench & GPT-4o mini$\to$GPT-oss-20B & 0.22 & 0.60 & 0.018 & 0.009 \\
\bottomrule
\end{tabular}

}
\end{table}

The raw frontiers contain local finite-sample nonconcavities, with median max
gaps of 0.005--0.018 accuracy. After cost-decile smoothing, the median concave
share rises to 0.55--0.83 and max gaps fall to 0.002--0.009 accuracy. This
pattern suggests that much of the observed nonconcavity comes from discrete
realized correctness and token-cost variation rather than large-scale frontier
shape. This check does not validate the assumptions of
Proposition~\ref{prop:duality} directly. It measures whether the realized
threshold frontier exhibits the concavity behavior characterized by the
proposition, while the corresponding premise checks are reported separately in
Appendices~\ref{app:segments} and~\ref{app:cost-variability}.

\subsubsection{Cost and Benefit-Cost Ratio Diagnostics}
\label{app:cost-variability}

Proposition~\ref{prop:duality} is stated under stagewise dominance and a
decreasing benefit-cost ratio condition. Score-independence,
$\mathbb{E}[C_H(x) \mid s_L(x)=s]=c_H$, is the
interpretable corollary in which this condition reduces to decreasing
escalation benefit. We therefore report two diagnostics. Table~\ref{tab:cost_score_corr}
checks the score-independence approximation using the Spearman rank correlation
between the cheap model confidence score $s_L$ and the expensive model realized
token cost $C_H$ for every cost-ordered model pair.

\begin{table}[h]
\centering
\tiny
\caption{Cost--score independence diagnostic. Each entry summarizes pair-level
Spearman correlations between the cheap model confidence score $s_L$ and the
expensive model realized cost $C_H$ across cost-ordered pairs.}
\label{tab:cost_score_corr}
\resizebox{\linewidth}{!}{%
\begin{tabular}{lccccc}
\toprule
Dataset & Pairs & Median $|\rho_s|$ & 90th $|\rho_s|$ & Max $|\rho_s|$ & Share $|\rho_s|<0.20$ \\
\midrule
MMLU & 28 & 0.148 & 0.413 & 0.614 & 0.571 \\
TriviaQA & 28 & 0.142 & 0.390 & 0.534 & 0.714 \\
MATH & 28 & 0.150 & 0.487 & 0.506 & 0.607 \\
SimpleQA & 28 & 0.064 & 0.155 & 0.165 & 1.000 \\
LiveCodeBench & 28 & 0.325 & 0.500 & 0.526 & 0.179 \\
\bottomrule
\end{tabular}
}
\end{table}

\begin{table}[h]
\centering
\tiny
\caption{Benefit-cost ratio diagnostic. Ratio dec. is the median share of
adjacent confidence bins for which the estimated benefit-cost ratio is
non-increasing. $\rho_s$ is the median Spearman correlation between confidence
and the estimated ratio. All-pair median summarizes calibration-valid pairs.}
\label{tab:ratio_condition}
\resizebox{\linewidth}{!}{%
\begin{tabular}{lcccc}
\toprule
Dataset & Representative pair & Ratio dec. & $\rho_s$ & All-pair median \\
\midrule
MMLU & GPT-4o mini $\to$ GPT-oss-20B & 0.63 & -0.67 & 0.58 \\
TriviaQA & GPT-4o mini $\to$ Llama 3.3-70B & 0.53 & -0.73 & 0.58 \\
MATH & GPT-oss-20B $\to$ DeepSeek-V3 & 0.53 & -0.14 & 0.53 \\
SimpleQA & DeepSeek-V3 $\to$ GPT-4o & 0.53 & -0.54 & 0.47 \\
LiveCodeBench & GPT-4o mini $\to$ GPT-oss-20B & 0.53 & -0.31 & 0.53 \\
\bottomrule
\end{tabular}

}
\end{table}

\paragraph{Interpretation.} The score-independence approximation is most
plausible on SimpleQA, while LiveCodeBench exhibits the strongest cost-score
dependence. The direct ratio check in Table~\ref{tab:ratio_condition} is
directionally consistent with the decreasing-ratio premise, with negative rank
correlations for every representative pair. Adjacent-bin monotonicity is only
0.53--0.63, so we interpret this as noisy finite-sample evidence rather than
verification of the population premise.

\subsection{Scorer Discrimination and Robustness of the Depth Comparison}
\label{app:voi}

We evaluate scorer discrimination across the seven confidence-based scores and
two pair-specific learned deferral scores defined in
Appendix~\ref{app:scores}. The depth comparison
applies the sequence optimization of Section~\ref{sec:depth-value}
separately for each score.

\subsubsection{Scorer Discrimination}

We evaluate nine scores on test-set data across 50
calibration-test splits. Admissible model pairs are selected using calibration
data only. Table~\ref{tab:voi_summary} reports discrimination of positive
realized escalation benefit for these pairs.

\paragraph{Benefit-AUROC.} For each pair $(L,H)$, we compute the AUROC of
$-s_L(x)$ for predicting $\mathbf{1}\{U_H(x)>U_L(x)\}$ on test-set queries.
Larger values mean that low confidence better identifies queries on which
escalation changes the answer from incorrect to correct. For each scorer
and dataset, we take the median across splits within each pair and then
average equally across pairs. Single-class evaluations receive AUROC $0.5$.

Discrimination is strongest on MMLU and TriviaQA. On the other datasets,
values are closer to chance and vary across scorers. This metric measures
ranking of positive quality changes, without accounting for their costs.

\begin{table}[h]
\centering
\small
\caption{Test-set benefit-AUROC by scorer and dataset. Entries average pair-specific medians across 50 calibration-test splits over calibration-valid pairs. Higher values indicate better discrimination of positive escalation benefit.}
\label{tab:voi_summary}
\begin{tabular}{lccccc}
\toprule
Scorer & MMLU & TriviaQA & MATH & SimpleQA & LiveCodeBench \\
\midrule
Mean negentropy   & 0.68 & 0.76 & 0.55 & 0.49 & 0.57 \\
Min negentropy    & 0.68 & 0.78 & 0.60 & 0.45 & 0.55 \\
Prob.\ margin     & 0.67 & 0.73 & 0.55 & 0.50 & 0.56 \\
Min probability   & 0.66 & 0.77 & 0.58 & 0.45 & 0.54 \\
Length-norm.\ probability & 0.67 & 0.76 & 0.55 & 0.48 & 0.56 \\
Correctness-trained token-prob. & 0.68 & 0.79 & 0.59 & 0.45 & 0.55 \\
Correctness-trained embedding   & 0.63 & 0.68 & 0.58 & 0.46 & 0.50 \\
Benefit-trained token-prob.     & 0.69 & 0.79 & 0.59 & 0.53 & 0.55 \\
Benefit-trained embedding       & 0.55 & 0.72 & 0.56 & 0.53 & 0.52 \\
\bottomrule
\end{tabular}
\end{table}

\FloatBarrier
\subsubsection{Robustness of the Depth Comparison}

For all nine scores, we recompute the exact calibration-selected classes $S_2$
through $S_5$.
Each class contains all cost-ordered sequences
up to its depth cap, including standalone policies.
We retain the full eight-model pool and use no standalone-dominance pruning or accuracy restrictions
on escalation edges. All scorers share 50 calibration-test splits, seeds,
model pools, and 500 assigned budget positions. Rows are filtered jointly
for finite correctness, costs, and all five base scores before splitting.
Dollar budgets use the same calibration-only anchors as the main experiment.

The two correctness-based learned scores fit one predictor per model. The two
Diff-01 scores fit a separate predictor for each cost-ordered model transition,
using calibration labels only. We freeze all predictors before threshold search
and test-set evaluation. At each depth cap, we enumerate every distinct
calibration routing partition under the strict $s<\tau$ convention, including
never- and always-escalate endpoints. Selection maximizes calibration-set
accuracy under each budget, with the same cost and enumeration tie-breaking
rules as the main exact search. Sequences and thresholds are saved before
test-set policy evaluation. Frozen policies remain at their assigned budget
positions, including realized-cost overshoots.

Table~\ref{tab:scorer-depth} reports the paired $S_3$ minus $S_2$ accuracy
integrals over normalized budget, computed as described in
Appendix~\ref{app:exact-depth}, together with realized-cost differences. The 10th--90th percentiles describe variability across the
repeated, overlapping splits.

\begin{table}[htbp]
\centering
\scriptsize
\caption{Scorer robustness of the exact $S_3$ versus $S_2$ comparison. Top panel reports mean test-set accuracy differences in percentage points, with 10th--90th split percentiles in brackets. Bottom panel reports mean realized-cost differences in dollars per 1,000 queries. Differences integrate over 500 assigned budget positions and average equally over 50 splits. Positive values favor $S_3$ in accuracy and indicate greater cost.}
\label{tab:scorer-depth}
\resizebox{\linewidth}{!}{%
\begin{tabular}{lccccc}
\toprule
Scorer & MMLU & TriviaQA & MATH & SimpleQA & LiveCodeBench \\
\midrule
\multicolumn{6}{c}{Accuracy difference (percentage points)} \\
\midrule
Mean negentropy & \shortstack{$-0.053$ \\ $[-0.316, +0.300]$} & \shortstack{$+0.052$ \\ $[-0.163, +0.263]$} & \shortstack{$-0.077$ \\ $[-0.212, +0.004]$} & \shortstack{$-0.095$ \\ $[-0.221, -0.007]$} & \shortstack{$-0.025$ \\ $[-0.049, -0.002]$} \\
Min negentropy & \shortstack{$+0.272$ \\ $[-0.130, +0.801]$} & \shortstack{$+0.191$ \\ $[-0.092, +0.402]$} & \shortstack{$-0.104$ \\ $[-0.297, -0.000]$} & \shortstack{$-0.099$ \\ $[-0.232, +0.005]$} & \shortstack{$-0.037$ \\ $[-0.075, +0.000]$} \\
Prob.\ margin & \shortstack{$-0.069$ \\ $[-0.314, +0.295]$} & \shortstack{$-0.113$ \\ $[-0.395, +0.115]$} & \shortstack{$-0.110$ \\ $[-0.273, -0.006]$} & \shortstack{$-0.016$ \\ $[-0.199, +0.112]$} & \shortstack{$-0.020$ \\ $[-0.053, +0.000]$} \\
Min probability & \shortstack{$-0.076$ \\ $[-0.300, +0.158]$} & \shortstack{$+0.127$ \\ $[-0.063, +0.342]$} & \shortstack{$-0.114$ \\ $[-0.278, -0.015]$} & \shortstack{$-0.116$ \\ $[-0.289, -0.011]$} & \shortstack{$-0.005$ \\ $[-0.075, +0.040]$} \\
Length-norm.\ probability & \shortstack{$-0.044$ \\ $[-0.387, +0.284]$} & \shortstack{$+0.035$ \\ $[-0.131, +0.250]$} & \shortstack{$-0.075$ \\ $[-0.217, +0.036]$} & \shortstack{$-0.129$ \\ $[-0.303, +0.001]$} & \shortstack{$-0.040$ \\ $[-0.096, -0.008]$} \\
Correctness-trained token-prob. &\shortstack{$-0.073$ \\ $[-0.414, +0.291]$} & \shortstack{$+0.159$ \\ $[-0.187, +0.410]$} & \shortstack{$-0.095$ \\ $[-0.304, -0.007]$} & \shortstack{$-0.082$ \\ $[-0.248, +0.048]$} & \shortstack{$-0.030$ \\ $[-0.063, -0.000]$} \\
Correctness-trained embedding &\shortstack{$-0.022$ \\ $[-0.300, +0.312]$} & \shortstack{$-0.495$ \\ $[-0.912, -0.100]$} & \shortstack{$-0.254$ \\ $[-0.461, -0.072]$} & \shortstack{$-1.019$ \\ $[-1.426, -0.596]$} & \shortstack{$-0.055$ \\ $[-0.115, -0.007]$} \\
Benefit-trained token-prob. &\shortstack{$-0.127$ \\ $[-0.604, +0.316]$} & \shortstack{$+0.425$ \\ $[+0.182, +0.692]$} & \shortstack{$-0.072$ \\ $[-0.250, +0.090]$} & \shortstack{$-0.161$ \\ $[-0.421, +0.035]$} & \shortstack{$-0.067$ \\ $[-0.131, -0.009]$} \\
Benefit-trained embedding &\shortstack{$-0.954$ \\ $[-1.795, -0.201]$} & \shortstack{$-0.311$ \\ $[-1.049, +0.324]$} & \shortstack{$-1.522$ \\ $[-2.457, -0.631]$} & \shortstack{$-2.245$ \\ $[-3.012, -1.468]$} & \shortstack{$-0.210$ \\ $[-0.518, +0.056]$} \\
\midrule
\multicolumn{6}{c}{Realized-cost difference (\$/1,000 queries)} \\
\midrule
Mean negentropy & $-0.00461$ & $+0.00058$ & $+0.00819$ & $+0.00070$ & $+0.00059$ \\
Min negentropy & $-0.01467$ & $+0.00109$ & $+0.02217$ & $+0.00040$ & $+0.00024$ \\
Prob.\ margin & $-0.00395$ & $+0.00036$ & $+0.00762$ & $+0.00087$ & $+0.00055$ \\
Min probability & $+0.00539$ & $+0.00116$ & $+0.01792$ & $+0.00049$ & $+0.00239$ \\
Length-norm.\ probability & $-0.00484$ & $+0.00066$ & $-0.00105$ & $+0.00059$ & $+0.00071$ \\
Correctness-trained token-prob. &$-0.00464$ & $+0.00110$ & $+0.01919$ & $+0.00070$ & $+0.00046$ \\
Correctness-trained embedding &$+0.00248$ & $+0.00052$ & $+0.03761$ & $+0.00207$ & $+0.00062$ \\
Benefit-trained token-prob. &$-0.00216$ & $+0.00114$ & $+0.03005$ & $+0.00076$ & $+0.00143$ \\
Benefit-trained embedding &$+0.00005$ & $+0.00206$ & $+0.07287$ & $-0.00426$ & $+0.02483$ \\
\bottomrule
\end{tabular}}
\end{table}

Across the 45 scorer--dataset combinations, the mean budget-integrated test-set accuracy difference ranges from $-2.24$ to $+0.43$ percentage points. The largest positive difference occurs for the benefit-trained token-probability score on TriviaQA, with a 10th--90th split-percentile range of $[+0.182, +0.692]$ percentage points.

A third model does not improve mean test accuracy for most scorers and
datasets, although some scorer and dataset combinations show small
positive gains. Table~\ref{tab:scorer-depth} details the third-stage comparison,
while Table~\ref{tab:summary} reports the comparisons at greater depths.

\FloatBarrier

\subsection{Depth Comparison Details}
\label{app:depth-details}

This appendix reports supplementary results for the depth comparison in
Section~\ref{sec:exp-frontier}, using mean token negentropy and the selection
and evaluation procedure of Appendix~\ref{app:exact-depth}.

\paragraph{Gain over standalone selection.}
Table~\ref{tab:pair-standalone} quantifies the improvement from allowing pairs
before assessing further depth. Both classes select their policies on
calibration data at the same assigned budgets. Their realized test costs
can differ because a standalone policy may leave more of its budget unused.
The accuracy differences therefore describe performance under the same
assigned budget, rather than at equal realized cost.

\begin{table}[ht]
\centering
\small
\caption{Pairwise selection versus standalone selection ($S_2-S_1$).
Entries report means over 50 splits of budget-integrated test-set differences.
Brackets give 10th--90th split percentiles for accuracy, not confidence
intervals. Positive values indicate higher accuracy or realized cost for
pairwise selection.}
\label{tab:pair-standalone}
\setlength{\tabcolsep}{4pt}
\begin{tabular}{llll}
\toprule
Dataset & $\Delta$ accuracy (pp) [10th, 90th] & $\Delta$ cost (\$/1,000) & $\Delta$ cost (\% budget) \\
\midrule
MMLU & $+3.925$ [$ +2.222, +5.823 $] & $+0.05024$ & $+42.40$ \\
TriviaQA & $+5.724$ [$ +5.152, +6.264 $] & $+0.00899$ & $+42.14$ \\
MATH & $+2.804$ [$ +2.449, +3.141 $] & $+0.36906$ & $+44.08$ \\
SimpleQA & $+8.216$ [$ +7.375, +8.976 $] & $+0.12719$ & $+60.24$ \\
LiveCodeBench & $+7.008$ [$ +6.138, +7.751 $] & $+0.08164$ & $+34.75$ \\
\bottomrule
\end{tabular}

\end{table}

\begin{figure}[ht]
\centering
\includegraphics[width=\linewidth]{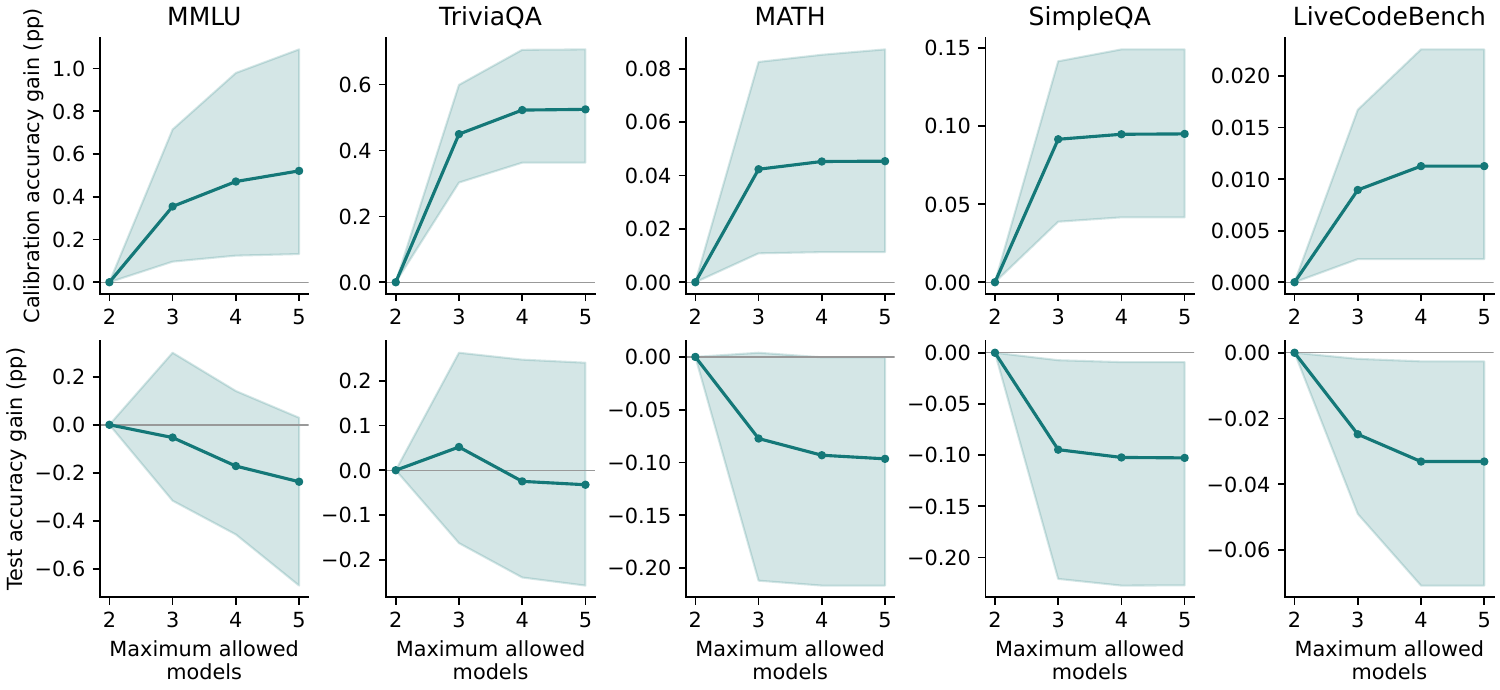}
\caption{Calibration-set and test-set gains as maximum allowed depth increases.
Each point is a mean budget-integrated accuracy difference relative to $S_2$,
with 10th--90th split-percentile shading. Calibration-set gains appear above and
test-set gains below. Calibration-set accuracy weakly increases because the policy
classes are nested. Their selected policies need not improve on test-set data.
Panel scales differ.}
\label{fig:exact-depth-selection}
\end{figure}

\begin{table}[!htbp]
\centering
\scriptsize
\renewcommand{\arraystretch}{0.85}
\caption{Budget dependence and descriptive margin frequencies. Low, middle,
and high columns report mean accuracy differences integrated within the
three equal budget intervals. The final columns give unweighted percentages
of split-budget cells above the stated margins. They are not equivalence tests.}
\label{tab:exact-thirds}
\vspace{4pt}
\begin{tabular}{llrrrrr}
\toprule
Dataset & Comparison & Low (pp) & Middle (pp) & High (pp) & Cells $>0.5$ pp (\%) & Cells $>1$ pp (\%) \\
\midrule
MMLU & $S_3-S_2$ & $-0.077$ & $+0.082$ & $-0.165$ & 8.8 & 4.7 \\
 & $S_4-S_2$ & $-0.116$ & $-0.130$ & $-0.270$ & 7.4 & 2.5 \\
 & $S_5-S_2$ & $-0.134$ & $-0.218$ & $-0.359$ & 5.9 & 2.5 \\
TriviaQA & $S_3-S_2$ & $+0.059$ & $-0.116$ & $+0.212$ & 15.3 & 3.1 \\
 & $S_4-S_2$ & $+0.052$ & $-0.154$ & $+0.028$ & 13.4 & 2.8 \\
 & $S_5-S_2$ & $+0.052$ & $-0.164$ & $+0.015$ & 13.4 & 2.6 \\
MATH & $S_3-S_2$ & $-0.063$ & $-0.072$ & $-0.098$ & 2.0 & 0.2 \\
 & $S_4-S_2$ & $-0.082$ & $-0.092$ & $-0.106$ & 1.9 & 0.2 \\
 & $S_5-S_2$ & $-0.088$ & $-0.096$ & $-0.106$ & 1.9 & 0.1 \\
SimpleQA & $S_3-S_2$ & $-0.048$ & $-0.199$ & $-0.038$ & 2.6 & 0.7 \\
 & $S_4-S_2$ & $-0.068$ & $-0.200$ & $-0.038$ & 2.6 & 0.6 \\
 & $S_5-S_2$ & $-0.070$ & $-0.200$ & $-0.038$ & 2.6 & 0.6 \\
LiveCodeBench & $S_3-S_2$ & $-0.018$ & $-0.017$ & $-0.039$ & 0.1 & 0.0 \\
 & $S_4-S_2$ & $-0.019$ & $-0.019$ & $-0.061$ & 0.1 & 0.0 \\
 & $S_5-S_2$ & $-0.019$ & $-0.019$ & $-0.061$ & 0.1 & 0.0 \\
\bottomrule
\end{tabular}

\end{table}

\begin{table}[!htbp]
\centering
\scriptsize
\renewcommand{\arraystretch}{0.85}
\caption{Test-set budget feasibility and expected calls. Overshoot frequency
and mean positive relative excess use the common trapezoidal budget weights.
Overshoots exceed the assigned budget by more than $10^{-12}$ dollars per query.
Mean excess includes zero-excess cells. The 90th percentile is the unweighted
empirical percentile over all split-budget cells, including zeros.}
\label{tab:exact-feasibility}
\vspace{4pt}
\begin{tabular}{llrrrr}
\toprule
Dataset & Class & Overshoot (\%) & Mean excess (\%) & 90th excess (\%) & Calls/query \\
\midrule
MMLU & $S_1$ & 0.6 & 0.01 & 0.00 & 1.000 \\
 & $S_2$ & 25.7 & 1.30 & 5.13 & 1.248 \\
 & $S_3$ & 30.7 & 1.71 & 6.33 & 1.379 \\
 & $S_4$ & 32.6 & 1.96 & 6.97 & 1.492 \\
 & $S_5$ & 33.4 & 2.05 & 7.33 & 1.561 \\
TriviaQA & $S_1$ & 0.8 & 0.01 & 0.00 & 1.000 \\
 & $S_2$ & 47.5 & 5.12 & 17.38 & 1.344 \\
 & $S_3$ & 54.0 & 5.70 & 18.52 & 1.378 \\
 & $S_4$ & 56.1 & 5.92 & 19.02 & 1.452 \\
 & $S_5$ & 56.3 & 5.95 & 19.05 & 1.459 \\
MATH & $S_1$ & 0.7 & 0.02 & 0.00 & 1.000 \\
 & $S_2$ & 25.9 & 1.30 & 5.01 & 1.259 \\
 & $S_3$ & 27.4 & 1.44 & 5.52 & 1.300 \\
 & $S_4$ & 27.5 & 1.47 & 5.58 & 1.308 \\
 & $S_5$ & 27.5 & 1.47 & 5.55 & 1.308 \\
SimpleQA & $S_1$ & 0.3 & 0.00 & 0.00 & 1.000 \\
 & $S_2$ & 41.0 & 1.36 & 4.27 & 1.431 \\
 & $S_3$ & 45.6 & 1.65 & 5.09 & 1.488 \\
 & $S_4$ & 46.0 & 1.68 & 5.18 & 1.493 \\
 & $S_5$ & 46.0 & 1.68 & 5.20 & 1.493 \\
LiveCodeBench & $S_1$ & 1.2 & 0.02 & 0.00 & 1.000 \\
 & $S_2$ & 50.3 & 2.40 & 7.67 & 1.219 \\
 & $S_3$ & 51.6 & 2.59 & 8.21 & 1.219 \\
 & $S_4$ & 52.5 & 2.66 & 8.36 & 1.219 \\
 & $S_5$ & 52.5 & 2.66 & 8.36 & 1.219 \\
\bottomrule
\end{tabular}

\end{table}

\begin{table}[!htbp]
\centering
\scriptsize
\renewcommand{\arraystretch}{0.85}
\setlength{\tabcolsep}{4pt}
\begin{minipage}[t]{0.53\linewidth}
\centering
\caption{Nominal lengths of frozen exact selections. Entries are percentages
of budget weight, averaged equally across splits. Nominal model sequence length counts the models in the selected sequence.
Table~\ref{tab:exact-reach} reports how often queries actually reach the later stages.}
\label{tab:exact-usage}
\resizebox{\linewidth}{!}{\begin{tabular}{lllllll}
\toprule
Dataset & Class & 1 model & 2 models & 3 models & 4 models & 5 models \\
\midrule
MMLU & $S_2$ & 0.3 & 99.7 & 0.0 & 0.0 & 0.0 \\
 & $S_3$ & 0.3 & 17.8 & 81.9 & 0.0 & 0.0 \\
 & $S_4$ & 0.3 & 15.5 & 23.1 & 61.1 & 0.0 \\
 & $S_5$ & 0.3 & 15.4 & 21.7 & 28.7 & 34.0 \\
TriviaQA & $S_2$ & 0.2 & 99.8 & 0.0 & 0.0 & 0.0 \\
 & $S_3$ & 0.2 & 10.6 & 89.2 & 0.0 & 0.0 \\
 & $S_4$ & 0.2 & 10.5 & 47.1 & 42.2 & 0.0 \\
 & $S_5$ & 0.2 & 10.5 & 46.9 & 37.1 & 5.3 \\
MATH & $S_2$ & 0.7 & 99.3 & 0.0 & 0.0 & 0.0 \\
 & $S_3$ & 0.6 & 56.3 & 43.2 & 0.0 & 0.0 \\
 & $S_4$ & 0.6 & 55.3 & 38.2 & 5.9 & 0.0 \\
 & $S_5$ & 0.6 & 55.3 & 37.8 & 5.6 & 0.7 \\
SimpleQA & $S_2$ & 0.5 & 99.5 & 0.0 & 0.0 & 0.0 \\
 & $S_3$ & 0.4 & 66.5 & 33.0 & 0.0 & 0.0 \\
 & $S_4$ & 0.4 & 65.7 & 28.4 & 5.5 & 0.0 \\
 & $S_5$ & 0.4 & 65.6 & 28.2 & 5.0 & 0.7 \\
LiveCodeBench & $S_2$ & 2.3 & 97.7 & 0.0 & 0.0 & 0.0 \\
 & $S_3$ & 2.0 & 90.1 & 7.9 & 0.0 & 0.0 \\
 & $S_4$ & 2.0 & 88.4 & 7.2 & 2.4 & 0.0 \\
 & $S_5$ & 2.0 & 88.4 & 7.2 & 2.4 & 0.0 \\
\bottomrule
\end{tabular}
}
\end{minipage}\hfill
\begin{minipage}[t]{0.44\linewidth}
\centering
\caption{Observed test-set reach rates of later stages, as percentages of
incoming queries. Entries use trapezoidal budget weights and equal split
weights. Values rounded to zero can conceal small positive rates. Per-stage stop
probabilities are retained in the reporting outputs.}
\label{tab:exact-reach}
\resizebox{\linewidth}{!}{\begin{tabular}{lllll}
\toprule
Dataset & Class & Reach 3 (\%) & Reach 4 (\%) & Reach 5 (\%) \\
\midrule
MMLU & $S_3$ & 7.52 & 0.00 & 0.00 \\
 & $S_4$ & 11.38 & 4.18 & 0.00 \\
 & $S_5$ & 12.50 & 6.37 & 1.51 \\
TriviaQA & $S_3$ & 4.60 & 0.00 & 0.00 \\
 & $S_4$ & 10.35 & 0.40 & 0.00 \\
 & $S_5$ & 10.51 & 0.63 & 0.03 \\
MATH & $S_3$ & 4.24 & 0.00 & 0.00 \\
 & $S_4$ & 4.37 & 0.29 & 0.00 \\
 & $S_5$ & 4.37 & 0.30 & 0.00 \\
SimpleQA & $S_3$ & 3.36 & 0.00 & 0.00 \\
 & $S_4$ & 3.48 & 0.06 & 0.00 \\
 & $S_5$ & 3.48 & 0.06 & 0.00 \\
LiveCodeBench & $S_3$ & 0.07 & 0.00 & 0.00 \\
 & $S_4$ & 0.09 & 0.01 & 0.00 \\
 & $S_5$ & 0.09 & 0.01 & 0.00 \\
\bottomrule
\end{tabular}
}
\end{minipage}
\end{table}

\begin{table}[!htbp]
\centering
\scriptsize
\caption{Completed test-set five-model extensions. Times describe each split's
extension of the already completed depth-four search, rather than total work
through five models. Wall time includes selection, witness replay, and frozen
evaluation. Kernel time measures compiled search.}
\label{tab:exact-search}
\begin{tabular}{lllllll}
\toprule
Dataset & Complete & 5-model sequences & Wall median (s) & Wall 90th & Kernel median (s) & Kernel 90th \\
\midrule
MMLU & 50/50 & 56 & 2.8 & 373.6 & 2.5 & 373.3 \\
TriviaQA & 50/50 & 56 & 7.0 & 8.2 & 6.7 & 8.0 \\
MATH & 50/50 & 56 & 4.4 & 5.8 & 4.2 & 5.5 \\
SimpleQA & 50/50 & 56 & 130.9 & 154.7 & 130.7 & 154.5 \\
LiveCodeBench & 50/50 & 56 & 0.2 & 0.3 & 0.0 & 0.1 \\
\bottomrule
\end{tabular}

\end{table}

\FloatBarrier

\subsection{Full-Sample Oracle Value of Depth}
\label{app:oracle-depth}

This analysis evaluates the value of depth $G_3(B)=U_3^*(B)-U_2^*(B)$ from
Section~\ref{sec:depth-value} under the full-sample empirical distribution.
The at-most-three class contains every pair-or-single policy. As a separate
descriptive analysis of empirical policy expressivity, we perform a
full-sample oracle analysis with no calibration--test split. Using all complete
examples and the full eight-model pool, we compare the best pairwise
value function $U_{\mathrm{pair}}^\star(B)$ with the best value function using
up to three models, $U_{\leq 3}^\star(B)$. The latter class explicitly includes
all single-model and pairwise policies, so it nests the pairwise class.
These frontiers are the empirical counterparts of $U_2^*(B)$ and $U_3^*(B)$
in Equation~\ref{eq:depth-frontier}. For
every cost-ordered model pair and triple, we exactly enumerate all empirically
distinct one- and two-threshold policies using score-rank cumulative sums; no
stochastic optimizer is used. We evaluate
\[
    \Delta_{\mathrm{oracle}}(B)
    = U_{\leq 3}^\star(B)-U_{\mathrm{pair}}^\star(B)
\]
on 10{,}001 evenly spaced budgets over the support from the lowest-cost model
to the highest-accuracy standalone model. Because policy selection and
evaluation use the same examples, these values describe the observed sample
and are not deployment or generalization estimates. The fitting sample sizes
also differ from the test-set experiment, so their difference is not a pure
estimation-error estimate. This appendix remains scoped to depth three.

\begin{table}[h]
\centering
\small
\caption{Descriptive full-sample oracle value of adding a third stage. Gaps are
accuracy differences in percentage points. Support columns give the fraction of the budget
grid on which $\Delta_{\mathrm{oracle}}$ strictly exceeds the stated margin.}
\label{tab:oracle-depth}
\begin{tabular}{lrrrr}
\toprule
Dataset & Grid max & Grid mean & $>0.5$ pp & $>1.0$ pp \\
\midrule
MMLU          & 0.35 & 0.07 & 0.0\%  & 0.0\% \\
TriviaQA      & 0.90 & 0.36 & 33.9\% & 0.0\% \\
MATH          & 0.10 & 0.01 & 0.0\%  & 0.0\% \\
SimpleQA      & 0.63 & 0.05 & 0.9\%  & 0.0\% \\
LiveCodeBench & 0.09 & 0.001 & 0.0\%  & 0.0\% \\
\bottomrule
\end{tabular}
\end{table}

Depth has modest oracle value in the observed samples, most visibly on
TriviaQA. The gain never exceeds $1$ percentage point in accuracy, and outside TriviaQA it
exceeds $0.5$ percentage points on at most $0.9\%$ of budget support. Together with the
test-set results, this suggests that the available depth advantage is generally
too small to select and transfer reliably.

\subsection{Calibration Learning Curve with a Fixed Test Set}
\label{app:cal-sensitivity}

Does more calibration-set data make the additional flexibility of three-stage
cascades useful on unseen queries? We compare exact search over standalone
models and cost-ordered pairs with exact search over standalone models, pairs,
and cost-ordered triples, using the selector in
Appendix~\ref{app:exact-depth}. The full eight-model pool
is retained without dominance pruning or
accuracy-based edge restrictions. Mean token negentropy, threshold endpoints,
strict escalation inequalities, and tie-breaking are identical to that analysis.
Mean-cost ordering is recomputed using only the current calibration subset.

For each seed from 42 through 91, a stratified 50/50 split assigns
a fixed test set and a calibration reservoir. Stratification uses correctness
of the first available model and rounds down the reservoir count within each
stratum. Within the reservoir, a seeded random permutation of each stratum is
interleaved by its proportional count deficit to form one stratified ordering.
Prefixes containing 10\%, 25\%, 50\%, 75\%, and 100\% of the reservoir,
rounded down to whole examples, give nested calibration subsets. Actual counts
are reported in Table~\ref{tab:cal_sensitivity}. Joint finite-row filtering
and stratified rounding leave reservoir/test counts of 999/1{,}000 for MMLU,
999/1{,}001 for TriviaQA, 985/986 for MATH, 945/946 for SimpleQA, and
527/528 for LiveCodeBench. These counts are constant across seeds.

This is a learning curve conditional on a fixed, reservoir-defined budget
schedule. For each seed, we define 500 assigned dollar budgets once using
Equation~\ref{eq:deployment-budget} on the full reservoir and hold them fixed
across calibration sizes. Only this externally supplied budget schedule uses
the full reservoir for policy selection. Search at smaller sizes receives
only the current subset and those budgets. At each size and budget, each
class selects the highest-calibration-accuracy feasible policy, breaking ties
by lower calibration-set cost and then the fixed enumeration order. A missing
feasible policy is recorded as infeasible. Selected sequences, thresholds,
estimated costs, and candidate identifiers for all sizes are saved before any
test-set evaluation. Frozen policies are evaluated on the same test-set examples
without test-based selection, filtering, or envelope construction. Realized
costs, overshoots, nominal depth, and expected calls are retained.

Within each seed, we restrict every size and both classes to the same budget
indices that are calibration-feasible throughout. We average paired
$\leq3$ minus $\leq2$ differences equally over those original grid positions,
then report the mean and 10th--90th percentiles across seeds. The bands measure
split sensitivity, not confidence intervals. Table~\ref{tab:cal_sensitivity}
reports common coverage and infeasibility rates on the unrestricted grid.
Calibration-set accuracy in the larger class is at least that in the smaller
class wherever both are feasible, but test-set accuracy need not be.

\begin{figure}[htbp]
\centering
\includegraphics[width=\linewidth]{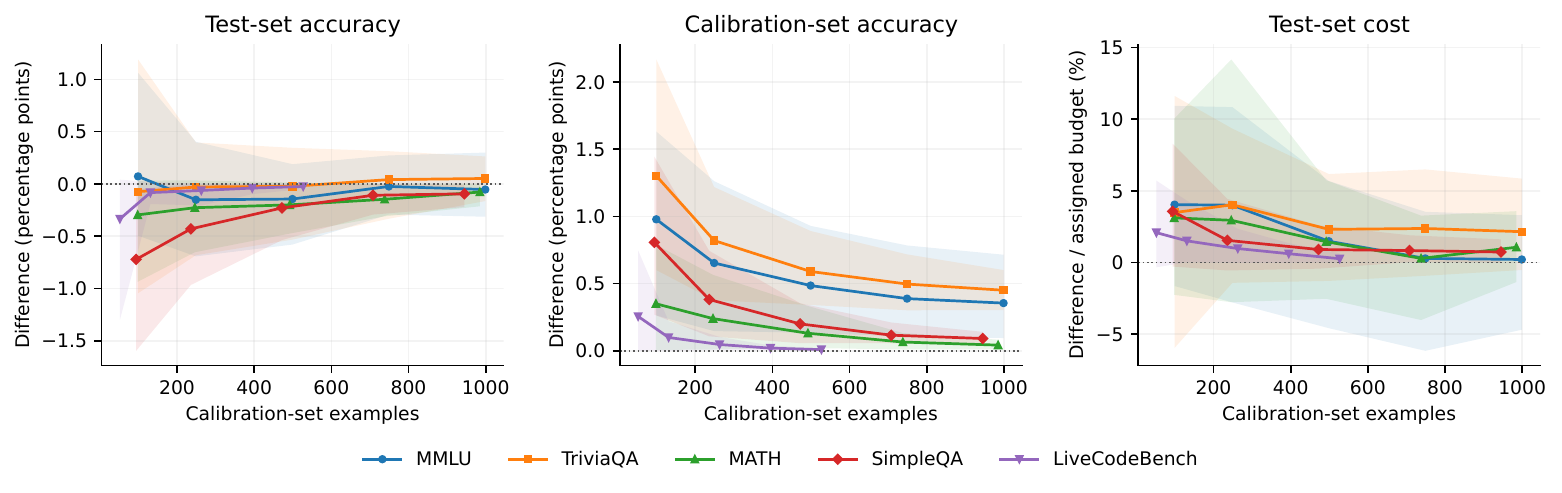}
\caption{Fixed-test calibration learning curve. Each panel reports
budget-averaged paired differences for exact search with at most three models
minus at most two. Left, test-set accuracy difference in percentage points. Center,
calibration-set accuracy difference in percentage points. Right, realized test-set cost
difference divided by assigned budget, expressed as a percentage. The
horizontal axis gives actual calibration-example counts. Lines show means
and shading shows 10th--90th split percentiles over seeds 42--91, using a
common feasible budget mask across all sizes and methods within each seed.}
\label{fig:cal_sensitivity}
\end{figure}

\begin{table}[htbp]
\centering
\scriptsize
\caption{Fixed-test learning-curve diagnostics, averaged within each seed and then across 50 seeds. Coverage is the percentage of the original 500 budgets feasible for both classes at every size. Infeasibility uses the original grid. Other diagnostics use the common feasible grid. Each paired entry reports $\leq2$ / $\leq3$ models. Overshoot frequency counts realized cost above budget, and magnitude averages $100\max(C-B,0)/B$, including zeros. Calls are expected model calls per query.}
\label{tab:cal_sensitivity}
\begin{tabular}{lrrrrrr}
\toprule
Dataset & $n_{\rm cal}$ & Coverage & Infeasible & Overshoot & Magnitude & Calls \\
 & & (\%) & (\%) & (\%) & (\%) & \\
\midrule
MMLU & 99 & 99.5 & 0.4 / 0.4 & 22.5 / 28.6 & 2.92 / 4.36 & 1.22 / 1.35 \\
 & 249 & 99.5 & 0.2 / 0.2 & 26.3 / 34.6 & 2.29 / 3.47 & 1.24 / 1.36 \\
 & 499 & 99.5 & 0.2 / 0.2 & 23.0 / 28.2 & 1.59 / 2.08 & 1.24 / 1.37 \\
 & 749 & 99.5 & 0.1 / 0.1 & 25.4 / 30.2 & 1.40 / 1.82 & 1.25 / 1.40 \\
 & 999 & 99.5 & 0.0 / 0.0 & 25.5 / 30.5 & 1.29 / 1.70 & 1.25 / 1.38 \\
\addlinespace
TriviaQA & 99 & 99.8 & 0.1 / 0.1 & 32.6 / 38.6 & 3.43 / 4.51 & 1.29 / 1.45 \\
 & 249 & 99.8 & 0.1 / 0.1 & 44.8 / 52.3 & 4.34 / 6.28 & 1.33 / 1.44 \\
 & 499 & 99.8 & 0.1 / 0.1 & 50.1 / 56.8 & 5.36 / 6.24 & 1.34 / 1.41 \\
 & 749 & 99.8 & 0.1 / 0.1 & 47.9 / 56.4 & 5.52 / 6.17 & 1.34 / 1.40 \\
 & 999 & 99.8 & 0.0 / 0.0 & 47.5 / 54.0 & 5.13 / 5.71 & 1.34 / 1.38 \\
\addlinespace
MATH & 98 & 99.7 & 0.3 / 0.3 & 17.9 / 19.9 & 1.86 / 2.11 & 1.24 / 1.28 \\
 & 246 & 99.7 & 0.2 / 0.2 & 19.6 / 20.7 & 1.40 / 1.54 & 1.21 / 1.30 \\
 & 492 & 99.7 & 0.1 / 0.1 & 20.2 / 21.1 & 1.21 / 1.33 & 1.22 / 1.28 \\
 & 738 & 99.7 & 0.1 / 0.1 & 21.9 / 23.6 & 1.17 / 1.34 & 1.24 / 1.28 \\
 & 985 & 99.7 & 0.0 / 0.0 & 25.8 / 27.3 & 1.29 / 1.44 & 1.26 / 1.30 \\
\addlinespace
SimpleQA & 94 & 99.8 & 0.1 / 0.1 & 29.8 / 36.6 & 2.49 / 3.85 & 1.43 / 1.79 \\
 & 236 & 99.8 & 0.1 / 0.1 & 36.3 / 42.5 & 2.22 / 2.81 & 1.43 / 1.72 \\
 & 472 & 99.8 & 0.1 / 0.1 & 36.8 / 41.2 & 1.61 / 1.92 & 1.43 / 1.56 \\
 & 708 & 99.8 & 0.1 / 0.1 & 40.7 / 45.6 & 1.49 / 1.83 & 1.43 / 1.51 \\
 & 945 & 99.8 & 0.0 / 0.0 & 41.0 / 45.6 & 1.36 / 1.65 & 1.43 / 1.49 \\
\addlinespace
LiveCodeBench & 52 & 99.3 & 0.5 / 0.5 & 43.1 / 47.3 & 5.78 / 7.20 & 1.20 / 1.22 \\
 & 131 & 99.3 & 0.3 / 0.3 & 44.6 / 49.3 & 3.40 / 4.37 & 1.21 / 1.21 \\
 & 263 & 99.3 & 0.2 / 0.2 & 50.7 / 54.2 & 2.76 / 3.47 & 1.21 / 1.22 \\
 & 395 & 99.3 & 0.2 / 0.2 & 49.0 / 51.5 & 2.49 / 2.96 & 1.22 / 1.22 \\
 & 527 & 99.3 & 0.0 / 0.0 & 50.3 / 51.5 & 2.40 / 2.60 & 1.22 / 1.22 \\
\bottomrule
\end{tabular}
\end{table}

\FloatBarrier

Calibration-set gains remain positive but decline as the calibration subsets grow.
At the smallest sizes, mean fitted gains range from 0.256 percentage points to 1.302 percentage points,
compared with 0.009 percentage points to 0.451 percentage points at the largest sizes. Test-set differences
generally move toward zero as calibration-set data increase, although the pattern
is not monotone for every dataset. At the largest sizes, mean test-set
accuracy differences are $-0.054$ percentage points on MMLU, $+0.052$ percentage points on
TriviaQA, $-0.078$ percentage points on MATH, $-0.095$ percentage points on SimpleQA, and $-0.025$ percentage points on
LiveCodeBench. The three-model class also spends more on average, with
budget-normalized cost differences of 0.21\% to 2.15\% at those sizes.
Common feasible coverage ranges from 99.32\% to 99.82\% across datasets.
These results suggest that more calibration-set data reduce losses from selecting
additional depth, but do not reveal a substantial average test-set advantage
from a third stage at the tested sizes. Persistent fitted gains, particularly
on MMLU and TriviaQA, do not transfer into comparable test-set gains.

The design tests at most 1{,}000 calibration-set examples on the datasets with
2{,}000 examples.
Testing larger calibration sizes while preserving a substantial common test
set requires additional examples. This comparison isolates calibration size
within the available data and the fixed budget schedule.

\subsection{Synthetic-Score Construction and Matched Diagnostics}
\label{app:synthetic-scores}

The synthetic-score experiment changes the scores while retaining each
model's recorded correctness label and generation cost. We retain all eight
models on every dataset, including GPT-oss-20B on LiveCodeBench. For each dataset,
we use the first 2,000 aligned queries, or all 1,055 LiveCodeBench queries.
Let $Y\in\mathbb R^{n\times K}$ contain the centered and standardized
correctness labels of the $K$ models. Missing labels are replaced by their
model-specific mean when forming $Y$. AUROC is calculated only on rows
with an observed label. Let $C$ be the Pearson correlation matrix of
observed mean token negentropy scores, computed with pairwise finite rows
and stored to six decimal places. This matrix is the correlation target,
not a rank correlation or a correlation of correctness labels.

For each dataset, we draw a Gaussian $n\times K$ matrix, project it off
the column space of $[\mathbf 1,Y]$, and whiten the residual to obtain
$E$ with sample covariance $I$ and sample covariance zero with $Y$.
For a coefficient vector $a$, define
\[
R(a)=C-\operatorname{diag}(a)\operatorname{Corr}(Y)
              \operatorname{diag}(a),\qquad
S(a)=Y\operatorname{diag}(a)+E\operatorname{chol}(R(a))^\top.
\]
We require $R(a)$ to be positive definite. The sample Pearson correlation
matrix of $S(a)$ then equals $C$, up to the stored precision of $C$.
For each target in $\{0.5,0.6,0.7,0.8,0.9\}$, we calibrate the entries of
$a$ by coordinate bisection until every model's empirical AUROC for its
own correctness label is within $0.001$ of the target. The calibration
allows up to 30 coordinate sweeps and 45 bisection steps per coordinate.
A positive affine transformation maps each score column to $[0,1]$ without
changing its AUROC or Pearson correlations. For the original draw in
Figure~\ref{fig:signal-quality}, we use NumPy's default random
generator with seeds $20260917$ through $20260921$ for MMLU, TriviaQA,
MATH, SimpleQA, and LiveCodeBench, respectively. The same noise draw is
reused across AUROC targets within each dataset. Each score's explicit
label term uses its own model's correctness. Correlated labels and the
shared residual covariance create dependence across model scores.
Generation costs are not used in the construction.

We create scores before the calibration-test splits and then apply the
same complete-case filter used by the main depth experiment. This filter
retains 1,999, 2,000, 1,971, 1,891, and 1,055 queries, respectively.
For this original draw, after filtering, the largest absolute difference between a synthetic
score correlation and the stored full-sample target matrix is $0.011$.
The largest deviation of a model's achieved correctness AUROC from its
target on these rows is less than $0.008$. These are filtering effects. The
construction matches its targets on the original rows to within $0.001$
AUROC and $10^{-10}$ correlation.

Table~\ref{tab:synthetic-diagnostics} compares real and synthetic scores
on those same complete-case rows. Correctness AUROC predicts each model's
own binary correctness using its score. Benefit-AUROC predicts
$\mathbf 1\{U_H>U_L\}$ using the negative lower-cost model score. We
average benefit-AUROC equally across all cost-ordered pairs, assigning
$0.5$ when the benefit label has one class. These are full-sample
descriptive diagnostics, not held-out policy evaluations. They differ
from Table~\ref{tab:voi_summary}, which averages pair-specific medians
across calibration-test splits.

\begin{table}[h]
\centering
\small
\caption{Matched real and synthetic score diagnostics on the complete-case
rows used for depth evaluation. Correctness AUROC is the median across
models, with the model range in brackets. Benefit-AUROC averages across
all cost-ordered model pairs. Synthetic columns use target correctness
AUROC $0.8$ and one draw per dataset.}
\label{tab:synthetic-diagnostics}
\setlength{\tabcolsep}{4pt}
\begin{tabular}{lrrrcc}
\toprule
Dataset & Queries & Models & \shortstack{Real correctness\\AUROC} & \shortstack{Real benefit-\\AUROC} & \shortstack{Synthetic benefit-\\AUROC} \\
\midrule
MMLU & 1,999 & 8 & $0.771\,[0.493,0.834]$ & $0.715$ & $0.758$ \\
TriviaQA & 2,000 & 8 & $0.729\,[0.645,0.847]$ & $0.729$ & $0.736$ \\
MATH & 1,971 & 8 & $0.653\,[0.514,0.750]$ & $0.568$ & $0.703$ \\
SimpleQA & 1,891 & 8 & $0.705\,[0.598,0.765]$ & $0.475$ & $0.538$ \\
LiveCodeBench & 1,055 & 8 & $0.774\,[0.713,0.856]$ & $0.558$ & $0.658$ \\
\bottomrule
\end{tabular}

\end{table}

The synthetic scores have median correctness AUROC between $0.799$ and
$0.800$ after filtering. Some real model scores already exceed $0.8$
correctness AUROC. The mean benefit-AUROC of real scores differs from
that of the synthetic scores, particularly on MATH and LiveCodeBench.
Correctness AUROC alone therefore does not establish that these scores
support the same escalation decisions.
The depth crossover in Figure~\ref{fig:signal-quality} describes this
score construction and one draw. We assess variation across draws below.

\paragraph{Variation across synthetic draws.}
We generate five additional independent score draws per dataset at target
correctness AUROC $0.8$ and $0.9$. The original draw is excluded from these
five replications. Each replication uses a fresh Gaussian matrix in the
construction above, with the same noise shared between its two AUROC targets.
We initialize NumPy's default random generator with
\texttt{SeedSequence([20260921, dataset index, draw index])}, using zero-based
indices and the dataset order MMLU, TriviaQA, MATH, SimpleQA, and LiveCodeBench.
The AUROC and Pearson correlation constraints are unchanged. We retain all
eight models and exactly match the original query rows, correctness labels,
costs, 50 calibration-test splits, and 500 assigned budget positions.
For each draw and target, we repeat exact calibration selection over $S_2$
and $S_3$ and evaluate the frozen policies on the corresponding test sets.

For each split, we integrate the held-out $S_3-S_2$ accuracy difference
over normalized assigned budget positions. A draw's mean gain averages
these 50 integrals. Table~\ref{tab:synthetic-draw-means} reports each draw's
mean and the mean and sample standard deviation across the five draw means.
Table~\ref{tab:synthetic-draw-splits} separately reports the sample standard
deviation across splits within each draw. The latter describes sensitivity
to overlapping calibration-test splits conditional on a score draw.
Neither standard deviation is a confidence interval. Variation across
draws conditions on the recorded data, correlation target, and fixed split
schedule. Because score construction uses full-sample correctness labels,
held-out evaluation applies to cascade selection rather than scorer learning.
The comparisons match assigned budgets, and realized test costs can differ.

\begin{table}[htbp]
\centering
\small
\caption{Variation across five independent synthetic score draws.
Entries are budget-integrated held-out $S_3-S_2$ accuracy gains in percentage
points, averaged over 50 matched splits within each draw. The final two
columns report the mean and sample SD of the five draw means. Draw numbers
1--5 correspond to stored draw indices 0--4.}
\label{tab:synthetic-draw-means}
\setlength{\tabcolsep}{3pt}
\begin{tabular}{llrrrrrrr}
\toprule
& & \multicolumn{5}{c}{Mean gain for each draw} & \multicolumn{2}{c}{Across draws} \\
\cmidrule(lr){3-7}\cmidrule(lr){8-9}
Dataset & AUROC & 1 & 2 & 3 & 4 & 5 & Mean & SD \\
\midrule
MMLU & 0.8 & $+0.623$ & $+0.619$ & $+0.639$ & $+0.379$ & $+0.428$ & $+0.537$ & $0.124$ \\
 & 0.9 & $+1.322$ & $+0.852$ & $+0.997$ & $+0.718$ & $+0.701$ & $+0.918$ & $0.255$ \\
\addlinespace[2pt]
TriviaQA & 0.8 & $+0.131$ & $+0.344$ & $-0.235$ & $-0.008$ & $+0.009$ & $+0.048$ & $0.212$ \\
 & 0.9 & $+1.048$ & $+0.982$ & $+0.834$ & $+0.722$ & $+1.089$ & $+0.935$ & $0.153$ \\
\addlinespace[2pt]
MATH & 0.8 & $+0.052$ & $+0.004$ & $-0.167$ & $-0.076$ & $+0.118$ & $-0.014$ & $0.111$ \\
 & 0.9 & $+0.363$ & $+0.375$ & $+0.289$ & $+0.277$ & $+0.461$ & $+0.353$ & $0.074$ \\
\addlinespace[2pt]
SimpleQA & 0.8 & $-0.105$ & $+0.066$ & $-0.067$ & $-0.079$ & $-0.352$ & $-0.108$ & $0.152$ \\
 & 0.9 & $+0.461$ & $+0.628$ & $+0.668$ & $+0.643$ & $+0.369$ & $+0.554$ & $0.131$ \\
\addlinespace[2pt]
LiveCodeBench & 0.8 & $-0.046$ & $-0.050$ & $+0.004$ & $-0.088$ & $-0.060$ & $-0.048$ & $0.033$ \\
 & 0.9 & $-0.065$ & $-0.070$ & $+0.033$ & $-0.077$ & $-0.045$ & $-0.045$ & $0.045$ \\
\bottomrule
\end{tabular}

\end{table}

\begin{table}[htbp]
\centering
\small
\caption{Split variability conditional on each synthetic score draw.
Entries are sample SDs of the 50 budget-integrated held-out $S_3-S_2$
accuracy gains, in percentage points. Draws match
Table~\ref{tab:synthetic-draw-means}. These SDs describe variation across
overlapping splits within a draw, separately from variation across draw means.}
\label{tab:synthetic-draw-splits}
\setlength{\tabcolsep}{6pt}
\begin{tabular}{llrrrrr}
\toprule
& & \multicolumn{5}{c}{SD across splits within each draw} \\
\cmidrule(lr){3-7}
Dataset & AUROC & 1 & 2 & 3 & 4 & 5 \\
\midrule
MMLU & 0.8 & $0.399$ & $0.312$ & $0.419$ & $0.266$ & $0.292$ \\
 & 0.9 & $0.539$ & $0.397$ & $0.466$ & $0.263$ & $0.349$ \\
\addlinespace[2pt]
TriviaQA & 0.8 & $0.348$ & $0.310$ & $0.215$ & $0.231$ & $0.252$ \\
 & 0.9 & $0.256$ & $0.280$ & $0.238$ & $0.206$ & $0.264$ \\
\addlinespace[2pt]
MATH & 0.8 & $0.217$ & $0.150$ & $0.221$ & $0.167$ & $0.199$ \\
 & 0.9 & $0.236$ & $0.219$ & $0.155$ & $0.213$ & $0.170$ \\
\addlinespace[2pt]
SimpleQA & 0.8 & $0.326$ & $0.332$ & $0.389$ & $0.344$ & $0.264$ \\
 & 0.9 & $0.348$ & $0.328$ & $0.330$ & $0.350$ & $0.311$ \\
\addlinespace[2pt]
LiveCodeBench & 0.8 & $0.066$ & $0.105$ & $0.109$ & $0.083$ & $0.062$ \\
 & 0.9 & $0.068$ & $0.118$ & $0.107$ & $0.072$ & $0.060$ \\
\bottomrule
\end{tabular}

\end{table}

At AUROC $0.9$, every draw has a positive mean gain on MMLU, TriviaQA,
MATH, and SimpleQA. Their means across draws are $0.918$, $0.935$,
$0.353$, and $0.554$ percentage points, respectively. LiveCodeBench has
a negative mean in four of five draws and an overall mean of $-0.045$
points. At AUROC $0.8$, MMLU is positive in every draw, whereas each
other dataset has both positive and negative draw means. In particular,
the mean across draws is $0.048$ points on TriviaQA and $-0.014$ on MATH.
Thus, the third-stage benefit at AUROC $0.9$ is consistent across these
draws on four benchmarks, while the AUROC level at which gains first
become positive is sensitive to the draw. These results concern this
score construction and do not establish robustness across alternative
score constructions.
\FloatBarrier

\subsection{Synthetic Scores with Shared Query Difficulty}
\label{app:shared-difficulty}

\paragraph{Construction.}
The scores in Appendix~\ref{app:synthetic-scores} contain each model's own
correctness label and noise that is uncorrelated with every model's labels
in the construction sample. We add a difficulty term. Let $d_j(x)$ denote
the fraction of the other seven models that answer $x$ incorrectly, and
let tildes denote centered and standardized columns. Missing labels are
replaced by their model-specific means when forming both terms, as in
Appendix~\ref{app:synthetic-scores}. For $w\in[0,1]$, define
\[
s_j=a_j\bigl[(1-w)\widetilde Y_j-w\widetilde d_j\bigr]+e_j,
\]
where $e_j$ is the residual noise column from the standardized $w=0$
construction at the same target. We hold this residual fixed across $w$.
For each model separately, we locate an empirical AUROC crossing on a
64-interval grid in the bounded coefficient $t=a_j/(1+a_j)$, then use
up to 50 bisection steps until full-sample correctness AUROC is within
$0.001$ of the target. This parameterization is equivalent up to positive
scaling to the displayed construction. A positive affine transformation
maps the scores to $[0,1]$. At $w=0$, we retain the original scores exactly.
We do not re-impose the Pearson correlation target, and score correlations
across models change by up to $0.970$ in absolute value.

Difficulty alone cannot reach every target. We search $w$ on the grid
$\{0,0.01,\ldots,1\}$ and report its largest feasible point for each
dataset and target. A candidate must have signal-only AUROC at least the
target minus $0.001$ for every model and must admit successful finite-noise
calibration for all eight models. This is a conservative operational rule,
not a proof of the global feasibility limit, since empirical AUROC need
not be monotone in the coefficient. Difficulty is discrete and a small
own-label term breaks its ties, so a maximum of $0.99$ should not be read
as a score built from difficulty alone. We use $w\in\{0,0.5,\max\}$,
targets $0.8$ and $0.9$, one score draw, and the same rows, costs, 50 splits,
and 500 budget positions as Appendix~\ref{app:synthetic-scores}. We repeat
exact selection over $S_2$ and $S_3$.

\paragraph{Depth gains.}
At target $0.9$, the mean $S_3-S_2$ gain on MMLU, TriviaQA, MATH, and
SimpleQA falls from between $+0.44$ and $+1.07$ percentage points at $w=0$
to between $-0.15$ and $-0.04$ at maximum $w$
(Table~\ref{tab:shared-depth-gains}). LiveCodeBench is negative throughout.
At target $0.8$, MMLU is the only dataset with a gain above $0.03$ points
at $w=0$, and its gain also becomes negative.

\begin{table}[htbp]
\centering
\scriptsize
\caption{Shared-difficulty depth gains and realized test-cost differences.
$S_3-S_2$ is the mean test-set $S_3-S_2$ accuracy difference, integrated over
assigned budgets and averaged over 50 splits, in percentage points.
Parentheses give the sample SD across splits. Cost differences are also
$S_3-S_2$, in dollars per 1,000 queries. Each three-column group corresponds
to $w=0$, $0.5$, and the maximum feasible weight in the last column.
Comparisons match assigned budgets rather than realized test costs.}
\label{tab:shared-depth-gains}
\setlength{\tabcolsep}{2pt}
\begin{tabular*}{\linewidth}{@{\extracolsep{\fill}}llrrrrrrr@{}}
\toprule
& & \multicolumn{3}{c}{Gain (split SD), pp} & \multicolumn{3}{c}{$\Delta$ test cost (\$/1,000)} & \\
\cmidrule(lr){3-5}\cmidrule(lr){6-8}
Dataset & Target & $0$ & $0.5$ & max & $0$ & $0.5$ & max & Max $w$ \\
\midrule
MMLU & 0.8 & $+0.358$ (0.275) & $-0.165$ (0.295) & $-0.055$ (0.146) & $+0.0059$ & $+0.0083$ & $+0.0023$ & 0.99 \\
 & 0.9 & $+0.809$ (0.416) & $+0.203$ (0.310) & $-0.147$ (0.249) & $+0.0067$ & $+0.0104$ & $+0.0052$ & 0.79 \\
TriviaQA & 0.8 & $+0.003$ (0.312) & $-0.092$ (0.201) & $-0.025$ (0.054) & $+0.0008$ & $+0.0004$ & $+0.0001$ & 1.00 \\
 & 0.9 & $+1.073$ (0.297) & $+0.031$ (0.197) & $-0.041$ (0.073) & $+0.0011$ & $+0.0010$ & $+0.0001$ & 0.81 \\
MATH & 0.8 & $+0.022$ (0.146) & $-0.113$ (0.098) & $-0.004$ (0.011) & $+0.0334$ & $+0.0198$ & $+0.0004$ & 0.99 \\
 & 0.9 & $+0.440$ (0.177) & $-0.107$ (0.165) & $-0.050$ (0.084) & $+0.0517$ & $+0.0435$ & $+0.0105$ & 0.87 \\
SimpleQA & 0.8 & $-0.038$ (0.322) & $-0.181$ (0.236) & $-0.012$ (0.047) & $+0.0016$ & $+0.0006$ & $+0.0008$ & 0.99 \\
 & 0.9 & $+0.569$ (0.320) & $-0.214$ (0.203) & $-0.079$ (0.102) & $+0.0028$ & $+0.0019$ & $+0.0014$ & 0.70 \\
LiveCodeBench & 0.8 & $-0.079$ (0.130) & $-0.057$ (0.079) & $-0.047$ (0.042) & $+0.0011$ & $+0.0009$ & $+0.0008$ & 1.00 \\
 & 0.9 & $-0.072$ (0.093) & $-0.067$ (0.056) & $-0.035$ (0.029) & $+0.0010$ & $+0.0008$ & $+0.0012$ & 0.82 \\
\bottomrule
\end{tabular*}

\end{table}

\paragraph{Stage diagnostics.}
For test-set queries reaching each nonterminal stage of the selected
policies, we compute $\operatorname{AUROC}(s_i,U_i)$, $\operatorname{AUROC}(s_i,U_{>i})$, and the right side of
Equation~\ref{eq:auc-gain} (Table~\ref{tab:shared-stage}). Both AUROCs are conditional on reaching stage $i$. The downstream
answer is what the frozen continuation would return for every reached
query, including those accepted at the current stage. Diagnostics use
assigned-budget weights with trapezoidal endpoints and average over
splits. Coverage records the budget weight of policies containing that
nonterminal stage. Single-model policies contribute no stage diagnostic.
Undefined AUROCs are excluded from AUROC means, and their corresponding
terms in Equation~\ref{eq:auc-gain} are zero. Empty reached populations
have zero contribution. An independent empirical escalation integral,
with half credit for score ties, verifies the identity for each policy.

For selected $S_3$ policies, first-stage $\operatorname{AUROC}(s_1,U_1)$ stays within $0.005$ of the
target. $\operatorname{AUROC}(s_1,U_{>1})$ rises and the right side of Equation~\ref{eq:auc-gain}
falls at both steps in all ten dataset-target comparisons. The same
directional pattern holds for selected $S_2$ policies, whose first-stage
$\operatorname{AUROC}(s_1,U_1)$ is within $0.006$ of target. These policies use only the first
model's score, so this pattern does not require combining scores at
successive thresholds. At maximum $w$, the first-stage right side is
negative in six of ten comparisons for each class, meaning that
score-based escalation is on average worse than random escalation at
the same rate when integrated over escalation rates.

\paragraph{Depth gaps across $w$.}
The held-out gap is often more negative at $w=0.5$ than at maximum $w$.
This pattern accompanies a reduction in stage-2 coverage. At $w=0.5$,
calibration search selects three-model sequences at 32--91\% of budget
positions, and their small calibration gains do not consistently transfer.
At maximum $w$, it selects them at 2--67\% of positions. Thus $S_3$
coincides with $S_2$ more often and the gap generally moves toward zero.

\paragraph{Later stages.}
At the second stage, $\operatorname{AUROC}(s_2,U_2)$ falls below its full-sample target at maximum
$w$, reaching $0.50$ to $0.78$ at target $0.8$. Queries that reach stage 2
have low $s_1$. When scores share a difficulty term, conditioning on low
$s_1$ can narrow the range of $s_2$ among reached queries and reduce its
ability to rank correctness. The diagnostics are consistent with shared
difficulty lowering the value of a later threshold both by raising
$\operatorname{AUROC}(s_i,U_{>i})$ and by lowering $\operatorname{AUROC}(s_i,U_i)$ among queries that reach the stage.

\paragraph{Observed-score diagnostics.}
Table~\ref{tab:observed-stage} applies the same diagnostics to frozen
$S_2$ and $S_3$ policies under the seven confidence-based scores. We reconstruct
learned scores from saved classifier coefficients, without new scorer
fitting or threshold search, and verify each policy's test accuracy,
cost, and number of model calls against its saved evaluation. Observed
scores are not calibrated to a common correctness AUROC, so their two
AUROC terms should be compared jointly rather than mapped to a single
shared-difficulty weight.

\paragraph{Limits.}
Equation~\ref{eq:auc-gain} describes gain over random escalation for a
fixed continuation policy and does not guarantee monotone optimized
depth gains. The selected model sequences can change across $w$, so the diagnostics
compare different model sequences. The supported comparison is the within-dataset
direction, not a prediction of depth-gain magnitudes across datasets.
For example, LiveCodeBench has a large first-stage value at $w=0$ but
no positive depth gain. Score correlations also change with $w$, so this
experiment does not separate downstream ranking from dependence across
scores. All synthetic scores use full-sample correctness labels, and
held-out evaluation concerns policy selection rather than scorer learning.
These are one-draw results, split percentiles are not confidence intervals,
and we make no significance claims.

\FloatBarrier
\begingroup
\scriptsize
\renewcommand{\arraystretch}{0.85}
\setlength{\tabcolsep}{4pt}
\begin{longtable}{llrrrrr}
\caption{Test-set diagnostics for selected synthetic-score policies. Blocks give $S_2$ stage 1, $S_3$ stage 1, and $S_3$ stage 2. AUROCs and the right side of Equation~\ref{eq:auc-gain} are averaged with assigned-budget weights over selected policies and 50 splits. The RHS is in percentage points and includes reach probability. Coverage is the fraction of budget weight with that nonterminal stage. Undefined AUROCs are omitted from their means.}\label{tab:shared-stage}\\
\toprule
Dataset & Target & $w$ & $\operatorname{AUROC}(s_i,U_i)$ & $\operatorname{AUROC}(s_i,U_{>i})$ & RHS (pp) & Coverage \\
\midrule
\endfirsthead
\multicolumn{7}{c}{\tablename\ \thetable\ continued} \\
\toprule
Dataset & Target & $w$ & $\operatorname{AUROC}(s_i,U_i)$ & $\operatorname{AUROC}(s_i,U_{>i})$ & RHS (pp) & Coverage \\
\midrule
\endhead
\midrule
\multicolumn{7}{r}{Continued on next page} \\
\endfoot
\bottomrule
\endlastfoot
\multicolumn{7}{l}{\textbf{$S_2$, stage 1}} \\*
\midrule
MMLU & 0.8 & 0.00 & 0.798 & 0.616 & $+3.207$ & 0.998 \\
 &  & 0.50 & 0.798 & 0.715 & $+1.831$ & 0.998 \\
 &  & 0.99 & 0.797 & 0.855 & $+0.302$ & 0.673 \\
 & 0.9 & 0.00 & 0.899 & 0.656 & $+4.286$ & 0.998 \\
 &  & 0.50 & 0.899 & 0.784 & $+2.483$ & 0.998 \\
 &  & 0.79 & 0.898 & 0.885 & $+1.304$ & 0.989 \\
TriviaQA & 0.8 & 0.00 & 0.799 & 0.681 & $+3.204$ & 0.999 \\
 &  & 0.50 & 0.801 & 0.755 & $+2.159$ & 0.998 \\
 &  & 1.00 & 0.802 & 0.855 & $+0.769$ & 0.957 \\
 & 0.9 & 0.00 & 0.900 & 0.739 & $+4.360$ & 0.999 \\
 &  & 0.50 & 0.900 & 0.835 & $+2.915$ & 0.998 \\
 &  & 0.81 & 0.901 & 0.908 & $+2.043$ & 0.984 \\
MATH & 0.8 & 0.00 & 0.801 & 0.669 & $+2.176$ & 0.997 \\
 &  & 0.50 & 0.802 & 0.740 & $+1.012$ & 0.992 \\
 &  & 0.99 & 0.802 & 0.869 & $-1.174$ & 0.554 \\
 & 0.9 & 0.00 & 0.900 & 0.720 & $+2.939$ & 0.998 \\
 &  & 0.50 & 0.901 & 0.823 & $+1.255$ & 0.995 \\
 &  & 0.87 & 0.901 & 0.931 & $-0.714$ & 0.878 \\
SimpleQA & 0.8 & 0.00 & 0.804 & 0.579 & $+2.629$ & 0.998 \\
 &  & 0.50 & 0.805 & 0.662 & $+0.517$ & 0.995 \\
 &  & 0.99 & 0.802 & 0.843 & $-4.599$ & 0.684 \\
 & 0.9 & 0.00 & 0.903 & 0.606 & $+3.533$ & 0.999 \\
 &  & 0.50 & 0.904 & 0.731 & $+0.746$ & 0.991 \\
 &  & 0.70 & 0.903 & 0.814 & $-1.631$ & 0.975 \\
LiveCodeBench & 0.8 & 0.00 & 0.801 & 0.608 & $+3.955$ & 0.994 \\
 &  & 0.50 & 0.802 & 0.698 & $+1.908$ & 0.988 \\
 &  & 1.00 & 0.802 & 0.858 & $-1.499$ & 0.941 \\
 & 0.9 & 0.00 & 0.901 & 0.643 & $+5.251$ & 0.994 \\
 &  & 0.50 & 0.901 & 0.765 & $+2.509$ & 0.986 \\
 &  & 0.82 & 0.901 & 0.886 & $-0.061$ & 0.938 \\
\multicolumn{7}{l}{\textbf{$S_3$, stage 1}} \\*
\midrule
MMLU & 0.8 & 0.00 & 0.797 & 0.630 & $+3.327$ & 0.998 \\
 &  & 0.50 & 0.797 & 0.721 & $+2.008$ & 0.998 \\
 &  & 0.99 & 0.798 & 0.855 & $+0.198$ & 0.692 \\
 & 0.9 & 0.00 & 0.898 & 0.680 & $+4.459$ & 0.998 \\
 &  & 0.50 & 0.899 & 0.792 & $+2.793$ & 0.998 \\
 &  & 0.79 & 0.899 & 0.886 & $+1.435$ & 0.991 \\
TriviaQA & 0.8 & 0.00 & 0.798 & 0.665 & $+3.599$ & 0.999 \\
 &  & 0.50 & 0.801 & 0.754 & $+2.151$ & 0.998 \\
 &  & 1.00 & 0.802 & 0.856 & $+0.716$ & 0.962 \\
 & 0.9 & 0.00 & 0.899 & 0.715 & $+5.659$ & 0.999 \\
 &  & 0.50 & 0.900 & 0.827 & $+3.043$ & 0.998 \\
 &  & 0.81 & 0.901 & 0.908 & $+2.007$ & 0.988 \\
MATH & 0.8 & 0.00 & 0.801 & 0.630 & $+2.728$ & 0.997 \\
 &  & 0.50 & 0.801 & 0.724 & $+1.196$ & 0.994 \\
 &  & 0.99 & 0.802 & 0.870 & $-1.245$ & 0.559 \\
 & 0.9 & 0.00 & 0.899 & 0.669 & $+3.855$ & 0.998 \\
 &  & 0.50 & 0.901 & 0.789 & $+1.694$ & 0.995 \\
 &  & 0.87 & 0.901 & 0.932 & $-1.092$ & 0.882 \\
SimpleQA & 0.8 & 0.00 & 0.803 & 0.566 & $+2.507$ & 0.998 \\
 &  & 0.50 & 0.805 & 0.654 & $+0.571$ & 0.995 \\
 &  & 0.99 & 0.802 & 0.844 & $-4.563$ & 0.693 \\
 & 0.9 & 0.00 & 0.902 & 0.591 & $+3.211$ & 0.999 \\
 &  & 0.50 & 0.902 & 0.689 & $+0.382$ & 0.992 \\
 &  & 0.70 & 0.902 & 0.791 & $-1.759$ & 0.978 \\
LiveCodeBench & 0.8 & 0.00 & 0.801 & 0.604 & $+4.031$ & 0.995 \\
 &  & 0.50 & 0.802 & 0.696 & $+1.952$ & 0.988 \\
 &  & 1.00 & 0.802 & 0.859 & $-1.576$ & 0.949 \\
 & 0.9 & 0.00 & 0.901 & 0.641 & $+5.313$ & 0.994 \\
 &  & 0.50 & 0.901 & 0.764 & $+2.520$ & 0.987 \\
 &  & 0.82 & 0.901 & 0.887 & $-0.116$ & 0.943 \\
\multicolumn{7}{l}{\textbf{$S_3$, stage 2}} \\*
\midrule
MMLU & 0.8 & 0.00 & 0.789 & 0.618 & $+1.047$ & 0.952 \\
 &  & 0.50 & 0.783 & 0.694 & $+0.501$ & 0.908 \\
 &  & 0.99 & 0.584 & 0.720 & $-0.224$ & 0.130 \\
 & 0.9 & 0.00 & 0.894 & 0.659 & $+1.384$ & 0.951 \\
 &  & 0.50 & 0.871 & 0.732 & $+0.747$ & 0.900 \\
 &  & 0.79 & 0.792 & 0.781 & $+0.042$ & 0.674 \\
TriviaQA & 0.8 & 0.00 & 0.783 & 0.633 & $+1.373$ & 0.925 \\
 &  & 0.50 & 0.766 & 0.673 & $+0.725$ & 0.640 \\
 &  & 1.00 & 0.692 & 0.732 & $-0.392$ & 0.135 \\
 & 0.9 & 0.00 & 0.892 & 0.694 & $+2.452$ & 0.953 \\
 &  & 0.50 & 0.868 & 0.741 & $+1.182$ & 0.756 \\
 &  & 0.81 & 0.820 & 0.751 & $+0.538$ & 0.450 \\
MATH & 0.8 & 0.00 & 0.755 & 0.594 & $+1.190$ & 0.843 \\
 &  & 0.50 & 0.731 & 0.642 & $+0.524$ & 0.651 \\
 &  & 0.99 & 0.499 & 0.617 & $-0.326$ & 0.022 \\
 & 0.9 & 0.00 & 0.859 & 0.629 & $+1.487$ & 0.878 \\
 &  & 0.50 & 0.838 & 0.678 & $+0.703$ & 0.724 \\
 &  & 0.87 & 0.845 & 0.879 & $-0.023$ & 0.268 \\
SimpleQA & 0.8 & 0.00 & 0.796 & 0.564 & $+1.393$ & 0.813 \\
 &  & 0.50 & 0.763 & 0.636 & $+0.223$ & 0.315 \\
 &  & 0.99 & 0.780 & 0.914 & $-4.666$ & 0.033 \\
 & 0.9 & 0.00 & 0.899 & 0.584 & $+2.085$ & 0.964 \\
 &  & 0.50 & 0.890 & 0.712 & $+0.592$ & 0.617 \\
 &  & 0.70 & 0.889 & 0.825 & $-0.302$ & 0.323 \\
LiveCodeBench & 0.8 & 0.00 & 0.835 & 0.609 & $+1.428$ & 0.630 \\
 &  & 0.50 & 0.795 & 0.641 & $+0.815$ & 0.591 \\
 &  & 1.00 & 0.598 & 0.719 & $-0.129$ & 0.329 \\
 & 0.9 & 0.00 & 0.927 & 0.639 & $+1.765$ & 0.552 \\
 &  & 0.50 & 0.892 & 0.678 & $+0.983$ & 0.446 \\
 &  & 0.82 & 0.850 & 0.816 & $+0.161$ & 0.281 \\
\end{longtable}
\endgroup

\begingroup
\scriptsize
\renewcommand{\arraystretch}{0.85}
\setlength{\tabcolsep}{4pt}
\begin{longtable}{llrrrr}
\caption{Test-set diagnostics for frozen policies under the seven confidence-based scores. No threshold search or scorer fitting is repeated. Saved logistic-regression coefficients reconstruct learned scores. Blocks give $S_2$ stage 1, $S_3$ stage 1, and $S_3$ stage 2, with the same weighting and conventions as Table~\ref{tab:shared-stage}. The RHS of Equation~\ref{eq:auc-gain} is in percentage points.}\label{tab:observed-stage}\\
\toprule
Dataset & Scorer & $\operatorname{AUROC}(s_i,U_i)$ & $\operatorname{AUROC}(s_i,U_{>i})$ & RHS (pp) & Coverage \\
\midrule
\endfirsthead
\multicolumn{6}{c}{\tablename\ \thetable\ continued} \\
\toprule
Dataset & Scorer & $\operatorname{AUROC}(s_i,U_i)$ & $\operatorname{AUROC}(s_i,U_{>i})$ & RHS (pp) & Coverage \\
\midrule
\endhead
\midrule
\multicolumn{6}{r}{Continued on next page} \\
\endfoot
\bottomrule
\endlastfoot
\multicolumn{6}{l}{\textbf{$S_2$, stage 1}} \\*
\midrule
MMLU & Mean negentropy & 0.797 & 0.684 & $+2.277$ & 0.997 \\
 & Min negentropy & 0.788 & 0.688 & $+2.181$ & 0.996 \\
 & Probability margin & 0.792 & 0.681 & $+2.246$ & 0.997 \\
 & Min probability & 0.776 & 0.674 & $+2.329$ & 0.995 \\
 & Length-normalized probability & 0.786 & 0.665 & $+2.372$ & 0.995 \\
 & Correctness-trained token-prob. &0.798 & 0.686 & $+2.273$ & 0.997 \\
 & Correctness-trained embedding &0.669 & 0.599 & $+1.376$ & 0.996 \\
TriviaQA & Mean negentropy & 0.745 & 0.650 & $+2.613$ & 0.998 \\
 & Min negentropy & 0.771 & 0.652 & $+3.093$ & 0.996 \\
 & Probability margin & 0.704 & 0.617 & $+2.300$ & 0.998 \\
 & Min probability & 0.764 & 0.640 & $+3.110$ & 0.995 \\
 & Length-normalized probability & 0.747 & 0.637 & $+2.833$ & 0.992 \\
 & Correctness-trained token-prob. &0.779 & 0.654 & $+3.199$ & 0.996 \\
 & Correctness-trained embedding &0.710 & 0.613 & $+2.402$ & 0.997 \\
MATH & Mean negentropy & 0.643 & 0.633 & $+0.276$ & 0.993 \\
 & Min negentropy & 0.692 & 0.686 & $+0.264$ & 0.917 \\
 & Probability margin & 0.635 & 0.626 & $+0.271$ & 0.994 \\
 & Min probability & 0.685 & 0.676 & $+0.259$ & 0.928 \\
 & Length-normalized probability & 0.645 & 0.635 & $+0.270$ & 0.994 \\
 & Correctness-trained token-prob. &0.678 & 0.660 & $+0.363$ & 0.993 \\
 & Correctness-trained embedding &0.748 & 0.726 & $+0.345$ & 0.996 \\
SimpleQA & Mean negentropy & 0.696 & 0.573 & $+1.175$ & 0.995 \\
 & Min negentropy & 0.726 & 0.621 & $+0.078$ & 0.987 \\
 & Probability margin & 0.676 & 0.562 & $+1.130$ & 0.994 \\
 & Min probability & 0.718 & 0.616 & $+0.075$ & 0.984 \\
 & Length-normalized probability & 0.709 & 0.591 & $+0.823$ & 0.995 \\
 & Correctness-trained token-prob. &0.748 & 0.624 & $+0.556$ & 0.994 \\
 & Correctness-trained embedding &0.632 & 0.584 & $+0.030$ & 0.996 \\
LiveCodeBench & Mean negentropy & 0.779 & 0.675 & $+2.076$ & 0.977 \\
 & Min negentropy & 0.759 & 0.680 & $+1.655$ & 0.986 \\
 & Probability margin & 0.758 & 0.666 & $+1.871$ & 0.978 \\
 & Min probability & 0.735 & 0.664 & $+1.129$ & 0.889 \\
 & Length-normalized probability & 0.756 & 0.666 & $+1.715$ & 0.990 \\
 & Correctness-trained token-prob. &0.782 & 0.692 & $+1.718$ & 0.984 \\
 & Correctness-trained embedding &0.698 & 0.668 & $+0.331$ & 0.986 \\
\multicolumn{6}{l}{\textbf{$S_3$, stage 1}} \\*
\midrule
MMLU & Mean negentropy & 0.813 & 0.714 & $+2.654$ & 0.997 \\
 & Min negentropy & 0.811 & 0.703 & $+2.758$ & 0.996 \\
 & Probability margin & 0.811 & 0.711 & $+2.654$ & 0.997 \\
 & Min probability & 0.790 & 0.691 & $+2.432$ & 0.995 \\
 & Length-normalized probability & 0.794 & 0.688 & $+2.686$ & 0.995 \\
 & Correctness-trained token-prob. &0.812 & 0.713 & $+2.619$ & 0.997 \\
 & Correctness-trained embedding &0.678 & 0.597 & $+1.731$ & 0.996 \\
TriviaQA & Mean negentropy & 0.747 & 0.643 & $+2.782$ & 0.998 \\
 & Min negentropy & 0.772 & 0.648 & $+3.211$ & 0.997 \\
 & Probability margin & 0.710 & 0.616 & $+2.503$ & 0.998 \\
 & Min probability & 0.765 & 0.635 & $+3.238$ & 0.995 \\
 & Length-normalized probability & 0.748 & 0.633 & $+2.925$ & 0.994 \\
 & Correctness-trained token-prob. &0.779 & 0.649 & $+3.323$ & 0.997 \\
 & Correctness-trained embedding &0.710 & 0.604 & $+2.481$ & 0.997 \\
MATH & Mean negentropy & 0.643 & 0.621 & $+0.362$ & 0.994 \\
 & Min negentropy & 0.693 & 0.679 & $+0.431$ & 0.937 \\
 & Probability margin & 0.635 & 0.617 & $+0.298$ & 0.995 \\
 & Min probability & 0.685 & 0.670 & $+0.351$ & 0.947 \\
 & Length-normalized probability & 0.645 & 0.622 & $+0.357$ & 0.995 \\
 & Correctness-trained token-prob. &0.680 & 0.651 & $+0.537$ & 0.994 \\
 & Correctness-trained embedding &0.747 & 0.692 & $+0.626$ & 0.996 \\
SimpleQA & Mean negentropy & 0.692 & 0.569 & $+1.185$ & 0.996 \\
 & Min negentropy & 0.725 & 0.613 & $+0.155$ & 0.989 \\
 & Probability margin & 0.674 & 0.557 & $+1.214$ & 0.995 \\
 & Min probability & 0.723 & 0.607 & $+0.129$ & 0.986 \\
 & Length-normalized probability & 0.708 & 0.584 & $+0.995$ & 0.995 \\
 & Correctness-trained token-prob. &0.748 & 0.616 & $+0.587$ & 0.995 \\
 & Correctness-trained embedding &0.647 & 0.556 & $+0.470$ & 0.996 \\
LiveCodeBench & Mean negentropy & 0.779 & 0.675 & $+2.080$ & 0.980 \\
 & Min negentropy & 0.759 & 0.680 & $+1.653$ & 0.987 \\
 & Probability margin & 0.758 & 0.666 & $+1.874$ & 0.982 \\
 & Min probability & 0.735 & 0.663 & $+1.152$ & 0.910 \\
 & Length-normalized probability & 0.756 & 0.666 & $+1.732$ & 0.991 \\
 & Correctness-trained token-prob. &0.782 & 0.693 & $+1.710$ & 0.986 \\
 & Correctness-trained embedding &0.698 & 0.667 & $+0.338$ & 0.986 \\
\multicolumn{6}{l}{\textbf{$S_3$, stage 2}} \\*
\midrule
MMLU & Mean negentropy & 0.709 & 0.558 & $+0.921$ & 0.819 \\
 & Min negentropy & 0.702 & 0.557 & $+0.767$ & 0.912 \\
 & Probability margin & 0.702 & 0.565 & $+0.900$ & 0.797 \\
 & Min probability & 0.651 & 0.583 & $+0.456$ & 0.772 \\
 & Length-normalized probability & 0.709 & 0.563 & $+0.959$ & 0.747 \\
 & Correctness-trained token-prob. &0.710 & 0.561 & $+0.914$ & 0.775 \\
 & Correctness-trained embedding &0.660 & 0.537 & $+0.655$ & 0.865 \\
TriviaQA & Mean negentropy & 0.696 & 0.581 & $+0.806$ & 0.892 \\
 & Min negentropy & 0.715 & 0.598 & $+0.791$ & 0.805 \\
 & Probability margin & 0.684 & 0.563 & $+0.941$ & 0.849 \\
 & Min probability & 0.704 & 0.597 & $+0.698$ & 0.841 \\
 & Length-normalized probability & 0.696 & 0.575 & $+0.809$ & 0.889 \\
 & Correctness-trained token-prob. &0.721 & 0.610 & $+0.749$ & 0.841 \\
 & Correctness-trained embedding &0.656 & 0.559 & $+0.705$ & 0.826 \\
MATH & Mean negentropy & 0.652 & 0.597 & $+0.347$ & 0.432 \\
 & Min negentropy & 0.681 & 0.652 & $+0.104$ & 0.641 \\
 & Probability margin & 0.636 & 0.584 & $+0.272$ & 0.403 \\
 & Min probability & 0.701 & 0.641 & $+0.278$ & 0.627 \\
 & Length-normalized probability & 0.692 & 0.637 & $+0.264$ & 0.509 \\
 & Correctness-trained token-prob. &0.682 & 0.662 & $+0.093$ & 0.665 \\
 & Correctness-trained embedding &0.635 & 0.612 & $+0.091$ & 0.762 \\
SimpleQA & Mean negentropy & 0.691 & 0.588 & $+0.590$ & 0.330 \\
 & Min negentropy & 0.710 & 0.623 & $+0.005$ & 0.289 \\
 & Probability margin & 0.678 & 0.561 & $+0.509$ & 0.376 \\
 & Min probability & 0.732 & 0.619 & $+0.181$ & 0.263 \\
 & Length-normalized probability & 0.727 & 0.586 & $+0.350$ & 0.430 \\
 & Correctness-trained token-prob. &0.752 & 0.633 & $+0.318$ & 0.366 \\
 & Correctness-trained embedding &0.608 & 0.570 & $-0.002$ & 0.846 \\
LiveCodeBench & Mean negentropy & 0.812 & 0.726 & $+0.525$ & 0.079 \\
 & Min negentropy & 0.738 & 0.616 & $+0.355$ & 0.204 \\
 & Probability margin & 0.807 & 0.722 & $+0.672$ & 0.098 \\
 & Min probability & 0.727 & 0.657 & $+0.458$ & 0.319 \\
 & Length-normalized probability & 0.841 & 0.690 & $+0.644$ & 0.206 \\
 & Correctness-trained token-prob. &0.803 & 0.624 & $+0.370$ & 0.096 \\
 & Correctness-trained embedding &0.556 & 0.532 & $+0.110$ & 0.329 \\
\end{longtable}
\endgroup

\FloatBarrier

\subsection{Synthetic Cost Structures}
\label{app:cost-structure}

This appendix tests whether the depth results depend on the particular
prices of the eight models in our pool. We replace observed prices with
synthetic ones and repeat the depth comparison. Correctness labels and
recorded per-query token counts are unchanged. We run the sweep twice:
once with the primary observed scorer, mean token negentropy, and once
with the AUROC $0.9$ synthetic scores of
Appendix~\ref{app:synthetic-scores}. The second run tests whether cost
structure matters on datasets and scores for which a third stage already
pays.

\paragraph{Construction.}
For each dataset, let $c_i$ denote model $i$'s observed expected
per-query cost, with models ordered by $c_i$ and ties broken by model
name as elsewhere. Let $R = c_K / c_1$. We place the models at
normalized log-cost positions
\[
u_i = Q_{\mathrm{Beta}(a,b)}\!\left(\tfrac{i-1}{K-1}\right),
\qquad
c_i^{\text{new}} = c_1 \cdot R^{\,u_i},
\]
where $Q_{\mathrm{Beta}(a,b)}$ is the Beta quantile function. Because
$Q(0) = 0$ and $Q(1) = 1$, the cheapest and most expensive models keep
their observed full-sample mean costs. Because $Q$ is increasing, the
full-sample mean-cost ordering is preserved. Within each calibration
split, however, we recompute mean-cost ordering and dollar-budget anchors
after applying the price multipliers. These can differ across settings,
so the comparisons match normalized budget positions rather than fixed
dollar budgets. All cost-ordered sequences under each setting's
calibration ordering remain available.

We apply the change as a per-model price multiplier
$m_j = c_j^{\text{new}} / c_j$, scaling model $j$'s input and output
per-million token prices by $m_j$ and recomputing all per-query costs
from the stored token counts. We do not rescale per-query costs
directly. This preserves the within-model variation in cost across
queries and any dependence between cost and confidence score
(Appendix~\ref{app:cost-variability}). Because $c_j$ is dataset-specific,
so are the multipliers, and a model can carry different prices on
different datasets.

\paragraph{Settings.}
Table~\ref{tab:cost-settings} lists the nine $(a,b)$ pairs. Five are
symmetric and vary how far the intermediate models sit from the
endpoints. Four are asymmetric and shift them toward the cheap or the
expensive end.

\begin{table}[h]
\centering
\caption{Synthetic cost settings. Positions $u_i$ are equally spaced
$\mathrm{Beta}(a,b)$ quantiles on normalized log cost, with $u_1 = 0$
and $u_K = 1$.}
\label{tab:cost-settings}
\begin{tabular}{llll}
\toprule
Label & $a$ & $b$ & Intermediate models \\
\midrule
S$--$ & 0.25 & 0.25 & near both endpoints \\
S$-$  & 0.5  & 0.5  & mildly toward both endpoints \\
U     & 1    & 1    & evenly spaced in log cost \\
S$+$  & 2    & 2    & mildly toward the middle \\
S$++$ & 4    & 4    & concentrated in the middle \\
L$+$  & 1    & 2    & toward the cheap end \\
L$++$ & 1    & 4    & strongly toward the cheap end \\
R$+$  & 2    & 1    & toward the expensive end \\
R$++$ & 4    & 1    & strongly toward the expensive end \\
\bottomrule
\end{tabular}
\end{table}

\paragraph{Locating the observed pool.}
To place the observed prices on the same axes, we fit $(a,b)$ to each
dataset's observed normalized log-cost positions by least squares over
the $K$ points. Table~\ref{tab:cost-fit} reports the fitted values, the
residual, and $R$.

\begin{table}[h]
\centering
\caption{Fitted $\mathrm{Beta}(a,b)$ positions of the observed pools.
Parameters are fit by least squares to each dataset's observed
normalized log-cost positions, with RMSE the residual over the $K$
points and $R = c_K / c_1$ the observed cost ratio between the most and
least expensive model.}
\label{tab:cost-fit}
\begin{tabular}{lrrrr}
\toprule
Dataset & $a$ & $b$ & RMSE & $R$ \\
\midrule
MMLU          & 1.27 & 1.50 & 0.044 & 59.0 \\
TriviaQA      & 1.13 & 1.23 & 0.030 & 84.7 \\
MATH          & 2.29 & 2.35 & 0.027 & 100.1 \\
SimpleQA      & 1.25 & 1.30 & 0.038 & 114.7 \\
LiveCodeBench & 1.18 & 1.08 & 0.037 & 53.3 \\
\bottomrule
\end{tabular}
\end{table}

\paragraph{Pipeline.}
Each setting uses the exact selection procedure of
Appendix~\ref{app:exact-depth} without modification: the same 50
calibration-test splits, the same 500 normalized budget positions, the
same threshold conventions, tie-breaking, witness replay, and freeze
before test evaluation. 
The two score conditions share the same prices, multipliers, splits, and
budget positions, and differ only in the confidence scores supplied to
the search.
Policies are selected on calibration data and
the reported differences are test-set differences, integrated over
budget positions within each split and averaged equally across splits.
Prices, multipliers, position vectors, and source fingerprints are
written to the manifest before any test row is evaluated. We verify in
every setting that the full-sample endpoint mean costs match their observed
values and that $c^{\text{new}}$ is strictly increasing. The manifests
also record changes in calibration ordering, budget anchors, and the
models defining those anchors. The U setting is not identical
to observed prices, since it respaces the intermediate models. The
observed-price results are those of Section~\ref{sec:exp-frontier}.

\begin{figure}[!htb]
\centering
\includegraphics[width=0.7\linewidth]{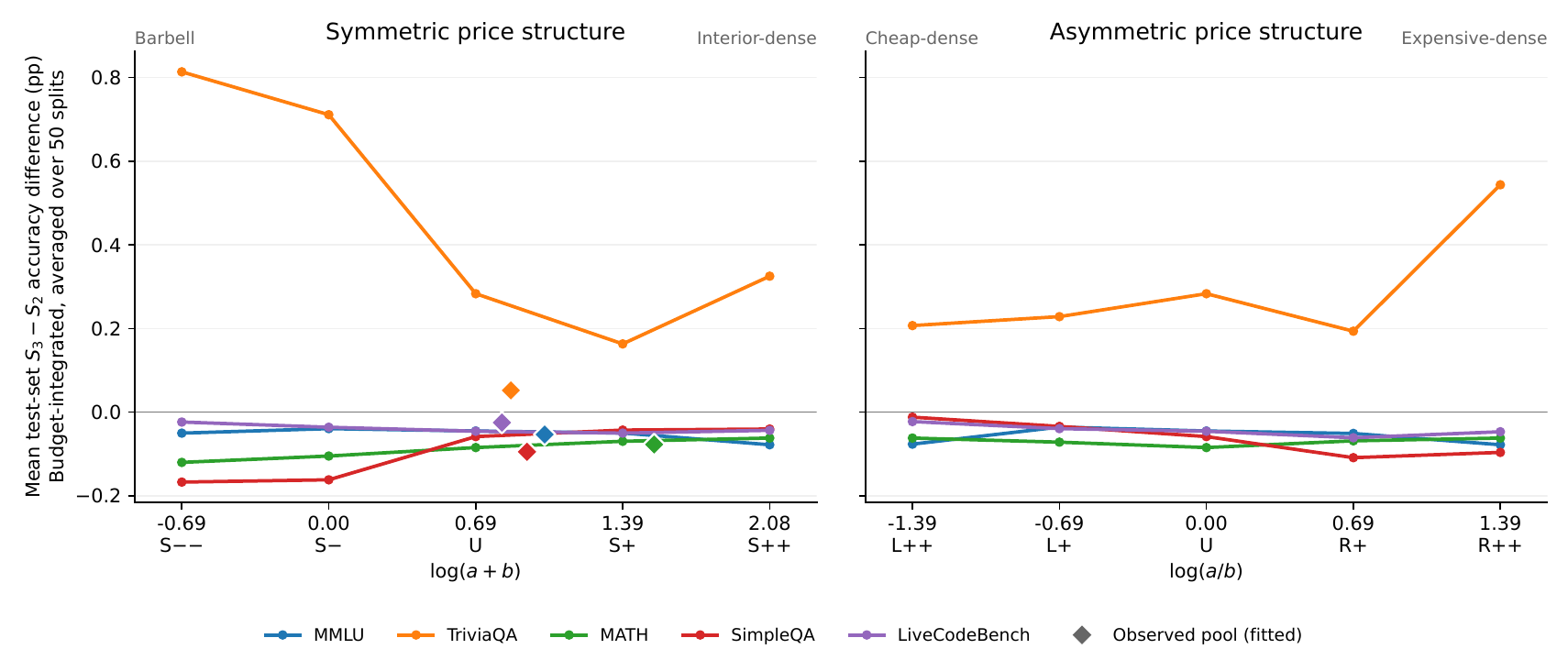}
\par\smallskip
\includegraphics[width=0.7\linewidth]{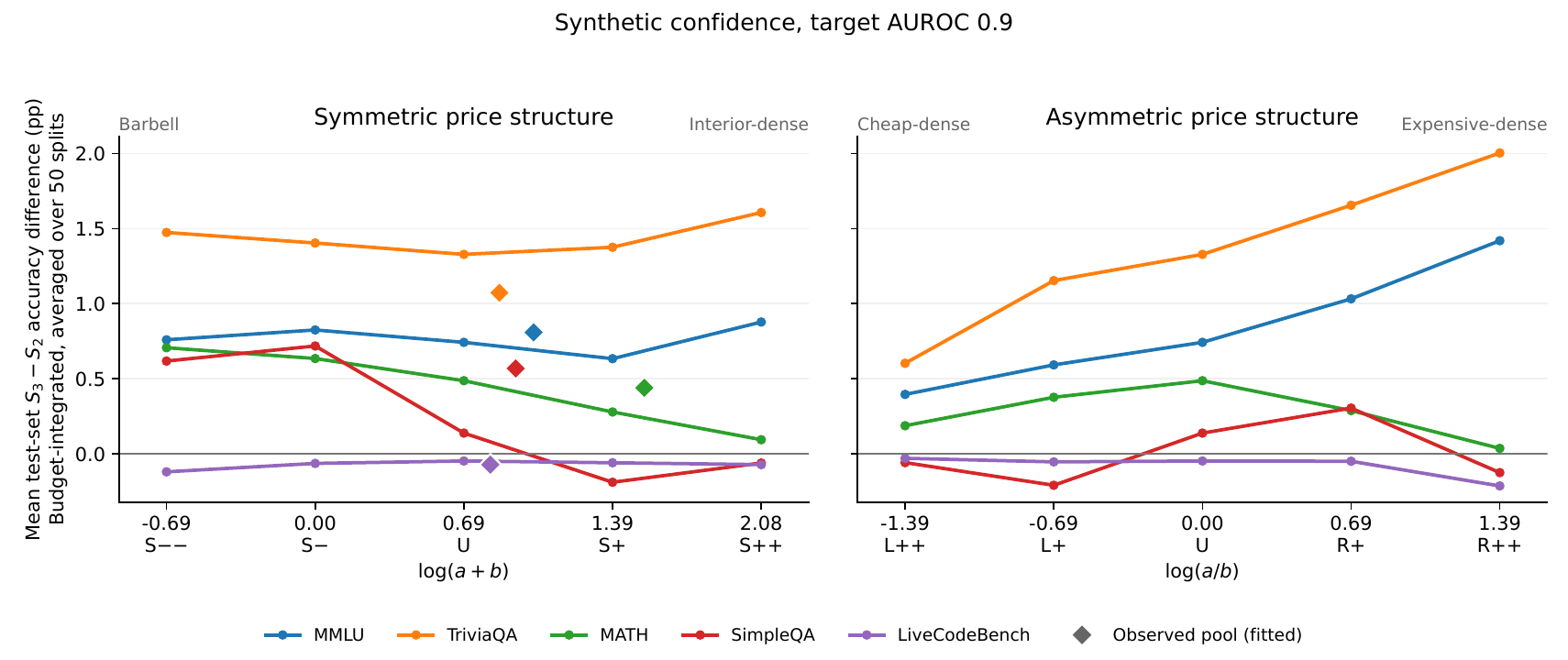}
\caption{\textbf{Depth gain under synthetic cost structures.} Mean
test-set $S_3 - S_2$ accuracy difference in percentage points, integrated
over normalized budget positions and averaged over 50 splits, for
observed mean token negentropy scores (top) and
AUROC $0.9$ synthetic scores (bottom). Left panels vary how far the
intermediate models sit from the cheapest and most expensive models,
from near both endpoints (S$--$) to concentrated in the middle
(S$++$), along $\log(a+b)$. Right panels shift them toward the cheap
end (L) or the expensive end (R), along $\log(a/b)$. U is the same
setting in both panels. Diamonds mark each dataset's observed-price
test result at its fitted position on the nearer parameter path,
measured in $(\log a,\log b)$ space.}
\label{fig:cost-structure}
\end{figure}

\paragraph{Results.}
Figure~\ref{fig:cost-structure} plots the mean test-set $S_3 - S_2$
difference against the two axes for both score conditions. With observed
scores, every setting gives a mean difference between $-0.17$ percentage points and
$-0.01$ percentage points on MMLU, MATH, SimpleQA, and LiveCodeBench. TriviaQA is
positive in all nine settings, ranging from $+0.16$ percentage points under S$+$ to
$+0.81$ percentage points under S$--$. Split variability on TriviaQA is large relative to
the spread across settings: the 10th--90th split percentiles are
$[+0.62, +0.97]$ at S$--$, $[-0.07, +0.63]$ at U, and $[+0.16, +0.57]$
at S$++$. The U range overlaps both other ranges, while the S$--$ and
S$++$ ranges do not overlap. These percentiles describe variability across
overlapping splits rather than confidence intervals for differences
between settings.

With the AUROC $0.9$ synthetic scores, the settings separate more. On
MMLU the gain ranges from $+0.40$ percentage points under L$++$ to $+1.42$ percentage points under
R$++$, and on MATH from $+0.04$ percentage points under R$++$ to $+0.71$ percentage points under S$--$.
TriviaQA ranges from $+0.60$ percentage points to $+2.00$ percentage points. SimpleQA changes sign
across settings, from $+0.72$ percentage points under S$-$ to $-0.21$ percentage points under L$+$.
LiveCodeBench stays negative under every setting in both score
conditions. The setting that gives the largest gain
differs across datasets, so these results do not identify a cost
spacing that favors depth in general. Clustering the intermediate models
near the cheapest model (L$++$) gives the smallest gain on MMLU and
TriviaQA.

The observed pools fit near the middle of the symmetric axis, so the
prices we work with are close to evenly spaced in log cost rather than
clustered at either end. The depth results in
Section~\ref{sec:exp-frontier} are therefore not explained by an unusual
cost structure. Cost structure does affect the size of the third-stage
gain, but only on datasets and scores where that gain is already
positive.

\end{document}